\documentclass[nohyperref]{article}

\usepackage{microtype}
\usepackage{graphicx}
\usepackage{subfigure}
\usepackage{booktabs} 
\usepackage{enumerate}
\usepackage{algorithmic}
\usepackage{algorithm}
\usepackage{adjustbox}
\usepackage{diagbox}

\usepackage{hyperref}

\usepackage[accepted]{icml2023}

\usepackage{amsmath}
\usepackage{amssymb}
\usepackage{mathtools}
\usepackage{amsthm}
\usepackage{thmtools, thm-restate}
\usepackage{bbm}
\input m.macros
\input m.symbols

\newcommand{\Tstrut}{\rule{0pt}{2.6ex}} 

\newcommand{\reload}{{R}e{LOAD}}

\usepackage[capitalize]{cleveref} 

\theoremstyle{plain}
\newtheorem{theorem}{Theorem}[section]

\theoremstyle{definition}
\newtheorem{definition}[theorem]{Definition}

\theoremstyle{remark}

\icmltitlerunning{~ \hfill ReLOAD for Last-Iterate Convergence in Constrained MDPs \hfill \thepage}

\begin{document}

\setlength{\abovedisplayskip}{4pt}
\setlength{\belowdisplayskip}{4pt}

\twocolumn[
\icmltitle{ReLOAD: Reinforcement Learning with Optimistic Ascent-Descent \\for Last-Iterate Convergence in Constrained MDPs}



\icmlsetsymbol{equal}{*}

\begin{icmlauthorlist}
\icmlauthor{Ted Moskovitz}{equal,sch}
\icmlauthor{Brendan O'Donoghue}{comp}
\icmlauthor{Vivek Veeriah}{comp} \\
\icmlauthor{Sebastian Flennerhag}{comp}
\icmlauthor{Satinder Singh}{comp}
\icmlauthor{Tom Zahavy}{comp}
\end{icmlauthorlist}

\icmlaffiliation{comp}{DeepMind}
\icmlaffiliation{sch}{Gatsby Unit, UCL}

\icmlcorrespondingauthor{Ted Moskovitz}{ted@gatsby.ucl.ac.uk}

\icmlkeywords{Machine Learning, ICML}

\vskip 0.3in
]



\printAffiliationsAndNotice{\icmlEqualContribution} 

\begin{figure*}[!t]
    \centering
    \includegraphics[width=0.85\textwidth]{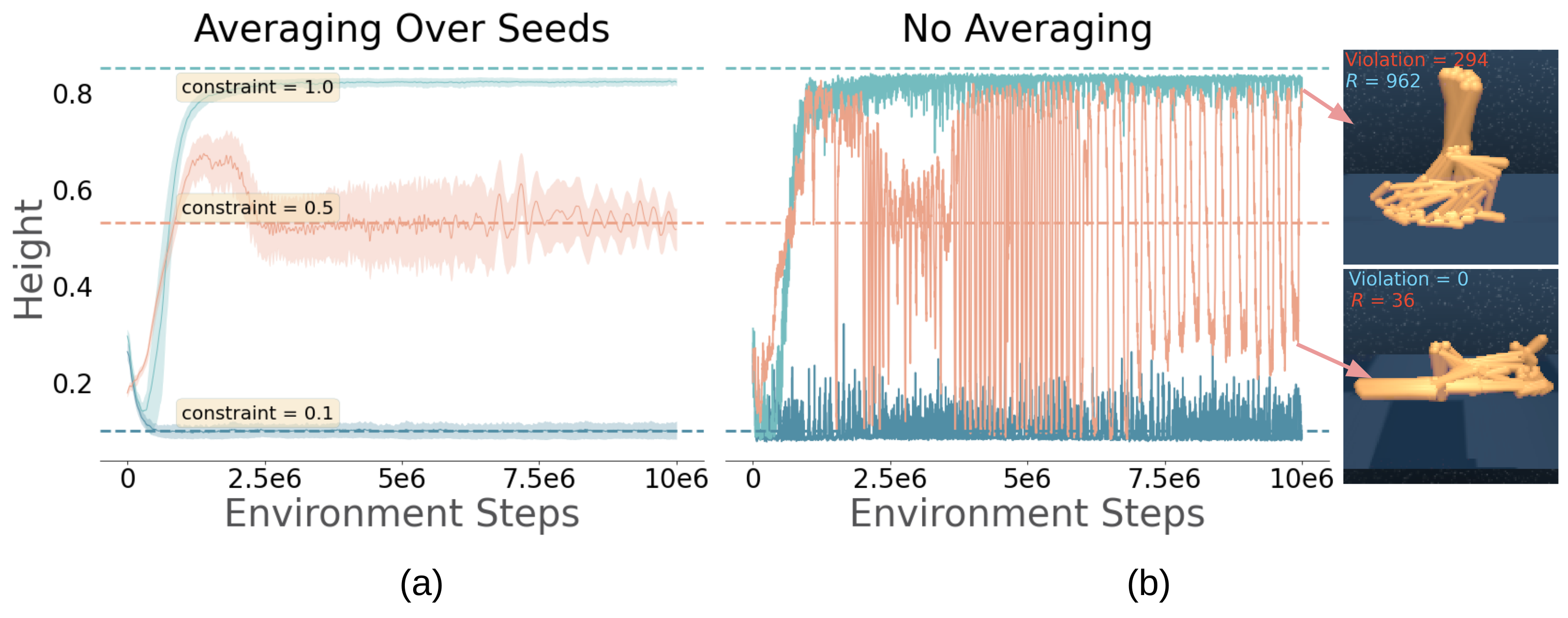}
    \caption{Standard gradient-based methods suffer from oscillations in constrained RL. (a) Training and agent to walk while keeping its height below different thresholds, learning looks stable when averaged across seeds. (b) Examining a single training run, we can see that learning oscillates dramatically, frequently leading to extreme behavior at the end of training (right).}
    \label{fig:the_problem}
\end{figure*}

\begin{abstract}
In recent years, Reinforcement Learning (RL) has been applied to real-world problems with increasing success. Such applications often require to put constraints on the agent's behavior.
Existing algorithms for constrained RL (CRL) rely on gradient descent-ascent, but this approach comes with a caveat. While these algorithms are guaranteed to converge \textit{on average}, they do not guarantee \textit{last-iterate} convergence, i.e., the current policy of the agent may never converge to the optimal solution. In practice, it is often observed that the policy alternates between satisfying the constraints and maximizing the  reward, rarely accomplishing both objectives simultaneously. Here, we address this problem by introducing \textit{Reinforcement Learning with Optimistic Ascent-Descent} (\reload), a principled CRL method with guaranteed last-iterate convergence. We demonstrate its empirical effectiveness on a wide variety of CRL problems including discrete MDPs and continuous control. In the process we establish a benchmark of challenging CRL problems. 
\end{abstract}

\section{Introduction}
From navigating stratospheric balloons through the atmosphere to coordinating plasma control for nuclear fusion research, reinforcement learning \citep[RL;][]{sutton2018reinforcement} has proved increasingly effective at solving real-world problems \citep{bellemare2020balloons,degrave2022fusion}. In RL, an agent interacts with an environment over a series of time steps with the goal of maximizing its expected cumulative reward. 
When developing intelligent systems which must learn and act in the real world, it's often the case that constraints are placed on agents to ensure that certain safety or efficiency requirements are satisfied. Examples range from training a robot to run while avoiding placing too much torque on its joints, to video compression \citep{mandhane2022learn2encode}, to maximizing the efficiency of commercial cooling systems under stability constraints \citep{luo2022controlling}. 

A natural question, then, is how to solve such constrained tasks. In standard RL, we often look to the Reward Hypothesis, which postulates that all goals and purposes of an intelligent agent can be achieved by maximization of a scalar reward \citep{sutton2004reward}. \citet{szepesvari2020cmdps} studied the implications of this hypothesis to CRL. Accordingly, constrained RL problems can be solved by integrating the constraints and task reward into a single, non-stationary reward signal \citep{altman99cmdps}. However, this approach carries a subtle but important challenge that can be easily overlooked: gradient-based optimization for such constrained problems only guarantees that the \textit{average} of the agent's behavior over the course of training converges to an optimal solution, with no assurances for the final policy. 

Unfortunately, simple solutions like averaging model parameters are ineffective when the policy is implemented as a deep neural network. While this problem has been studied in the context of GANs \citep{daskalakis2018training_gans,balduzzi2018mechanics}, it has not been addressed within CRL, which brings additional complications.
As an illustration, consider the problem of training an agent in the \texttt{Walker} domain in DeepMind Control Suite \citep{tassa2018dmc} to walk subject to varying upper limits on its height. In \cref{fig:the_problem}a, learning looks stable due to averaging across multiple seeds, but in \cref{fig:the_problem}b, we can see dramatic oscillations over the course of a single training run, with the agent either simply walking normally or lying on the ground. 

In this work, we address this issue, introducing \textit{Reinforcement Learning with Optimistic Ascent-Descent} (\reload), a principled CRL method for last-iterate convergence (LIC). Specifically, we make the following contributions:
    \textbf{(1)} We analyze several existing methods for CRL and show that they fail to achieve LIC (\cref{thm:md_bad}, \cref{thm:singly_optimistic}). We then prove LIC for a generalized form of optimistic mirror descent (\cref{thm:convergence}).
    \textbf{(2)} Building on these theoretical insights, we introduce ReLOAD    (\cref{sec:reload}) and demonstrate empirically that it achieves LIC in both the tabular and function approximation settings (\cref{sec:experiments_main}). Furthermore, ReLOAD improves the performance of strong baseline algorithms in challenging CRL problems. 
    \textbf{(3)} We identify constraints within a range of control tasks which are especially challenging for traditional methods. We list constraint and task details in the hope that these problems can be reused as a benchmark for CRL.

\section{Reinforcement Learning with Constraints}
In CRL, 
an agent not only seeks to maximize its cumulative reward, but must also obey constraints on its behavior.
Typically, this problem is modeled as a \textit{Constrained Markov Decision Process} \citep[CMDP,][]{altman99cmdps}. An infinite horizon, discounted CMDP is a tuple $\mathcal M_C = \left(\St, \A, r_0, \gamma, \rho, \{r_n\}_{n=1}^N, \{\theta_n\}_{n=1}^N\right)$, where $\St$ is the set of states, $\A$ is the set of available actions, $P: \St \times \A \to \mathcal P(\St)$ is the transition kernel, $r_0: \St \times \A \to \reals$ is the reward function, $\gamma\in[0, 1)$ is a discount factor, $\rho\in\mathcal P(\St)$ is the distribution over initial states, $\{r_n\}_{n=1}^N$ is the set of constraint rewards, and $\{\theta_n\}_{n=1}^N$ is the set of constraint thresholds. $\mathcal P(\cdot)$ denotes the set of distributions over a given space. 
At each time step, the agent samples an action from a stationary \textit{policy} $\pi: \St \to \mathcal P(\A)$, which causes the environment to transition to a new state. This process induces a cumulative, discounted state-action occupancy measure (henceforth, simply ``occupancy measure'') associated with the policy:
\begin{align}
    d_\pi(s,a) \triangleq (1 - \gamma) \sum\nolimits_{t=0}^\infty \gamma^t \pi(a|s)P_\pi(s_t=s).
\end{align}
This probability measure lies within the following convex feasible set (a polytope in the discrete case):
    \begin{align}
    \begin{split}
        \mathcal K \triangleq \big\{d_\pi \ \vert \ d_\pi \geq 0, \ \sum\nolimits_{a} d_\pi(s,a) = (1 - \gamma) \rho(s) \\ + \gamma \sum\nolimits_{s', a'} P(s|s',a')d_\pi(s',a') \big\}.
    \end{split}
    \end{align}
The agent's goal is to find a policy that maximizes its expected, cumulative, discounted reward while adhering to the designated constraints. This quantity is referred to as the policy's \textit{value}: $v_0^\pi \triangleq \langle r_0, d_\pi\rangle$. Mathematically, this can be formalized as a constrained optimization problem:
\begin{align} \label{eq:cmdp_v}
    \min_{\pi} \ -v^\pi_{0} \ \quad \mathrm{s.t.} \quad v_n^\pi \leq \theta_n,  \quad n=1,\dots,N,
\end{align}
where $v_n^\pi \triangleq \langle r_n, d_\pi\rangle$. On inspection, we can see that \cref{eq:cmdp_v}
defines a linear program in $d_\pi$:
\begin{align}
    \min_{d_\pi\in\mathcal K} \ -\langle r_0, d_\pi \rangle \enspace \mathrm{s.t.} \enspace \langle r_n, d_\pi \rangle \leq \theta_n, \enspace n=1,\dots,N
\end{align}
Typically, CMDPs are solved via Lagrangian relaxation \citep{everett1963lagrangian,altman99cmdps}, which reframes the objective as a convex-concave min-max game:
\begin{align} \label{eq:cmdp_lagrangian}
    \min_{d_\pi\in\mathcal K} \max_{\mu \geq 0} -\langle r_0, d_\pi \rangle + \sum_{n=1}^N \mu_n (\langle r_n, d_\pi \rangle - \theta_n) \triangleq \mathcal L(d_\pi, \mu).
\end{align}
We provide a more comprehensive review of relevant work on CMDPs in \cref{sec:additional_related}.

\subsection{The Scalarization Fallacy} Given \cref{eq:cmdp_lagrangian}, it would be tempting to simply solve the inner maximization problem over $\mu$ to find the optimal Lagrange multipliers $\mu^\star$. One could then ``scalarize'' the task and constraint rewards into a single stationary reward function $r^\star = r_0 + \sum_{n=1}^N \mu_n^\star r_n$ and solve the resulting standard MDP. However, the solution to this MDP is not typically the solution to the CMDP---one can easily define CMDPs for which the optimal policy must be stochastic \citep{altman99cmdps,szepesvari2020cmdps}, but all MDPs admit deterministic optimal policies \citep{puterman2014markov}. One notable exception occurs when one of the constraint thresholds $\theta_n$ is extreme (close to the highest or lowest possible $v_n$). In this case, one reward function $r_n$ will dominate the optimization and the solution will closely match the optimal policy for an MDP with reward $r_n$, as seen in \cref{fig:the_problem} for $\theta=0.1$ and $\theta=1.0$. Such cases are discussed in detail in \cref{sec:extreme_constraints}. For non-trivial cases, however, we must turn to more specialized approaches for minimax optimization.

\subsection{Average- vs. Last-Iterate Convergence}
A \textit{saddle-point} (SP), or equilibrium, of a two-player convex-concave zero-sum game like \cref{eq:cmdp_lagrangian}, is a pair $(d_\pi^\star, \mu^\star)$ such that $\mathcal L(d_\pi^\star, \mu) \leq \mathcal L(d_\pi^\star, \mu^\star) \leq \mathcal L(d_\pi, \mu^\star)$ $\forall d_\pi, \mu \in \mathcal K \times \reals^N_{\geq 0}$.  
A general procedure for finding such a point is presented in \cref{alg:cmdp_alg}, where at each round the Lagrange multipliers and occupancy measure are updated by procedures $\mathrm{Alg}_\mu$ and $\mathrm{Alg}_\pi$, respectively. Because both $\mu$ and $d_\pi$ lie within convex sets and the objective is bilinear, standard results in game theory guarantee that when $\mathrm{Alg}_\mu$ and $\mathrm{Alg}_\pi$ are no-regret algorithms (e.g., gradient ascent/descent), the averaged iterates $(\frac{1}{K}\sum_{k=1}^K \mu^k, \frac{1}{K}\sum_{k=1}^K d_\pi^k)$ converge to a SP of the Lagrangian \citep{freund1997online}. Significantly, however, in general there is no guarantee regarding the \textit{last} iterates of the optimization---that is, there is no reason to expect that $(d_\pi^K, \mu^K)$ will be close to $(d_\pi^\star, \mu^\star)$. We formalize the distinction between average- and last-iterate convergence with the following definitions:

\begin{definition}[Average-Iterate Convergence (AIC)] \label{def:avg-iterate}
    Consider a SP problem $\min_{x\in\mathcal X} \max_{y\in\mathcal Y}\ \mathcal L(x, y)$ where $\mathcal X \subseteq \reals^n$ and $\mathcal Y \subseteq \reals^m$ with a nonempty set of equilibria $\mathcal Z^\star \subseteq \mathcal X \times \mathcal Y$. We say that a sequence $(x^1, y^1), (x^2, y^2), \dots, (x^K, y^K)$ displays \textit{average-iterate convergence} if $\left(\frac{1}{K}\sum_{k=1}^K x^k, \frac{1}{K}\sum_{k=1}^K y^k \right) \to (x^\star, y^\star) \in \mathcal Z^\star$ as $K\to\infty$. 
\end{definition}

\begin{definition}[Last-Iterate Convergence (LIC)]
    In the same setting as \cref{def:avg-iterate}, we say that a sequence $(x^1, y^1), (x^2, y^2), \dots, (x^K, y^K)$ displays \textit{last-iterate convergence} if $\left(x^K, y^K\right) \to (x^\star, y^\star) \in \mathcal Z^\star$ as $K\to\infty$.
\end{definition}

This discrepancy can be easily overlooked, but it has significant practical implications, as seen in \cref{fig:the_problem}. How can we address this? For inspiration, we look beyond RL for the moment to optimization theory. 



\begin{algorithm}[!t]
	\caption{Lagrange Optimization for Constrained MDPs}\label{alg:cmdp_alg}
		\begin{algorithmic}[1] 
		    \STATE Input: Lagrangian $\mathcal L: \mathcal K \times \reals_{\geq0}^N \to \reals$, $K\in\mathbb N_+$
            \FOR{$k=1,\dots,K$}
                \STATE $\mu^k = \mathrm{Alg}_\mu(d_\pi^1,\dots,d_\pi^{k-1}; \mathcal L)$
                \STATE $d_\pi^k = \mathrm{Alg}_\pi(\mu^1, \dots, \mu^{k-1}; \mathcal L)$ 
            \ENDFOR
            \STATE Return $d_\pi^K, \ \mu^K$ (last-iterates) or $\bar d_\pi^K = \frac{1}{K}\sum_{k=1}^K d_\pi^k$, $\bar \mu^K = \frac{1}{K}\sum_{k=1}^K \mu^k$ (average-iterates)
	\end{algorithmic}
\end{algorithm}

\section{Optimization in Min-Max Games} \label{sec:optimization}


The CMDP Langrangian min-max game belongs to a larger family of problems for which standard GD-style approaches (in fact, any algorithm in the broad family of follow-the-regularized-leader approaches) fail to converge in the last-iterate, typically displaying cyclic behavior \textit{around} an equilibrium rather than converging to it \citep{mertikopoulos2018optimistic}. 
Specifically, GD approaches are instances of primal-dual \textit{mirror descent} (MD), which for a generic SP problem with a differentiable objective $\min_{x\in\mathcal X}\max_{y\in\mathcal Y} \ \mathcal L(x, y)$ uses MD to update both $x$ and $y$:
\begin{align} \label{eq:md}
    x^{k+1} &= \argmin_{x \in \mathcal X}\ \langle \grad_x \mathcal L^k, x\rangle + \frac{1}{\eta^k} D_{\Omega_x}(x; x^k) \\
    y^{k+1} &= \argmax_{y\in \mathcal Y}\ \langle \grad_y \mathcal L^k, y\rangle - \frac{1}{\eta^k} D_{\Omega_y}(y; y^k),
\end{align}
where we use the notation $\grad\mathcal L^k \triangleq \grad \mathcal L(x^k, y^k)$ for brevity, $\eta^k$ is the step-size at iteration $k$, and
\begin{align}
    D_\Omega(u\text{;} v) = \Omega(u) - \Omega(v) - \langle \grad\Omega(v), u - v\rangle 
\end{align}
is the \textit{Bregman divergence} generated by a strictly convex, continuously differentiable function $\Omega(\cdot)$. Common choices for $\Omega$ include the squared Euclidean distance $\Omega(u) = \frac{1}{2}\|u\|^2$ which induces $D_\Omega(u\text{;}v) = \frac{1}{2}\|u - v\|^2$ and corresponds to standard GD, as well as the negative entropy $\Omega(u) = \langle u, \log u\rangle$, which induces $D_\Omega(u;v) = \kl[u||v]$, corresponding to multiplicative weights updating \citep[MWU;][]{grigoriadis1995mwu}.  When $\mathcal X$ and $\mathcal Y$ are convex and $\mathcal L(x,y)$ is convex in $x$ and concave in $y$, MD guarantees AIC, but may not converge in the last-iterate:
\begin{restatable}[Insufficiency of MD; \citep{daskalakis2018ogda}]{lemma}{mdbad} \label{thm:md_bad}
    There exist convex-concave SP problems for which primal-dual mirror descent does not achieve LIC.
\end{restatable}

\subsection{Optimistic Mirror Descent}
Fortunately, there exists 
an approach called \textit{optimistic mirror descent} (OMD) which has been shown to achieve LIC for convex-concave games \citep{daskalakis2018ogda,daskalakis2018omwu,daskalakis2018training_gans}. OMD makes the following simple modification to standard MD (changes in \textcolor{MidnightBlue}{blue}):
\begin{align*}
    x^{k+1} = \argmin_{x \in \mathcal X}\ \langle \textcolor{MidnightBlue}{2}\grad_x \mathcal L^k \textcolor{MidnightBlue}{ - \grad_x \mathcal L^{k-1}}, x^k\rangle + \frac{1}{\eta^k} D_{\Omega_x}(x; x^k).
\end{align*}
In other words, rather than use the gradient at iteration $k$, OMD uses the \textit{optimistic} gradient obtained by doubling the current gradient and subtracting the previous one. For concision, we'll denote the optimistic gradient of $\mathcal L$ at step $k$ by $\widetilde \grad \mathcal L^k  = 2\grad \mathcal L^k - \grad\mathcal L^{k-1} = \grad \mathcal L^k + (\grad \mathcal L^k - \grad\mathcal L^{k-1})$. In general, optimistic optimization attempts to use a ``hint'' about the next gradient to augment the current update. OMD can be seen as setting the hint to be the current gradient $\grad \mathcal L^k$ and then removing the previous hint $\grad \mathcal L^{k-1}$. It belongs to a broader family of so-called ``single-call'' extra-gradient methods \citep{hsieh2019single_call} that only requires one gradient estimate and a single projection into the constraint set at each step, making it particularly scalable to high-dimensional problems. For more discussion of this family of methods, see \cref{sect:extragradients}, and for other approaches to LIC in min-max games, see \cref{sec:additional_related}. As a demonstration, we applied gradient descent-ascent (GDA) and optimistic GDA (OGDA) to the simple bilinear problem $\min_{x\in\reals} \max_{y\in\reals} xy$, for which GDA can be shown to diverge for any positive learning rate \citep{daskalakis2018ogda}. (Note that (O)GD is (O)MD with $\Omega(\cdot) = \frac{1}{2}\|\cdot\|^2$.) While both achieve AIC, \cref{fig:minmax_xy} shows that that GDA's iterates diverge while OGDA converges to the SP at $(0, 0)$. To get an intuition for why OMD works, we can look to the loss surface in \cref{fig:minmax_xy}b. By adding the current hint $\textcolor{Salmon}{\grad \mathcal L^k}$ and subtracting the previous one $\textcolor{Orchid}{\grad\mathcal L^{k-1}}$ to the standard gradient update, the optimistic gradient $\textcolor{MidnightBlue}{\widetilde\grad\mathcal L^k}$ bends inwards towards the optimum. In contrast, we can see that $\textcolor{Salmon}{\grad\mathcal L^k}$ alone is always orthogonal to the radius pointing towards the optimum, and thus never converges to it. 

Anticipating the application of OMD to RL, a natural question when building on top of existing approaches is how much can be left unchanged. Since existing CRL approaches use standard gradients when performing primal-dual optimization or focus on stabilizing only the Lagrange multiplier \citep{calian2020metal,stooke2020pid}, it would simplify matters if we only required one player in \cref{alg:cmdp_alg} to use optimistic gradients---particularly $\mathrm{Alg}_\mu$, as this would allow the $\mathrm{Alg}_\pi$ player to simply implement a standard RL algorithm. Unfortunately, this is not the case:

\begin{figure}[!t]
    \begin{center}
    \centerline{\includegraphics[width=0.99\columnwidth]{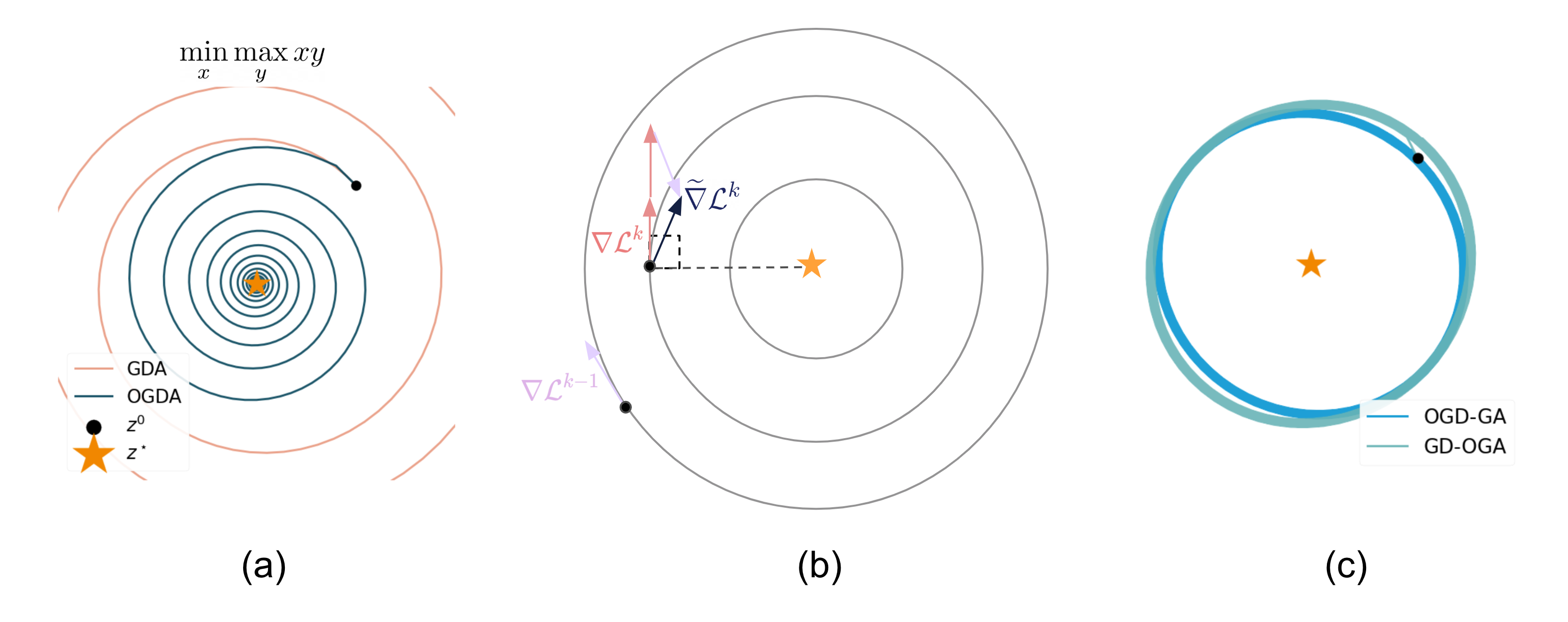}}
    \caption{Optimistic gradients achieve LIC. (a) GDA spirals outwards, while OGDA reaches the optimum. 
    (b) 
    The optimistic update $\textcolor{MidnightBlue}{\widetilde\grad\mathcal L^k}$ bends the trajectory inwards. (c) One optimistic player is not enough for LIC.}
    \label{fig:minmax_xy}
    \end{center}
    \vskip -0.4in
\end{figure}

\begin{restatable}[Being Singly-Optimistic is Not Enough]{lemma}{singlyoptimistic} \label{thm:singly_optimistic}
    There exists a convex-concave objective $\mathcal L$ for which no combination of one OMD and one MD player achieves LIC.  
\end{restatable}

All proofs are provided in \cref{sec:theory}. This result is reflected in \cref{fig:minmax_xy}c, where we can see that combining one GD player with one OGD player cycles rather than reaching LIC. 
We then need to show that OMD's guarantees carry over to the specific setting we require. Typically, convergence results for (O)MD set $\Omega_x = \Omega_y$, but in practice $\mathcal X$ and $\mathcal Y$ may be such that optimization is more natural with different MD algorithms for $x$ and $y$. For example, $\mathcal X$ may be the probability simplex, for which MWU is a natural approach, while $\mathcal Y$ may be the non-negative reals, for which projected GD is more appropriate. Because the OMD-inspired CRL algorithm we derive in \cref{sec:reload} is more naturally implemented with $\Omega_x \neq \Omega_y$, we'd like to show that MD achieves LIC in this setting. To do so, we turn to monotone operator theory \citep{bauschke2011monotone}.

\subsubsection{Analysis via Monotone Operators}
Here, we'll prove the LIC of a generalized form of OMD which permits different Bregman divergences by casting OMD as the application of several \textit{monotone operators}. In the next section, we'll show that the CMDP Lagrangian can be solved by this method. 
We first provide a few definitions:
\begin{definition}[Set-Valued Operator]
    An operator $F$ on a Hilbert space $\mathcal H$ is said to be \textit{set-valued} if $F$ maps a point in $\mathcal H$ to subset of $\mathcal H$. We denote this by $F: \mathcal H \rightrightarrows \mathcal H$. 
\end{definition}
\begin{definition}[Graph]
    The \textit{graph} of an operator $F$ is
    \begin{align*}
        \mathrm{Gra} \ F = \{(x, u) \mid u \in F(x) \} \subseteq \mathcal H \times \mathcal H.
    \end{align*}
\end{definition}
\begin{definition}[Monotone Operator]
    An operator $F$ on a Hilbert space $\mathcal H$ is said to be $m$-\textit{strongly monotone} if 
    \begin{align*}
        \langle F(x_1) - F(x_2), x_1 - x_2 \rangle \geq \frac{m}{2}\|x_1 - x_2\|^2 \quad \forall x_1, x_2 \in \mathcal H,
    \end{align*}
    with $m > 0$. If the inequality holds with $m=0$, $F$ is simply called monotone.
    $F$ is \textit{maximal monotone} if there is no other monotone operator $G$ such that $\mathrm{Gra} \ F \subset \mathrm{Gra} \ G$. 
\end{definition}

\paragraph{Mirror Descent as Fixed Point Iteration} 
Let $\ell(x) \triangleq \mathcal L(x, y^k)$. Then \cref{eq:md} can be written as
\begin{align*}
    x^{k+1} 
    &= \argmin_{x\in\reals^d}\  \langle \grad\ell(x^k), x \rangle + \frac{1}{\eta^k} D_\Omega(x\text{;} x^k) + \mathbb I_\mathcal X(x),
\end{align*}
where $\mathbb I_\mathcal X(x)$ is the indicator function which equals $0$ if $x\in\mathcal X$ and $+\infty$ if $x \notin\mathcal X$.
Because this problem is convex, we only need a first-order optimality condition, and it is equivalent to solving the following inclusion problem:
\begin{align*}
\begin{split}
    &0 \in \grad\ell(x^k) + \frac{1}{\eta^k}(\grad\Omega(x) - \grad\Omega(x^k)) + N_\mathcal X(x) \\
    &\iff \grad\Omega(x^k) - \eta^k \grad\ell(x^k) \in (\grad\Omega + \eta^k N_\mathcal X)(x) \\
    &\iff x \in \mathrm{Prt}_{\eta^k N_\mathcal X}^\Omega(\grad\Omega(x^k) - \eta^k \grad\ell(x^k)), 
\end{split}
\end{align*}
where $N_\mathcal X = \partial \mathbb I_\mathcal X$ is the normal cone operator for $\mathcal X$ and $\mathrm{Prt}_{\eta^k N_\mathcal X}^\Omega = (\grad\Omega + \eta^k N_\mathcal X)^{-1}$ is the \textit{proto-resolvent} of $\eta^k N_\mathcal X$ relative to $\Omega$ \citep{reich2011bregman}. Thus, mirror descent can be seen as performing fixed point iteration:
\begin{align}
    x^{k+1} = \mathrm{Prt}_{\eta^k N_\mathcal X}^\Omega(\grad\Omega - \eta^k \grad\ell)(x^k).
\end{align}

For optimistic mirror descent, we write the update as
\begin{align*}
    x^{k+1} = \mathrm{Prt}_{\eta^k N_\mathcal X}^\Omega(&\grad\Omega(x^k) - \eta^k\grad\ell(x^k) \\  &- \eta^{k-1}(\grad\ell(x^k) - \grad\ell(x^{k-1})) ).
\end{align*}
In general, we can replace $N_\mathcal X$ with any maximal monotone operator $A$ and $\grad \ell$ with any monotone and $L$-Lipschitz operator $B$:
\begin{align} \label{eq:omd_monotone}
\begin{split}
    x^{k+1} = \mathrm{Prt}_{\eta^k A}^\Omega(&\grad\Omega(x^k) - \eta^kB(x^k)  \\ &- \eta^{k-1}(B(x^k) - B(x^{k-1}))).
\end{split}
\end{align}
When $\Omega(\cdot) = \frac{1}{2}\|\cdot\|^2$, this is \textit{forward-reflected-backward splitting} \citep{malitsky2020forb}. We now generalize the convergence result of \citet{malitsky2020forb} to OMD 
for split monotone inclusion problems of the form:
\begin{align} \label{eq:monotone_inc}
    \mathrm{find }\  x \in \mathcal H \quad \mathrm{s.t.} \quad 0 \in (A + B)(x).
\end{align}

\begin{restatable}[Convergence]{theorem}{convergence} \label{thm:convergence}
    Let $A: \mathcal H \rightrightarrows \mathcal H$ be maximal monotone and let $B: \mathcal H \to \mathcal H$ be monotone and $L$-Lipschitz and suppose that $(A+B)^{-1}(0) \neq \varnothing$. Suppose that $(\eta^k) \subseteq [\varepsilon, \frac{1 - 2\varepsilon}{2L}]$ for some $\varepsilon > 0$. Given $x^0, x^{-1} \in \mathcal H$, define the sequence $(x^k)$ according to \cref{eq:omd_monotone} with $\Omega$ $\sigma$-strongly convex, $\sigma\geq 1$. Then $(x^k)$ converges weakly to a point contained in $(A + B)^{-1}(0)$. 
\end{restatable}

\begin{restatable}[Convergence Rate]{theorem}{rate} \label{thm:rate}
    In the setting of \cref{thm:convergence}, if $A$ or $B$ is $m$-strongly monotone and $\Omega$ has an $\Lo$-Lipschitz continuous gradient, then $(x^k)$ converges at a rate $\mathcal O(1/\alpha^k)$ to the unique element $x^\star \in (A + B)^{-1}(0)$, where $\alpha>1$ is a constant.
\end{restatable}
In the following section, we'll show that this approach can be adapted to the RL setting to achieve LIC for CMDPs.


\section{Applying Optimistic Optimization to RL} \label{sec:reload}
Next, we will show that the guarantees we developed in the previous section apply to OMD in CMDPs. We begin by writing the CMDP Lagrangian problem as follows:
\begin{align}
    \min_{d_\pi \in \reals^{|\mathcal S||\mathcal A|}} \max_{\mu\in \reals^{N}} \ \mathbb I_{\mathcal K}(d_\pi) + \mathcal L(d_\pi, \mu) + \mathbb I_{\mathbb R_{\geq 0}^N}(\mu).
\end{align}
 We can then express the CMDP problem in the form of \cref{eq:monotone_inc} by noting that a SP must satisfy: 
\begin{align} \label{eq:saddlepoint_monotone}
    \mathrm{find }\  \begin{bmatrix} d_\pi \\ \mu \end{bmatrix}  \  \mathrm{s.t.} \ 
    \begin{bmatrix} 0 \\ 0 \end{bmatrix} \in \begin{bmatrix}
        \partial  \mathbb I_{\mathcal K}(d_\pi) \\
        \partial \mathbb I_{\mathbb R_{\geq 0}^N}(\mu)
    \end{bmatrix}
    + \begin{bmatrix}
        \grad_{d_\pi} \mathcal L(d_\pi, \mu) \\
        -\grad_\mu \mathcal L(d_\pi, \mu)
    \end{bmatrix}
\end{align}
Define ${\grad_{d_\pi} \mathcal L} = - r_0 + \sum_{n=1}^N \mu_n  r_n \triangleq  r_\mu$ as the \textit{mixed reward vector} for Lagrange multipliers $ \mu$. Substituting this and ${\grad_{\mu} \mathcal L}^k =  v_{1:N} -  \theta$ into the OMD updates above yields:
\begin{align}
\begin{split} \label{eq:convex_reload}
    {d_\pi}^{k+1} &= \argmin_{{d_\pi}\in \mathcal K} \ \langle {\tilde r_\mu}^k , {d_\pi} \rangle + \frac{1}{\eta} D_{\Omega_\pi}({d_\pi}\text{;} {d_\pi}^k)  \\
    {\mu}^{k+1} &= \argmax_{\mu \geq  0} \ \left\langle {\tilde v}_{1:N}^k -  \theta,  \mu \right\rangle - \frac{1}{\eta} D_{\Omega_\mu}(\mu\text{;} \mu^k).
\end{split}
\end{align}
where ${\tilde r_\mu}^k \triangleq 2{r_\mu}^k - {r_\mu}^{k-1}$ and ${\tilde v}_{1:N}^k \triangleq 2{ v}_{1:N}^k - { v}_{1:N}^{k-1}$.
Hereafter, we'll refer to this general family of approaches (determined by different choices of $\Omega_\pi$ and $\Omega_\mu$) as \textit{Reinforcement Learning with Optimistic Ascent-Descent} (ReLOAD). Based on the above, we can guarantee the LIC of ReLOAD.
\begin{restatable}[ReLOAD Convergence]{corollary}{convergenceconvexreload} \label{thm:reload_convergence_fixed}
    The sequence $((d_\pi^k, \mu^k))$ generated by \cref{eq:convex_reload} converges in the last-iterate for $\eta \in (0, 1 /2)$. 
\end{restatable}

At first, glance, one might think that \cref{thm:reload_convergence_fixed} violates the Scalarization Fallacy \cite{szepesvari2020cmdps},\citep[Lemma 1]{zahavy2021reward}---as $\mu^k$ converges to $\mu^\star$, the policy is optimizing an increasingly stationary reward. However, this is not the case: $(d_\pi^k, \mu^k)$ are jointly converging (last-iterate) to the optimal SP $(d_\pi^\star, \mu^\star)$. This implies that $\pi^\star$ is an optimal policy w.r.t to the $\mu^\star$-weighted reward $r_{\mu^\star}$. However, there might exist other policies that are optimal w.r.t to the $r_{\mu^\star}$ that are not in Nash equilibrium with $\mu^\star$. ReLOAD is guaranteed to converge in last iterate to $\pi^\star$ and not to these other policies. An algorithm that maximizes the stationary reward $r_{\mu^\star}$, on the other hand, will be optimal w.r.t to $r_{\mu^\star}$ but will not necessarily return $\pi^\star$ and therefore will not be in Nash equilibrium with $\mu^
\star$. We revisit this observation in our experiments, with \cref{fig:osc} showing that simply optimizing $r_{\mu^\star}$ fails to match ReLOAD's performance.  
\begin{figure*}[!t]
    \centering
    \includegraphics[width=0.99\textwidth]{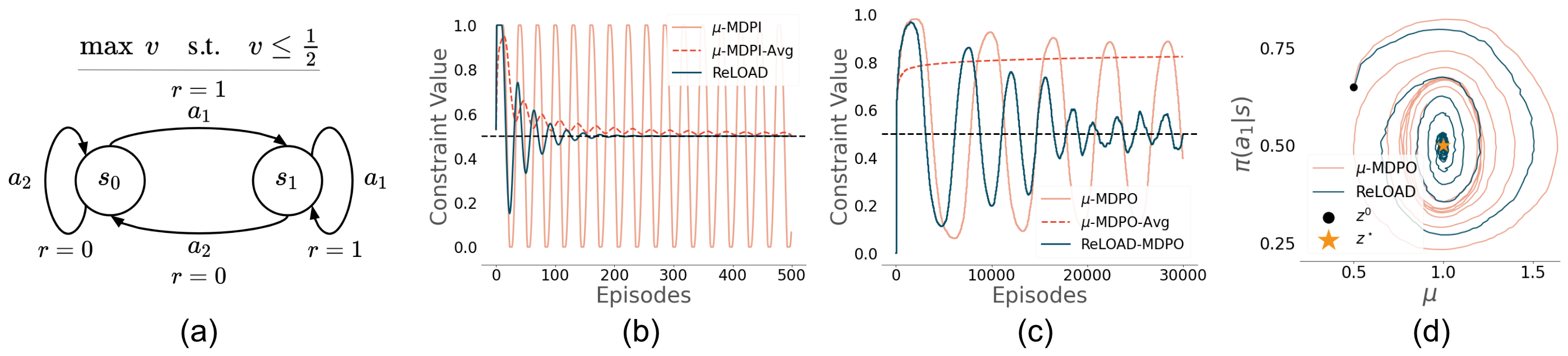}
    \caption{Optimization in a simple CMDP. (a) Schematic of a CMDP whose constraint and reward conflict. (b) Adding ReLOAD to tabular policy iteration damps oscillations. (c) ReLOAD also damps oscillations with function approximation, a setting in which averaging parameters no longer works. (d) Unlike the baseline, ReLOAD successfully converges to the SP with function approximation. }
    \label{fig:toy}
    \vspace{-3mm}
\end{figure*}
\subsection{Policy-Based ReLOAD} \label{sec:tabular_reload}
In larger settings, it is easier to optimize policies than occupancy measures. Accordingly, virtually all scalable RL algorithms either learn a policy directly or define one implicitly, e.g., via $q$-learning, such that the resulting occupancy measure is guaranteed to lie in the feasible set. Rewriting the value in terms of the policy $v_\pi = \langle  r,  d_\pi \rangle = \langle  q_\pi,  \pi \rangle$ is not convex in $\pi$. Nevertheless, policy optimization methods based on non-convex relaxations of convex objectives often converge to the optimal solution \citep{shani2020mdpotheory} and in particular, OMD has been shown to reach first-order---and occasionally global---equilibria in non-convex settings \citep{cai2022accelerated}. We can thus rewrite the Lagrangian as 
\begin{align*}
\begin{split}
    \mathcal L( d_\pi, \mu) &= \big\langle - r_0 + \sum_{n} \mu_n  r_n,  d_\pi \big\rangle - \langle  \mu,  \theta \rangle \\
    &= \big\langle - q_0^\pi + \sum_{n} \mu_n  q_n,  \pi \big\rangle - \langle \mu, \theta \rangle = \mathcal L(\pi, \mu),
\end{split}
\end{align*}
where $ q_n \triangleq  q_{\pi,r_n}$. We note that in the non-parametric setting, $\grad_\pi \mathcal L = - q_0 + \sum_{n=1}^N \mu_n  q_n \triangleq  q_\mu$. In other words, gradient estimation is equivalent to policy evaluation, resulting in the mixed $q$-value $ q_\mu$. As before, the gradient with respect to the Lagrange multipliers is the vector of constraint violations: $\grad_\mu \mathcal L =  v_{1:N} -  \theta$. When dealing with probability measures, a natural choice for $\Omega_\pi$ is the negative entropy $\Omega_\pi( u) = \langle  u, \log  u \rangle$, 
and setting $\Omega_\mu( u)=\frac{1}{2}\| u\|_2^2$, 
Applying OMD, ReLOAD then performs the following updates:
\begin{align*} 
\begin{split}
    \label{eq:tabular_reload_policy}
     \pi^{k+1} &= \argmin_{ \pi\in\Pi} \ -\langle \textcolor{MidnightBlue}{2} q_\mu^k \textcolor{MidnightBlue}{-\  q_{\mu}^{k-1}}, \pi \rangle + \frac{1}{\eta} \kl[ \pi ||  \pi^k] \\
    &= \frac{ \pi^k \exp\left( ( \textcolor{MidnightBlue}{2} q_\mu^k \textcolor{MidnightBlue}{-\  q_\mu^{k-1}})/\eta_\pi^k \right)}{\left\langle  \pi^k \exp\left( (\textcolor{MidnightBlue}{2} q_\mu^k \textcolor{MidnightBlue}{-\  q_\mu^{k-1}})/\eta_\pi^k \right), \mathbf 1 \right\rangle} 
    \\
     \mu^{k+1} &= \argmax_{\mu\geq  0} \ \langle \textcolor{MidnightBlue}{2} v^k \textcolor{MidnightBlue}{-\  v^{k-1}} - \theta,  \mu \rangle - \frac{1}{2\eta} \| \mu -  \mu^k\|_2^2 \\
     &= \max\{ \mu^k + \eta_\mu^k (\textcolor{MidnightBlue}{2}  v^k_{1:N} \textcolor{MidnightBlue}{-\  v^{k-1}_{1:N}} -  \theta), \  0 \},
\end{split}
\end{align*}
with the full algorithm presented in \cref{alg:tabular_reload}.
Here, accurate policy evaluation is especially important, as estimating $q$ amounts to computing the gradient of the Lagrangian.

\subsection{ReLOAD with Function Approximation} \label{sec:reload_func_approx}
For large state and action spaces, 
function approximation, i.e., deep RL (DRL), becomes preferable. 
To implement ReLOAD using DRL, we require the following ingredients: (1) optimistic value estimates and (2) a trust region for the policy update. Fortunately, both requirements are easy to satisfy for most state-of-the-art policy optimization methods. For (1), a crucial factor is that with function approximation, rather than compute $q$-value estimates for every state-action pair in the environment, policy evaluation is performed over minibatches of sampled experience. This means that the agent can't simply store past gradients to compute optimistic values, as those gradients may have been obtained from the values of different $(s,a)$ pairs from those used to compute the current gradient. Instead, the agent must maintain a copy of the previous value network so that the optimistic $q$-values can be computed from the same samples: $\tilde q^k_\mu(s,a) = 2q^k_\mu(s,a) - q^{k-1}_\mu(s,a)$. Similarly, $\mu$ must also be updated using value estimates from the same data. This is a complication which makes using OMD directly, as optimistic GAN methods do \cite{daskalakis2018training_gans}, nontrivial for RL. As for (2), many high-performing policy optimization algorithms employ trust regions as a means of stabilizing learning and improving sample efficiency. Examples include TRPO \citep{schulman2015trpo}, MDPO \citep{tomar2021mdpo}, and MPO \citep{abdolmaleki2018mpo}. 


\section{Experiments} \label{sec:experiments_main}

Next, we study the LIC of  ReLOAD empirically on a variety of CMDPs with discrete and continuous state and action spaces. We augmented ReLOAD with popular DRL algorithms including: MD Policy Iteration \citep[MDPI;][]{geist2019mdpi}, MD Policy Optimization \citep[MDPO;][]{tomar2021mdpo}, IMPALA \citep{espeholt2018impala}, and distributional MPO \citep[DMPO;][]{abdolmaleki2020dmpo}. \textbf{Notation:} In the following, when augmenting unconstrained methods to perform Lagrangian optimization, we prefix the method name with ``$\mu$-''. 
For more detail on all algorithms, see \cref{sect:algorithm_details}. To measure performance, we use the \textit{weighted reward}: the task reward minus the multiplier-weighted constraint overshoot, given by
$
r_0 - \sum_{n=1}^N \hat\mu_n^\star \max\{r_n - \theta_n, 0\},
$
where $\hat\mu_n^\star$ is the normalized optimal Lagrange multiplier \citep{stooke2020pid}. 
$\hat\mu^\star$ was calculated by averaging the Lagrange multipliers learned by the non-optimistic baseline agents (details in \cref{sect:experiment_details}). Experiments were repeated over 8 random seeds, and error bars denote one standard error. Additional plots and result tables can be found in \cref{sect:experiment_results}. Videos of trained agents can be found at 
\url{https://tedmoskovitz.github.io/ReLOAD/}.

\textbf{Toy Example: A Paradoxical CMDP} In many real-world applications of CMDPs, constraints are introduced 
to ensure the system's integrity is maintained. For example, a robot may be trained to run as fast as possible but restrained so that it does not place sufficient torque on its joints to break them. To capture these conflicting goals in a simple setting, we tested tabular ReLOAD-MDPI (\cref{alg:tabular_reload}) on the two-state CMDP in \cref{fig:toy}a. In this task, there is a single constraint reward which is \textit{equal} to the primary reward $r_0 = r_1 = r$, so that $r=1$ when the agent takes action $a_1$ placing it in $s_1$, and $r=0$ for action $a_2$ which moves the agent to $s_2$. The constraint $\theta = 1/2$ was chosen so that the agent may only choose $a_1$ at most half of the time. Plotting the value over the course of learning in \cref{fig:toy}b, we can see that ReLOAD converges, while $\mu$-MDPI oscillates and fails to converge in the last-iterate. However, this approach does achieve AIC (``$\mu$-MDPI-Avg''). 
\begin{figure}[ht]
    \begin{center}
    \centerline{\includegraphics[width=0.55\columnwidth]{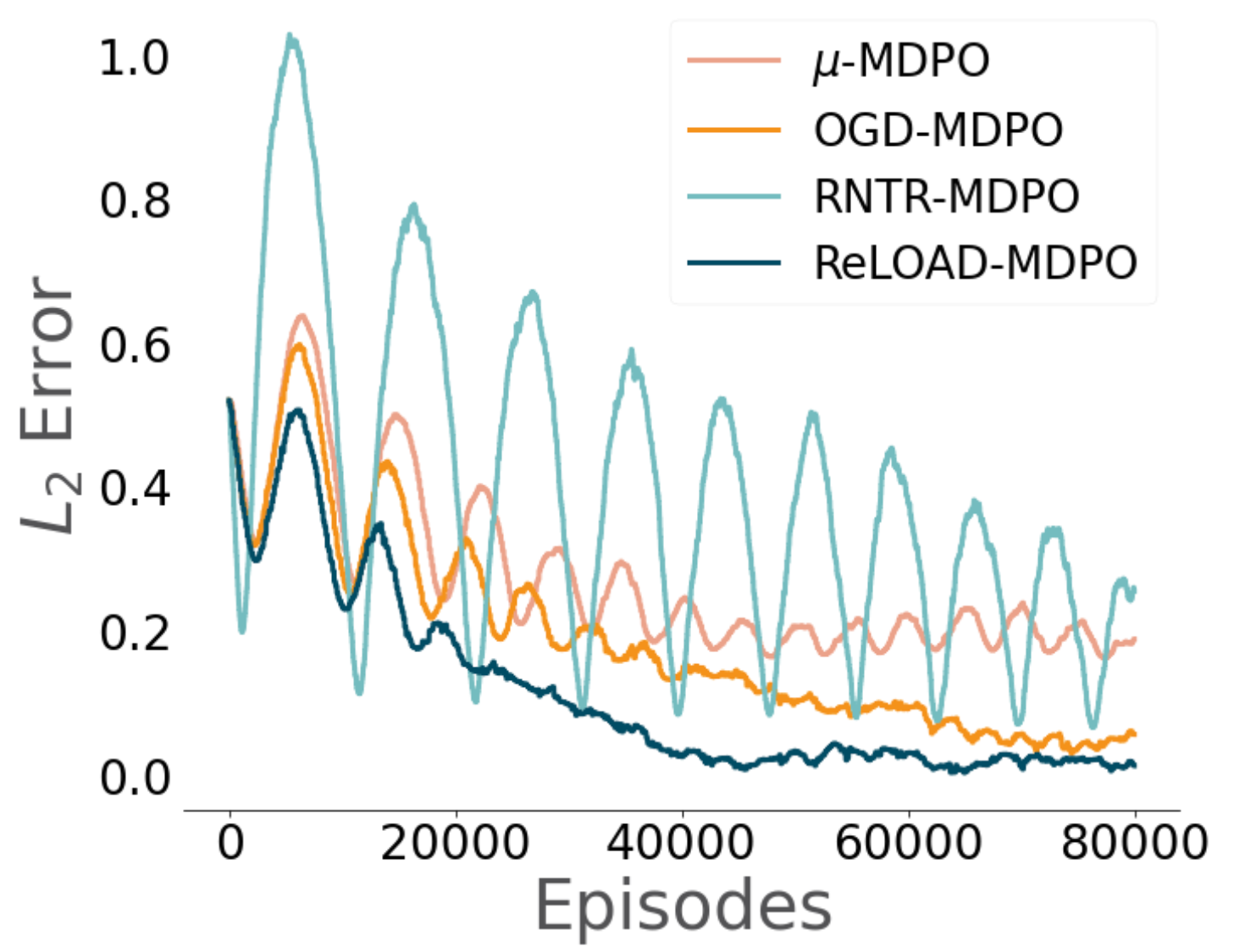}}
    \vspace{-3mm}
    \caption{Ablations in the toy CMDP. Both $\mu$-MDPO and RNTR-MDPO fail to converge. However, OGD nearly matches ReLOAD: there are only two states, so using the gradient from the previous minibatch will be close to the current gradient. }
    \label{fig:toy_ablations}
    \end{center}
    \vskip -0.2in
\end{figure}
To test performance with function approximation, we then applied ReLOAD-MDPO and $\mu$-MDPO to the same problem. As in the tabular case, ReLOAD significantly dampens oscillations compared to its non-optimistic counterpart (\cref{fig:toy}c). (However, some noise remains due to noise in the approximate policy evaluation.) Importantly, $\mu$-MDPO-Avg does not converge in this setting, as it corresponds to averaging the parameters of a nonlinear network. Examining the optimization trajectories in \cref{fig:toy}d, we can see that ReLOAD-MDPO converges to the SP, while $\mu$-MDPO gets ``stuck,'' circling but never reaching the optimum. 
In \cref{fig:toy_ablations}, we compare ReLOAD-MDPO against various ablations by measuring the $L_2$ distance from the SP over the course of training. Agents which use OGD directly on the policy parameters are prefixed by ``OGD-'', and ReLOAD with no trust region is denoted by ``RNTR-''. Both $\mu$-MDPO and RNTR-MDPO fail to converge. However, performing OGD directly on the policy parameters, rather than via optimistic value estimates, performs nearly as well as ReLOAD. This is because the CMDP only has two states, so the gradient computed from the previous minibatch will be close to the current gradient. This difference becomes significant on large-scale problems.

\textbf{Catch.} 
We then applied $\mu$-MDPO and ReLOAD-MDPO to a constrained version of Bsuite's Catch \citep{osband2019bsuite}. In the standard version of this task, the agent moves a paddle left or right to catch a falling ball. 
To convert this problem into a CMDP, we added a constraint reward $r_1$ which was 0.2 in the leftmost three columns of the environment and 0 elsewhere, with the constraint $\theta_1 = 1.0$. 
To both obey the constraint and catch the ball, the agent could only effectively make one trip per episode to the left side of the arena. In \cref{fig:catch}a, ReLOAD successfully damps oscillations and achieves LIC, while $\mu$-MDPO only achieves AIC. As an illustration of the consequences, we can see that when the Lagrange multiplier spikes upwards, indicating a constraint violation, the agent successfully catches the ball but lingers on the left side of the environment and thus violates the constraint (\cref{fig:catch}b, top). Conversely, a drop in the Lagrange multiplier indicates that the while the constraint is satisfied, performance suffers---the agent simply ignores balls falling on the left side of the arena (\cref{fig:catch}b, bottom). ReLOAD 
learns to stay to the right until the last moment, catching the ball while obeying the constraint (\cref{fig:catch}b, middle).

\begin{figure}[ht]
    \begin{center}
    \centerline{\includegraphics[width=0.99\columnwidth]{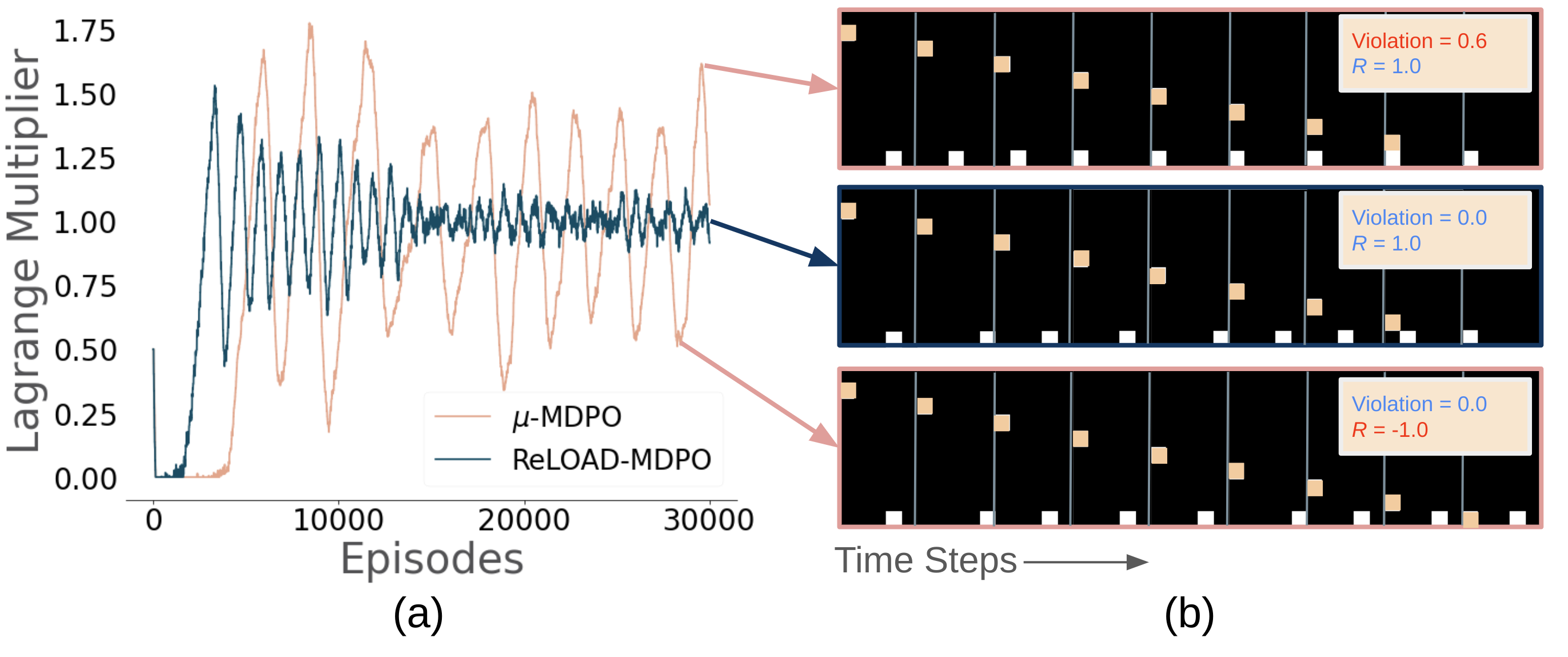}}
    \vspace{-3mm}
    \caption{LIC in Catch. (a) ReLOAD reduces oscillations. (b) ReLOAD catches the ball and obeys the constraint, while the standard method only does one or the other.}
    \label{fig:catch}
    \end{center}
    \vskip -0.2in
\end{figure}

\textbf{The Real-World RL Suite}  \citep[RWRL;][]{dulacarnold2019rwrl} is a collection of DeepMind Control Suite tasks modified with constraints as well as a variety of other real-world challenges which has become a benchmark for applied RL agents \citep{dulac2020empirical,huang2022lp3,calian2020metal,brunke2022rwrl}.
\begin{figure}[ht]
    \begin{center}
    \centerline{\includegraphics[width=0.99\columnwidth]{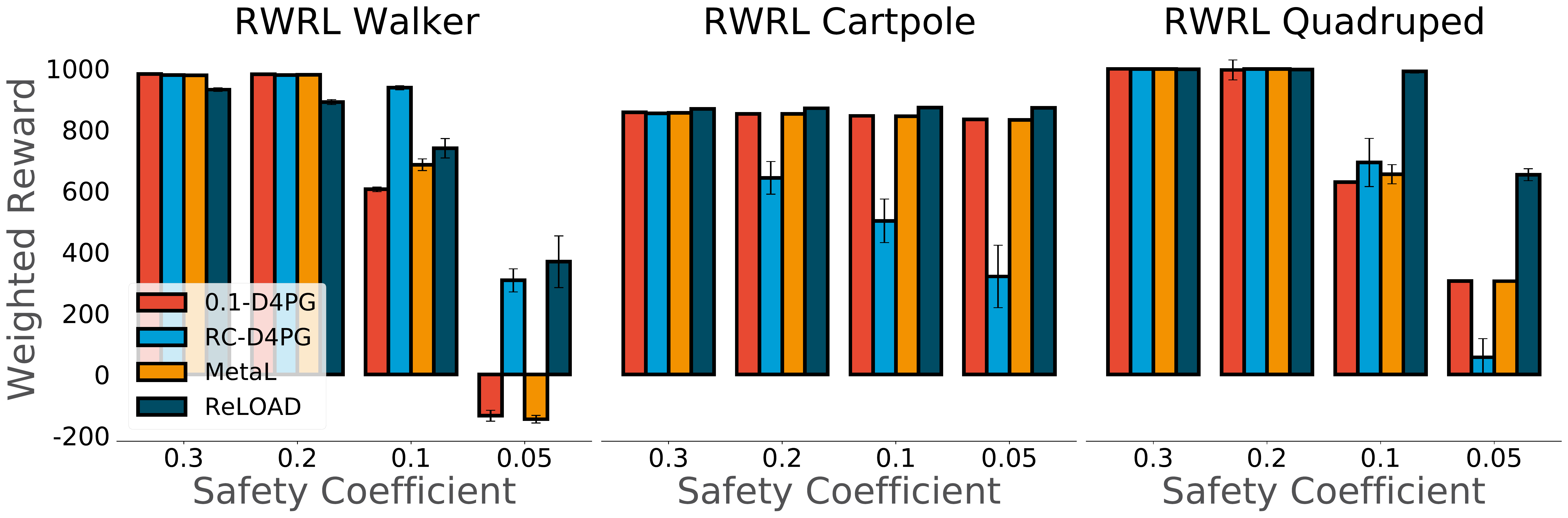}}
    \caption{ReLOAD outperforms baselines on the RWRL Suite for the most challenging safety coefficients (lower = harder).}
    \label{fig:rwrl_weighted_reward}
    \end{center}
    \vskip -0.3in
\end{figure}
\begin{figure*}[!t]
    \centering
    \includegraphics[width=0.95\textwidth]{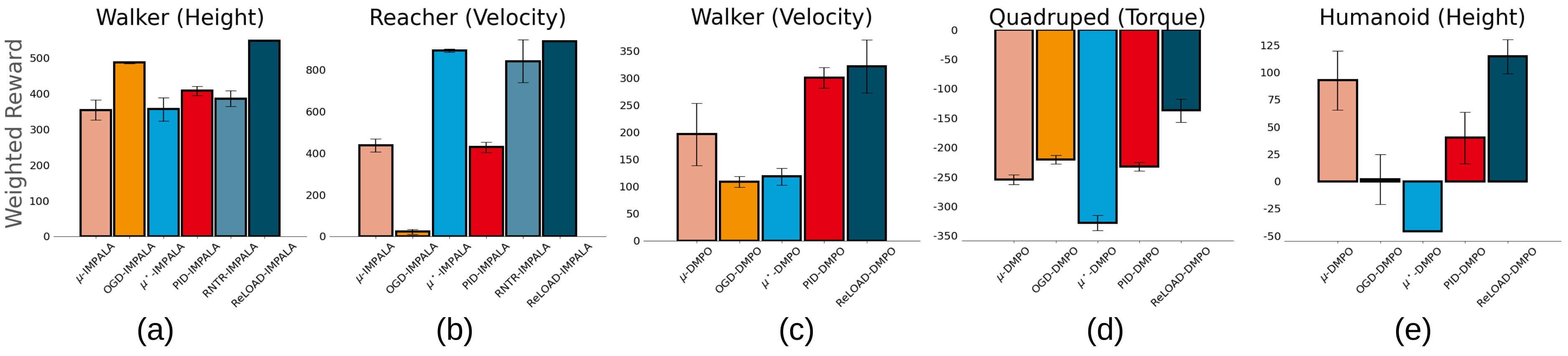}
    \caption{ReLOAD (dark blue) achieves the strongest performance on a variety of CMDPs which induce oscillations in control suite.}
    \label{fig:osc}
\end{figure*}

In addition to the choice of constraint threshold, each task in the RWRL Suite has a \textit{safety coefficient}, where low values of this coefficient indicate that it's harder to satisfy the constraints. 
We trained ReLOAD-DMPO on three challenging tasks: \texttt{RWRL-Walker}, \texttt{RWRL-Cartpole}, and \texttt{RWRL-Quadruped} across the same three constraint thresholds and four safety coefficients for each task used by \citet{calian2020metal}. As baselines, we applied MetaL, a tuned, fixed-Lagrange method (0.1-D4PG), and a primal-dual D4PG variant similar to $\mu$-D4PG (RC-D4PG), all as in \citet{calian2020metal}. We also compared ReLOAD against $\mu$-DMPO and multi-objective DMPO \citep{huang2022lp3} on the easier safety coefficient settings used by \citet{huang2022lp3} (see \cref{sect:algorithm_details} and \cref{fig:rwrl_lp3} for further details). As we can see in \cref{fig:rwrl_weighted_reward}, ReLOAD solves nearly all the CMDPs at least as well as the baselines and outperforms them for the most challenging safety coefficients. Interestingly, we found that the benchmark constraint thresholds introduced by \citet{calian2020metal} were selected to be extreme so as to avoid oscillations. The RWRL suite therefore serves as a useful sanity check that ReLOAD performs strongly even without the threat of oscillations, but we would still like to test it on high-dimensional CMDPs which carry this threat.

\textbf{Oscillating Control Suite.} 
\begin{figure*}[!t]
    \centering
    \includegraphics[width=0.85\textwidth]{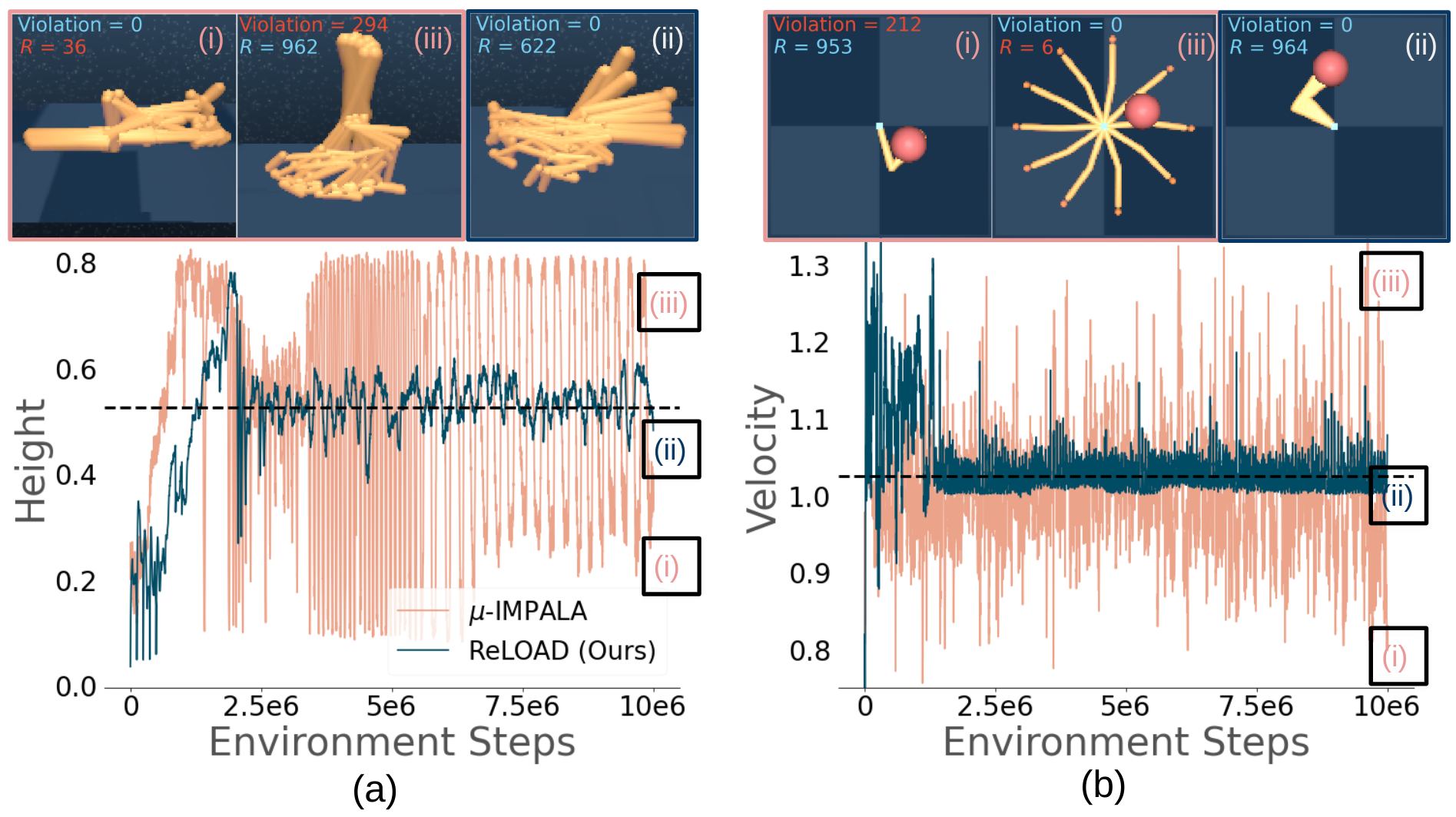}
    \caption{ReLOAD significantly damps oscillations, resulting in agents which perform the desired task while obeying constraints.}
    \label{fig:the_solution}
\end{figure*}
Finally, we identified tasks and constraint settings in DeepMind Control Suite which cause standard agents to oscillate. We trained ReLOAD on the following tasks: \texttt{Walker, Walk} with a constraint on the height of the agent, \texttt{Reacher, Easy} with a velocity constraint, \texttt{Walker, Walk} with a velocity constraint, \texttt{Quadruped, Walk} with a constraint on the torque applied to its joints, and \texttt{Humanoid, Walk} with a height constraint. We call this set of tasks and constraints, whose details can be found in \cref{sect:experiment_details}, the \textit{Oscillating Control Suite}, and we believe it can serve as a challenging benchmark of CMDPs. To test its generality, we paired ReLOAD with both IMPALA and DMPO base agents,
and compared it against the corresponding $\mu$- agent, OGD, an agent which uses the fixed optimal Lagrange multiplier obtained by averaging the final iterates of the associated $\mu$- across seeds ($\mu^\star$-), and PID control \citep[PID-;][]{stooke2020pid}. For IMPALA, we also tested RNTR. We found that in all cases, ReLOAD achieves higher average weighted reward at the end of training than the baselines (\cref{fig:osc}). Importantly, the performance gap between ReLOAD and OGD is much greater than in \cref{fig:toy_ablations}, as the data distribution varies more significantly from batch to batch in high-dimensional problems.  \cref{fig:the_solution} depicts example training curves and final behaviors for $\mu$-IMPALA and ReLOAD-IMPALA on \texttt{Walker} 
and \texttt{Reacher}, with curves for other domains in Appendix \cref{fig:dmpo_curves}. 
We can see that ReLOAD significantly dampens oscillations. For \texttt{Walker} (\cref{fig:the_solution}a), it produces an agent which moves forward with a modified, kneeling walk (panel (ii)), while $\mu$-IMPALA typically either ends up lying down (panel (i))---thus obeying the constraint but not performing the task---or walking normally and ignoring the constraint (panel (iii)). We see a similar pattern with \texttt{Reacher} (\cref{fig:the_solution}b), with the ReLOAD agent moving quickly while keeping the tip of its arm in the rewarded area (panel (ii)), while $\mu$-IMPALA either stops moving within the rewarded area (panel (i)) or maximizes velocity while swinging in a circle and ignoring the task (panel (iii)).

\section{Conclusion}
In this work, we introduced ReLOAD, an RL framework for LIC in constrained MDPs. We examined the challenges in achieving LIC in this setting from a formal perspective, derived a convergence guarantee for a generalized form of OMD for whom a special case is the convex formulation of ReLOAD, and demonstrated ReLOAD's strong empirical performance on a range of challenging CMDPs. One shortcoming of the current analysis is a lack of theoretical understanding of non-convex ReLOAD, and in the future we'd like to experiment with more constraints. We believe that ReLOAD may offer insights into policy optimization with non-stationary rewards more generally, as well as the oscillations which are known to plague standard RL combined with function approximation \citep{young2020oscillations,gopalan2022approximate}, improving optimization within the primal-dual formulation of RL \citep{neu2021logistic}, and extensions to convex MDPs \citep{zahavy2021reward}.

\paragraph{Acknowledgements} Work funded by DeepMind. The authors would like to thank Abbas Abdolmaleki, Steven Bohez, Edouard Leurent, Daniel Mankowitz, Dan Calian, Lior Shani, Yash Chandak, Chris Lu, Robert Lange, DJ Strouse, Jack Parker-Holder, Kate Baumli, Kevin Waugh, and other colleagues on the Discovery team and at DeepMind for helpful discussions and feedback over the course of this project.

\newpage





\bibliography{main}

\begin{thebibliography}{90}
\providecommand{\natexlab}[1]{#1}
\providecommand{\url}[1]{\texttt{#1}}
\expandafter\ifx\csname urlstyle\endcsname\relax
  \providecommand{\doi}[1]{doi: #1}\else
  \providecommand{\doi}{doi: \begingroup \urlstyle{rm}\Url}\fi

\bibitem[Abbeel \& Ng(2004)Abbeel and Ng]{abbeel2004apprenticeship}
Abbeel, P. and Ng, A.~Y.
\newblock Apprenticeship learning via inverse reinforcement learning.
\newblock In \emph{Proceedings of the twenty-first international conference on
  Machine learning}, pp.\ ~1, 2004.

\bibitem[Abdolmaleki et~al.(2018)Abdolmaleki, Springenberg, Tassa, Munos,
  Heess, and Riedmiller]{abdolmaleki2018mpo}
Abdolmaleki, A., Springenberg, J.~T., Tassa, Y., Munos, R., Heess, N., and
  Riedmiller, M.
\newblock Maximum a posteriori policy optimisation, 2018.
\newblock URL \url{https://arxiv.org/abs/1806.06920}.

\bibitem[Abdolmaleki et~al.(2020)Abdolmaleki, Huang, Hasenclever, Neunert,
  Song, Zambelli, Martins, Heess, Hadsell, and Riedmiller]{abdolmaleki2020dmpo}
Abdolmaleki, A., Huang, S., Hasenclever, L., Neunert, M., Song, F., Zambelli,
  M., Martins, M., Heess, N., Hadsell, R., and Riedmiller, M.
\newblock A distributional view on multi-objective policy optimization.
\newblock In III, H.~D. and Singh, A. (eds.), \emph{Proceedings of the 37th
  International Conference on Machine Learning}, volume 119 of
  \emph{Proceedings of Machine Learning Research}, pp.\  11--22. PMLR, 13--18
  Jul 2020.
\newblock URL \url{https://proceedings.mlr.press/v119/abdolmaleki20a.html}.

\bibitem[Abernethy et~al.(2021)Abernethy, Lai, and Wibisono]{abernethy2021last}
Abernethy, J., Lai, K.~A., and Wibisono, A.
\newblock Last-iterate convergence rates for min-max optimization: Convergence
  of hamiltonian gradient descent and consensus optimization.
\newblock In \emph{Algorithmic Learning Theory}, pp.\  3--47. PMLR, 2021.

\bibitem[Achiam et~al.(2017)Achiam, Held, Tamar, and
  Abbeel]{achiam2017constrained}
Achiam, J., Held, D., Tamar, A., and Abbeel, P.
\newblock Constrained policy optimization.
\newblock In Precup, D. and Teh, Y.~W. (eds.), \emph{Proceedings of the 34th
  International Conference on Machine Learning}, volume~70 of \emph{Proceedings
  of Machine Learning Research}, pp.\  22--31. PMLR, 06--11 Aug 2017.
\newblock URL \url{https://proceedings.mlr.press/v70/achiam17a.html}.

\bibitem[Agarwal et~al.(2021)Agarwal, Kakade, Lee, and
  Mahajan]{agarwal2021theory}
Agarwal, A., Kakade, S.~M., Lee, J.~D., and Mahajan, G.
\newblock On the theory of policy gradient methods: Optimality, approximation,
  and distribution shift.
\newblock \emph{The Journal of Machine Learning Research}, 22\penalty0
  (1):\penalty0 4431--4506, 2021.

\bibitem[Altman(1999)]{altman99cmdps}
Altman, E.
\newblock Constrained markov decision processes, 1999.

\bibitem[Auer et~al.(2002)Auer, Cesa-Bianchi, Freund, and Schapire]{exp3}
Auer, P., Cesa-Bianchi, N., Freund, Y., and Schapire, R.~E.
\newblock The nonstochastic multiarmed bandit problem.
\newblock \emph{SIAM journal on computing}, 32\penalty0 (1):\penalty0 48--77,
  2002.

\bibitem[Balduzzi et~al.(2018)Balduzzi, Racaniere, Martens, Foerster, Tuyls,
  and Graepel]{balduzzi2018mechanics}
Balduzzi, D., Racaniere, S., Martens, J., Foerster, J., Tuyls, K., and Graepel,
  T.
\newblock The mechanics of n-player differentiable games.
\newblock In \emph{International Conference on Machine Learning}, pp.\
  354--363. PMLR, 2018.

\bibitem[Barth-Maron et~al.(2018)Barth-Maron, Hoffman, Budden, Dabney, Horgan,
  TB, Muldal, Heess, and Lillicrap]{barthmaron2018d4pg}
Barth-Maron, G., Hoffman, M.~W., Budden, D., Dabney, W., Horgan, D., TB, D.,
  Muldal, A., Heess, N., and Lillicrap, T.
\newblock Distributed distributional deterministic policy gradients, 2018.
\newblock URL \url{https://arxiv.org/abs/1804.08617}.

\bibitem[Bas-Serrano et~al.(2021)Bas-Serrano, Curi, Krause, and
  Neu]{neu2021logistic}
Bas-Serrano, J., Curi, S., Krause, A., and Neu, G.
\newblock Logistic q-learning.
\newblock In Banerjee, A. and Fukumizu, K. (eds.), \emph{Proceedings of The
  24th International Conference on Artificial Intelligence and Statistics},
  volume 130 of \emph{Proceedings of Machine Learning Research}, pp.\
  3610--3618. PMLR, 13--15 Apr 2021.
\newblock URL \url{https://proceedings.mlr.press/v130/bas-serrano21a.html}.

\bibitem[Bauschke \& Combettes(2011)Bauschke and
  Combettes]{bauschke2011monotone}
Bauschke, H.~H. and Combettes, P.~L.
\newblock \emph{Convex Analysis and Monotone Operator Theory in Hilbert
  Spaces}.
\newblock Springer Publishing Company, Incorporated, 1st edition, 2011.
\newblock ISBN 1441994661.

\bibitem[Bellemare et~al.(2017)Bellemare, Dabney, and
  Munos]{bellemare2017distributional}
Bellemare, M.~G., Dabney, W., and Munos, R.
\newblock A distributional perspective on reinforcement learning, 2017.
\newblock URL \url{https://arxiv.org/abs/1707.06887}.

\bibitem[Bellemare et~al.(2020)Bellemare, Candido, Castro, Gong, Machado,
  Moitra, Ponda, and Wang]{bellemare2020balloons}
Bellemare, M.~G., Candido, S., Castro, P.~S., Gong, J., Machado, M.~C., Moitra,
  S., Ponda, S.~S., and Wang, Z.
\newblock Autonomous navigation of stratospheric balloons using reinforcement
  learning.
\newblock \emph{Nature}, 588\penalty0 (7836):\penalty0 77--82, 2020.

\bibitem[Bhatnagar \& Lakshmanan(2012)Bhatnagar and
  Lakshmanan]{bhatnagar2012online}
Bhatnagar, S. and Lakshmanan, K.
\newblock An online actor--critic algorithm with function approximation for
  constrained markov decision processes.
\newblock \emph{Journal of Optimization Theory and Applications}, 153\penalty0
  (3):\penalty0 688--708, 2012.

\bibitem[Bohez et~al.(2019)Bohez, Abdolmaleki, Neunert, Buchli, Heess, and
  Hadsell]{bohez2019success}
Bohez, S., Abdolmaleki, A., Neunert, M., Buchli, J., Heess, N., and Hadsell, R.
\newblock Success at any cost: value constrained model-free continuous control,
  2019.
\newblock URL \url{https://openreview.net/forum?id=rJlJ-2CqtX}.

\bibitem[Borkar(2005)]{borkar2005actor}
Borkar, V.~S.
\newblock An actor-critic algorithm for constrained markov decision processes.
\newblock \emph{Systems \& control letters}, 54\penalty0 (3):\penalty0
  207--213, 2005.

\bibitem[Brunke et~al.(2022)Brunke, Greeff, Hall, Yuan, Zhou, Panerati, and
  Schoellig]{brunke2022rwrl}
Brunke, L., Greeff, M., Hall, A.~W., Yuan, Z., Zhou, S., Panerati, J., and
  Schoellig, A.~P.
\newblock Safe learning in robotics: From learning-based control to safe
  reinforcement learning.
\newblock \emph{Annual Review of Control, Robotics, and Autonomous Systems},
  5\penalty0 (1):\penalty0 411--444, 2023/01/11 2022.

\bibitem[Bura et~al.(2022)Bura, Hasanzadezonuzy, Kalathil, Shakkottai, and
  Chamberland]{bura2022dope}
Bura, A., Hasanzadezonuzy, A., Kalathil, D., Shakkottai, S., and Chamberland,
  J.-F.
\newblock Dope: Doubly optimistic and pessimistic exploration for safe
  reinforcement learning.
\newblock In \emph{Advances in Neural Information Processing Systems}, 2022.

\bibitem[Cai \& Zheng(2022)Cai and Zheng]{cai2022accelerated}
Cai, Y. and Zheng, W.
\newblock Accelerated single-call methods for constrained min-max optimization.
\newblock \emph{arXiv preprint arXiv:2210.03096}, 2022.

\bibitem[Calian et~al.(2020)Calian, Mankowitz, Zahavy, Xu, Oh, Levine, and
  Mann]{calian2020metal}
Calian, D.~A., Mankowitz, D.~J., Zahavy, T., Xu, Z., Oh, J., Levine, N., and
  Mann, T.
\newblock Balancing constraints and rewards with meta-gradient d4pg, 2020.
\newblock URL \url{https://arxiv.org/abs/2010.06324}.

\bibitem[Chambolle \& Pock(2011)Chambolle and Pock]{chambolle2011first}
Chambolle, A. and Pock, T.
\newblock A first-order primal-dual algorithm for convex problems with
  applications to imaging.
\newblock \emph{Journal of Mathematical Imaging and Vision}, 40\penalty0
  (1):\penalty0 120--145, 2011.

\bibitem[Chiang et~al.(2012)Chiang, Yang, Lee, Mahdavi, Lu, Jin, and
  Zhu]{chiang2012online}
Chiang, C.-K., Yang, T., Lee, C.-J., Mahdavi, M., Lu, C.-J., Jin, R., and Zhu,
  S.
\newblock Online optimization with gradual variations.
\newblock In \emph{COLT ’12: Proceedings of the 25th Annual Conference on
  Learning Theory}, 2012.

\bibitem[Chow et~al.(2018)Chow, Nachum, Duenez-Guzman, and
  Ghavamzadeh]{chow2018lyapunov}
Chow, Y., Nachum, O., Duenez-Guzman, E., and Ghavamzadeh, M.
\newblock A lyapunov-based approach to safe reinforcement learning.
\newblock In Bengio, S., Wallach, H., Larochelle, H., Grauman, K.,
  Cesa-Bianchi, N., and Garnett, R. (eds.), \emph{Advances in Neural
  Information Processing Systems}, volume~31. Curran Associates, Inc., 2018.
\newblock URL
  \url{https://proceedings.neurips.cc/paper/2018/file/4fe5149039b52765bde64beb9f674940-Paper.pdf}.

\bibitem[Cui \& Shanbhag(2016)Cui and Shanbhag]{cui2016analysis}
Cui, S. and Shanbhag, U.~V.
\newblock On the analysis of reflected gradient and splitting methods for
  monotone stochastic variational inequality problems.
\newblock In \emph{CDC '16: Proceedings of the 57th IEEE Annual Conference on
  Decision and Control}, 2016.

\bibitem[Dalal et~al.(2018)Dalal, Dvijotham, Vecerik, Hester, Paduraru, and
  Tassa]{dalal2018safe}
Dalal, G., Dvijotham, K., Vecerik, M., Hester, T., Paduraru, C., and Tassa, Y.
\newblock Safe exploration in continuous action spaces.
\newblock \emph{arXiv preprint arXiv:1801.08757}, 2018.

\bibitem[Daskalakis \& Panageas(2018{\natexlab{a}})Daskalakis and
  Panageas]{daskalakis2018ogda}
Daskalakis, C. and Panageas, I.
\newblock The limit points of (optimistic) gradient descent in min-max
  optimization, 2018{\natexlab{a}}.
\newblock URL \url{https://arxiv.org/abs/1807.03907}.

\bibitem[Daskalakis \& Panageas(2018{\natexlab{b}})Daskalakis and
  Panageas]{daskalakis2018omwu}
Daskalakis, C. and Panageas, I.
\newblock Last-iterate convergence: Zero-sum games and constrained min-max
  optimization, 2018{\natexlab{b}}.
\newblock URL \url{https://arxiv.org/abs/1807.04252}.

\bibitem[Daskalakis et~al.(2018)Daskalakis, Ilyas, Syrgkanis, and
  Zeng]{daskalakis2018training_gans}
Daskalakis, C., Ilyas, A., Syrgkanis, V., and Zeng, H.
\newblock Training {GAN}s with optimism.
\newblock In \emph{International Conference on Learning Representations}, 2018.
\newblock URL \url{https://openreview.net/forum?id=SJJySbbAZ}.

\bibitem[Dayan \& Sejnowski(1996)Dayan and Sejnowski]{dayan1996exploration}
Dayan, P. and Sejnowski, T.~J.
\newblock Exploration bonuses and dual control.
\newblock \emph{Machine Learning}, 25\penalty0 (1):\penalty0 5--22, 1996.

\bibitem[Degrave et~al.(2022)Degrave, Felici, Buchli, Neunert, Tracey,
  Carpanese, Ewalds, Hafner, Abdolmaleki, de~las Casas, Donner, Fritz,
  Galperti, Huber, Keeling, Tsimpoukelli, Kay, Merle, Moret, Noury, Pesamosca,
  Pfau, Sauter, Sommariva, Coda, Duval, Fasoli, Kohli, Kavukcuoglu, Hassabis,
  and Riedmiller]{degrave2022fusion}
Degrave, J., Felici, F., Buchli, J., Neunert, M., Tracey, B., Carpanese, F.,
  Ewalds, T., Hafner, R., Abdolmaleki, A., de~las Casas, D., Donner, C., Fritz,
  L., Galperti, C., Huber, A., Keeling, J., Tsimpoukelli, M., Kay, J., Merle,
  A., Moret, J.-M., Noury, S., Pesamosca, F., Pfau, D., Sauter, O., Sommariva,
  C., Coda, S., Duval, B., Fasoli, A., Kohli, P., Kavukcuoglu, K., Hassabis,
  D., and Riedmiller, M.
\newblock Magnetic control of tokamak plasmas through deep reinforcement
  learning.
\newblock \emph{Nature}, 602\penalty0 (7897):\penalty0 414--419, 2022.

\bibitem[Dulac-Arnold et~al.(2019)Dulac-Arnold, Mankowitz, and
  Hester]{dulacarnold2019rwrl}
Dulac-Arnold, G., Mankowitz, D., and Hester, T.
\newblock Challenges of real-world reinforcement learning, 2019.
\newblock URL \url{https://arxiv.org/abs/1904.12901}.

\bibitem[Dulac-Arnold et~al.(2020)Dulac-Arnold, Levine, Mankowitz, Li,
  Paduraru, Gowal, and Hester]{dulac2020empirical}
Dulac-Arnold, G., Levine, N., Mankowitz, D.~J., Li, J., Paduraru, C., Gowal,
  S., and Hester, T.
\newblock An empirical investigation of the challenges of real-world
  reinforcement learning, 2020.
\newblock URL \url{https://arxiv.org/abs/2003.11881}.

\bibitem[Efroni et~al.(2020)Efroni, Mannor, and Pirotta]{efroni2020cmdps}
Efroni, Y., Mannor, S., and Pirotta, M.
\newblock Exploration-exploitation in constrained mdps, 2020.
\newblock URL \url{https://arxiv.org/abs/2003.02189}.

\bibitem[Espeholt et~al.(2018)Espeholt, Soyer, Munos, Simonyan, Mnih, Ward,
  Doron, Firoiu, Harley, Dunning, Legg, and Kavukcuoglu]{espeholt2018impala}
Espeholt, L., Soyer, H., Munos, R., Simonyan, K., Mnih, V., Ward, T., Doron,
  Y., Firoiu, V., Harley, T., Dunning, I., Legg, S., and Kavukcuoglu, K.
\newblock Impala: Scalable distributed deep-rl with importance weighted
  actor-learner architectures, 2018.
\newblock URL \url{http://arxiv.org/abs/1802.01561}.

\bibitem[Everett(1963)]{everett1963lagrangian}
Everett, H.
\newblock Generalized lagrange multiplier method for solving problems of
  optimum allocation of resources.
\newblock \emph{Oper. Res.}, 11\penalty0 (3):\penalty0 399–417, jun 1963.
\newblock URL \url{https://doi.org/10.1287/opre.11.3.399}.

\bibitem[Flennerhag et~al.(2021)Flennerhag, Schroecker, Zahavy, van Hasselt,
  Silver, and Singh]{flennerhag2021bootstrapped}
Flennerhag, S., Schroecker, Y., Zahavy, T., van Hasselt, H., Silver, D., and
  Singh, S.
\newblock Bootstrapped meta-learning.
\newblock \emph{arXiv preprint arXiv:2109.04504}, 2021.

\bibitem[Flennerhag et~al.(2023)Flennerhag, Zahavy, O'Donoghue, van Hasselt,
  Gy{\"o}rgy, and Singh]{flennerhag2023optimistic}
Flennerhag, S., Zahavy, T., O'Donoghue, B., van Hasselt, H., Gy{\"o}rgy, A.,
  and Singh, S.
\newblock Optimistic meta-gradients.
\newblock \emph{arXiv preprint arXiv:2301.03236}, 2023.

\bibitem[Freund \& Schapire(1997)Freund and Schapire]{freund1997online}
Freund, Y. and Schapire, R.~E.
\newblock A decision-theoretic generalization of on-line learning and an
  application to boosting.
\newblock \emph{Journal of Computer and System Sciences}, 55\penalty0
  (1):\penalty0 119--139, 1997.
\newblock URL
  \url{https://www.sciencedirect.com/science/article/pii/S002200009791504X}.

\bibitem[Geist et~al.(2019)Geist, Scherrer, and Pietquin]{geist2019mdpi}
Geist, M., Scherrer, B., and Pietquin, O.
\newblock A theory of regularized {M}arkov decision processes.
\newblock In Chaudhuri, K. and Salakhutdinov, R. (eds.), \emph{Proceedings of
  the 36th International Conference on Machine Learning}, volume~97 of
  \emph{Proceedings of Machine Learning Research}, pp.\  2160--2169. PMLR,
  09--15 Jun 2019.
\newblock URL \url{https://proceedings.mlr.press/v97/geist19a.html}.

\bibitem[Gidel et~al.(2019)Gidel, Berard, Vignoud, Vincent, and
  Lacoste-Julien]{gidel2019variational}
Gidel, G., Berard, H., Vignoud, G., Vincent, P., and Lacoste-Julien, S.
\newblock A variational inequality perspective on generative adversarial
  networks.
\newblock In \emph{ICLR '19: Proceedings of the 2019 International Conference
  on Learning Representations}, 2019.

\bibitem[Gopalan \& Thoppe(2022)Gopalan and Thoppe]{gopalan2022approximate}
Gopalan, A. and Thoppe, G.
\newblock Approximate q-learning and sarsa (0) under the epsilon-greedy policy:
  a differential inclusion analysis.
\newblock \emph{arXiv preprint arXiv:2205.13617}, 2022.

\bibitem[Grigoriadis \& Khachiyan(1995)Grigoriadis and
  Khachiyan]{grigoriadis1995mwu}
Grigoriadis, M.~D. and Khachiyan, L.~G.
\newblock A sublinear-time randomized approximation algorithm for matrix games.
\newblock \emph{Operations Research Letters}, 18\penalty0 (2):\penalty0 53--58,
  1995.
\newblock URL
  \url{https://www.sciencedirect.com/science/article/pii/0167637795000320}.

\bibitem[Hsieh et~al.(2019)Hsieh, Iutzeler, Malick, and
  Mertikopoulos]{hsieh2019single_call}
Hsieh, Y.-G., Iutzeler, F., Malick, J., and Mertikopoulos, P.
\newblock On the convergence of single-call stochastic extra-gradient methods,
  2019.
\newblock URL \url{https://arxiv.org/abs/1908.08465}.

\bibitem[Huang et~al.(2022)Huang, Abdolmaleki, Vezzani, Brakel, Mankowitz,
  Neunert, Bohez, Tassa, Heess, Riedmiller, and Hadsell]{huang2022lp3}
Huang, S., Abdolmaleki, A., Vezzani, G., Brakel, P., Mankowitz, D.~J., Neunert,
  M., Bohez, S., Tassa, Y., Heess, N., Riedmiller, M., and Hadsell, R.
\newblock A constrained multi-objective reinforcement learning framework.
\newblock In Faust, A., Hsu, D., and Neumann, G. (eds.), \emph{Proceedings of
  the 5th Conference on Robot Learning}, volume 164 of \emph{Proceedings of
  Machine Learning Research}, pp.\  883--893. PMLR, 08--11 Nov 2022.
\newblock URL \url{https://proceedings.mlr.press/v164/huang22a.html}.

\bibitem[Korpelevich(1976)]{korpelevich1976eg}
Korpelevich, G.~M.
\newblock The extragradient method for finding saddle points and other
  problems.
\newblock In \emph{Ekonomika i Matematicheskie Metody}, volume~12, pp.\
  747--756, 1976.

\bibitem[Kumar et~al.(2020)Kumar, Kumar, Levine, and Finn]{kumar2020one}
Kumar, S., Kumar, A., Levine, S., and Finn, C.
\newblock One solution is not all you need: Few-shot extrapolation via
  structured maxent rl.
\newblock \emph{Advances in Neural Information Processing Systems},
  33:\penalty0 8198--8210, 2020.

\bibitem[Lillicrap et~al.(2015)Lillicrap, Hunt, Pritzel, Heess, Erez, Tassa,
  Silver, and Wierstra]{lillicrap2015ddpg}
Lillicrap, T.~P., Hunt, J.~J., Pritzel, A., Heess, N., Erez, T., Tassa, Y.,
  Silver, D., and Wierstra, D.
\newblock Continuous control with deep reinforcement learning, 2015.
\newblock URL \url{https://arxiv.org/abs/1509.02971}.

\bibitem[Liu et~al.(2021)Liu, Zhou, Kalathil, Kumar, and Tian]{liu2021learning}
Liu, T., Zhou, R., Kalathil, D., Kumar, P., and Tian, C.
\newblock Learning policies with zero or bounded constraint violation for
  constrained mdps.
\newblock In \emph{Thirty-fifth Conference on Neural Information Processing
  Systems}, 2021.

\bibitem[Luo et~al.(2022)Luo, Paduraru, Voicu, Chervonyi, Munns, Li, Qian,
  Dutta, Davis, Wu, et~al.]{luo2022controlling}
Luo, J., Paduraru, C., Voicu, O., Chervonyi, Y., Munns, S., Li, J., Qian, C.,
  Dutta, P., Davis, J.~Q., Wu, N., et~al.
\newblock Controlling commercial cooling systems using reinforcement learning.
\newblock \emph{arXiv preprint arXiv:2211.07357}, 2022.

\bibitem[Malitsky \& Tam(2018)Malitsky and Tam]{malitsky2020forb}
Malitsky, Y. and Tam, M.~K.
\newblock A forward-backward splitting method for monotone inclusions without
  cocoercivity, 2018.
\newblock URL \url{https://arxiv.org/abs/1808.04162}.

\bibitem[Mandhane et~al.(2022)Mandhane, Zhernov, Rauh, Gu, Wang, Xue, Shang,
  Pang, Claus, Chiang, Chen, Han, Chen, Mankowitz, Broshear, Schrittwieser,
  Hubert, Vinyals, and Mann]{mandhane2022learn2encode}
Mandhane, A., Zhernov, A., Rauh, M., Gu, C., Wang, M., Xue, F., Shang, W.,
  Pang, D., Claus, R., Chiang, C.-H., Chen, C., Han, J., Chen, A., Mankowitz,
  D.~J., Broshear, J., Schrittwieser, J., Hubert, T., Vinyals, O., and Mann, T.
\newblock Muzero with self-competition for rate control in vp9 video
  compression, 2022.
\newblock URL \url{https://arxiv.org/abs/2202.06626}.

\bibitem[Mertikopoulos et~al.(2019)Mertikopoulos, Lecouat, Zenati, Foo,
  Chandrasekhar, and Piliouras]{mertikopoulos2018optimistic}
Mertikopoulos, P., Lecouat, B., Zenati, H., Foo, C.-S., Chandrasekhar, V., and
  Piliouras, G.
\newblock Optimistic mirror descent in saddle-point problems: Going the
  extra(-gradient) mile.
\newblock In \emph{International Conference on Learning Representations}, 2019.
\newblock URL \url{https://openreview.net/forum?id=Bkg8jjC9KQ}.

\bibitem[Moskovitz et~al.(2021)Moskovitz, Arbel, Huszar, and
  Gretton]{moskovitz2021efficient}
Moskovitz, T., Arbel, M., Huszar, F., and Gretton, A.
\newblock Efficient wasserstein natural gradients for reinforcement learning.
\newblock In \emph{International Conference on Learning Representations}, 2021.
\newblock URL \url{https://openreview.net/forum?id=OHgnfSrn2jv}.

\bibitem[Moskovitz et~al.(2022)Moskovitz, Arbel, Parker-Holder, and
  Pacchiano]{moskovitz2022towards}
Moskovitz, T., Arbel, M., Parker-Holder, J., and Pacchiano, A.
\newblock Towards an understanding of default policies in multitask policy
  optimization.
\newblock In \emph{International Conference on Artificial Intelligence and
  Statistics}, pp.\  10661--10686. PMLR, 2022.

\bibitem[Nemirovski(2004)]{nemirovski2004prox}
Nemirovski, A.~S.
\newblock Prox-method with rate of convergence o(1/t) for variational
  inequalities with lipschitz continuous monotone operators and smooth
  convex-concave saddle point problems.
\newblock \emph{SIAM Journal on Optimization}, 15\penalty0 (1):\penalty0
  229--251, 2004.

\bibitem[O'Donoghue(2021)]{o2021klearning}
O'Donoghue, B.
\newblock Variational {B}ayesian reinforcement learning with regret bounds.
\newblock \emph{Advances in Neural Information Processing Systems},
  34:\penalty0 28208--28221, 2021.

\bibitem[O'Donoghue \& Lattimore(2021)O'Donoghue and Lattimore]{o2021vbos}
O'Donoghue, B. and Lattimore, T.
\newblock Variational {B}ayesian optimistic sampling.
\newblock \emph{Advances in Neural Information Processing Systems},
  34:\penalty0 12507--12519, 2021.

\bibitem[O'Donoghue et~al.(2020)O'Donoghue, Lattimore, and
  Osband]{o2020stochastic}
O'Donoghue, B., Lattimore, T., and Osband, I.
\newblock Stochastic matrix games with bandit feedback.
\newblock \emph{arXiv preprint arXiv:2006.05145}, 2020.

\bibitem[Osband et~al.(2019)Osband, Doron, Hessel, Aslanides, Sezener, Saraiva,
  McKinney, Lattimore, Szepesvari, Singh, Van~Roy, Sutton, Silver, and
  Van~Hasselt]{osband2019bsuite}
Osband, I., Doron, Y., Hessel, M., Aslanides, J., Sezener, E., Saraiva, A.,
  McKinney, K., Lattimore, T., Szepesvari, C., Singh, S., Van~Roy, B., Sutton,
  R., Silver, D., and Van~Hasselt, H.
\newblock Behaviour suite for reinforcement learning, 2019.
\newblock URL \url{https://arxiv.org/abs/1908.03568}.

\bibitem[Pacchiano et~al.(2020)Pacchiano, Parker-Holder, Tang, Choromanski,
  Choromanska, and Jordan]{pacchiano2020bgrl}
Pacchiano, A., Parker-Holder, J., Tang, Y., Choromanski, K., Choromanska, A.,
  and Jordan, M.
\newblock Learning to score behaviors for guided policy optimization.
\newblock In III, H.~D. and Singh, A. (eds.), \emph{Proceedings of the 37th
  International Conference on Machine Learning}, volume 119 of
  \emph{Proceedings of Machine Learning Research}, pp.\  7445--7454. PMLR,
  13--18 Jul 2020.
\newblock URL \url{https://proceedings.mlr.press/v119/pacchiano20a.html}.

\bibitem[Paternain et~al.(2019)Paternain, Chamon, Calvo-Fullana, and
  Ribeiro]{paternain2019duality}
Paternain, S., Chamon, L., Calvo-Fullana, M., and Ribeiro, A.
\newblock Constrained reinforcement learning has zero duality gap.
\newblock In Wallach, H., Larochelle, H., Beygelzimer, A., d\textquotesingle
  Alch\'{e}-Buc, F., Fox, E., and Garnett, R. (eds.), \emph{Advances in Neural
  Information Processing Systems}, volume~32. Curran Associates, Inc., 2019.
\newblock URL
  \url{https://proceedings.neurips.cc/paper/2019/file/c1aeb6517a1c7f33514f7ff69047e74e-Paper.pdf}.

\bibitem[Perolat et~al.(2021)Perolat, Munos, Lespiau, Omidshafiei, Rowland,
  Ortega, Burch, Anthony, Balduzzi, De~Vylder, et~al.]{perolat2021poincare}
Perolat, J., Munos, R., Lespiau, J.-B., Omidshafiei, S., Rowland, M., Ortega,
  P., Burch, N., Anthony, T., Balduzzi, D., De~Vylder, B., et~al.
\newblock From poincar{\'e} recurrence to convergence in imperfect information
  games: Finding equilibrium via regularization.
\newblock In \emph{International Conference on Machine Learning}, pp.\
  8525--8535. PMLR, 2021.

\bibitem[Popov(1980)]{popov1980modification}
Popov, L.~D.
\newblock A modification of the arrow-hurwicz method for search of saddle
  points.
\newblock \emph{Mathematical Notes of the Academy of Sciences of the USSR},
  28\penalty0 (5):\penalty0 845--848, 1980.

\bibitem[Puterman(2014)]{puterman2014markov}
Puterman, M.~L.
\newblock \emph{Markov decision processes: discrete stochastic dynamic
  programming}.
\newblock John Wiley \& Sons, 2014.

\bibitem[Reich \& Sabach(2011)Reich and Sabach]{reich2011bregman}
Reich, S. and Sabach, S.
\newblock Existence and approximation of fixed points of bregman firmly
  nonexpansive mappings in reflexive banach spaces.
\newblock In \emph{Springer Optimization and Its Applications}, chapter Chapter
  15, pp.\  301--316. Springer, 2011.
\newblock URL
  \url{https://EconPapers.repec.org/RePEc:spr:spochp:978-1-4419-9569-8_15}.

\bibitem[Ryu \& Yin(2022)Ryu and Yin]{ryu20222monotone}
Ryu, E.~K. and Yin, W.
\newblock \emph{Large-Scale Convex Optimization: Algorithms \&amp; Analyses via
  Monotone Operators}.
\newblock Cambridge University Press, 2022.
\newblock \doi{10.1017/9781009160865}.

\bibitem[Schrittwieser et~al.(2020)Schrittwieser, Antonoglou, Hubert, Simonyan,
  Sifre, Schmitt, Guez, Lockhart, Hassabis, Graepel, Lillicrap, and
  Silver]{schrittwieser2020muzero}
Schrittwieser, J., Antonoglou, I., Hubert, T., Simonyan, K., Sifre, L.,
  Schmitt, S., Guez, A., Lockhart, E., Hassabis, D., Graepel, T., Lillicrap,
  T., and Silver, D.
\newblock Mastering atari, go, chess and shogi by planning with a learned
  model.
\newblock \emph{Nature}, 588\penalty0 (7839):\penalty0 604--609, 2020.

\bibitem[Schulman et~al.(2015)Schulman, Levine, Abbeel, Jordan, and
  Moritz]{schulman2015trpo}
Schulman, J., Levine, S., Abbeel, P., Jordan, M.~I., and Moritz, P.
\newblock Trust region policy optimization.
\newblock In Bach, F.~R. and Blei, D.~M. (eds.), \emph{ICML}, volume~37 of
  \emph{JMLR Workshop and Conference Proceedings}, pp.\  1889--1897. JMLR.org,
  2015.
\newblock URL \url{http://proceedings.mlr.press/v37/schulman15.html}.

\bibitem[Schulman et~al.(2017)Schulman, Wolski, Dhariwal, Radford, and
  Klimov]{schulman2017ppo}
Schulman, J., Wolski, F., Dhariwal, P., Radford, A., and Klimov, O.
\newblock Proximal policy optimization algorithms.
\newblock \emph{arXiv preprint arXiv:1707.06347}, 2017.

\bibitem[Shani et~al.(2020)Shani, Efroni, and Mannor]{shani2020mdpotheory}
Shani, L., Efroni, Y., and Mannor, S.
\newblock Adaptive trust region policy optimization: Global convergence and
  faster rates for regularized mdps.
\newblock \emph{Proceedings of the AAAI Conference on Artificial Intelligence},
  34\penalty0 (04):\penalty0 5668--5675, 2023/01/16 2020.

\bibitem[Shani et~al.(2022)Shani, Zahavy, and Mannor]{shani2022oal}
Shani, L., Zahavy, T., and Mannor, S.
\newblock Online apprenticeship learning.
\newblock \emph{Proceedings of the AAAI Conference on Artificial Intelligence},
  36\penalty0 (8):\penalty0 8240--8248, 2023/01/13 2022.

\bibitem[Sim{\~a}o et~al.(2021)Sim{\~a}o, Jansen, and
  Spaan]{simao2021alwayssafe}
Sim{\~a}o, T.~D., Jansen, N., and Spaan, M.~T.
\newblock Alwayssafe: Reinforcement learning without safety constraint
  violations during training.
\newblock In \emph{Proceedings of the 20th International Conference on
  Autonomous Agents and MultiAgent Systems}. International Foundation for
  Autonomous Agents and Multiagent Systems, 2021.

\bibitem[Sokota et~al.(2022)Sokota, D'Orazio, Kolter, Loizou, Lanctot,
  Mitliagkas, Brown, and Kroer]{sokota2022unified}
Sokota, S., D'Orazio, R., Kolter, J.~Z., Loizou, N., Lanctot, M., Mitliagkas,
  I., Brown, N., and Kroer, C.
\newblock A unified approach to reinforcement learning, quantal response
  equilibria, and two-player zero-sum games.
\newblock \emph{arXiv preprint arXiv:2206.05825}, 2022.

\bibitem[Stooke et~al.(2020)Stooke, Achiam, and Abbeel]{stooke2020pid}
Stooke, A., Achiam, J., and Abbeel, P.
\newblock Responsive safety in reinforcement learning by pid lagrangian
  methods, 2020.
\newblock URL \url{https://arxiv.org/abs/2007.03964}.

\bibitem[Strehl \& Littman(2008)Strehl and Littman]{strehl2008analysis}
Strehl, A.~L. and Littman, M.~L.
\newblock An analysis of model-based interval estimation for markov decision
  processes.
\newblock \emph{Journal of Computer and System Sciences}, 74\penalty0
  (8):\penalty0 1309--1331, 2008.

\bibitem[Sutton(2004)]{sutton2004reward}
Sutton, R.
\newblock The reward hypothesis, 2004.
\newblock URL
  \url{http://incompleteideas.net/rlai.cs.ualberta.ca/RLAI/rewardhypothesis.html}.

\bibitem[Sutton \& Barto(2018)Sutton and Barto]{sutton2018reinforcement}
Sutton, R.~S. and Barto, A.~G.
\newblock \emph{Reinforcement learning: An introduction}.
\newblock MIT press, 2018.

\bibitem[Szepesv\'ari(2020)]{szepesvari2020cmdps}
Szepesv\'ari, C.
\newblock Constrained mdps and the reward hypothesis, Mar 2020.
\newblock URL
  \url{http://readingsml.blogspot.com/2020/03/constrained-mdps-and-reward-hypothesis.html}.

\bibitem[Tassa et~al.(2018)Tassa, Doron, Muldal, Erez, Li, Casas, Budden,
  Abdolmaleki, Merel, Lefrancq, Lillicrap, and Riedmiller]{tassa2018dmc}
Tassa, Y., Doron, Y., Muldal, A., Erez, T., Li, Y., Casas, D. d.~L., Budden,
  D., Abdolmaleki, A., Merel, J., Lefrancq, A., Lillicrap, T., and Riedmiller,
  M.
\newblock Deepmind control suite, 2018.
\newblock URL \url{https://arxiv.org/abs/1801.00690}.

\bibitem[Tessler et~al.(2019)Tessler, Mankowitz, and Mannor]{tessler2018reward}
Tessler, C., Mankowitz, D.~J., and Mannor, S.
\newblock Reward constrained policy optimization.
\newblock In \emph{International Conference on Learning Representations}, 2019.
\newblock URL \url{https://openreview.net/forum?id=SkfrvsA9FX}.

\bibitem[Thomas et~al.(2017)Thomas, da~Silva, Barto, and
  Brunskill]{thomas2017wellbehaved}
Thomas, P.~S., da~Silva, B.~C., Barto, A.~G., and Brunskill, E.
\newblock On ensuring that intelligent machines are well-behaved, 2017.
\newblock URL \url{https://arxiv.org/abs/1708.05448}.

\bibitem[Tomar et~al.(2020)Tomar, Shani, Efroni, and
  Ghavamzadeh]{tomar2021mdpo}
Tomar, M., Shani, L., Efroni, Y., and Ghavamzadeh, M.
\newblock Mirror descent policy optimization, 2020.
\newblock URL \url{https://arxiv.org/abs/2005.09814}.

\bibitem[Xu et~al.(2018)Xu, van Hasselt, and Silver]{xu2018metagrads}
Xu, Z., van Hasselt, H.~P., and Silver, D.
\newblock Meta-gradient reinforcement learning.
\newblock In Bengio, S., Wallach, H., Larochelle, H., Grauman, K.,
  Cesa-Bianchi, N., and Garnett, R. (eds.), \emph{Advances in Neural
  Information Processing Systems}, volume~31. Curran Associates, Inc., 2018.
\newblock URL
  \url{https://proceedings.neurips.cc/paper/2018/file/2715518c875999308842e3455eda2fe3-Paper.pdf}.

\bibitem[Young \& Sutton(2020)Young and Sutton]{young2020oscillations}
Young, K. and Sutton, R.~S.
\newblock Understanding the pathologies of approximate policy evaluation when
  combined with greedification in reinforcement learning, 2020.
\newblock URL \url{https://arxiv.org/abs/2010.15268}.

\bibitem[Zahavy et~al.(2020{\natexlab{a}})Zahavy, Cohen, Kaplan, and
  Mansour]{zahavy2020apprenticeship}
Zahavy, T., Cohen, A., Kaplan, H., and Mansour, Y.
\newblock Apprenticeship learning via frank-wolfe.
\newblock In \emph{Proceedings of the AAAI Conference on Artificial
  Intelligence}, volume~34, pp.\  6720--6728, 2020{\natexlab{a}}.

\bibitem[Zahavy et~al.(2020{\natexlab{b}})Zahavy, Xu, Veeriah, Hessel, Oh, van
  Hasselt, Silver, and Singh]{zahavy2020self}
Zahavy, T., Xu, Z., Veeriah, V., Hessel, M., Oh, J., van Hasselt, H.~P.,
  Silver, D., and Singh, S.
\newblock A self-tuning actor-critic algorithm.
\newblock \emph{Advances in neural information processing systems},
  33:\penalty0 20913--20924, 2020{\natexlab{b}}.

\bibitem[Zahavy et~al.(2021{\natexlab{a}})Zahavy, O'Donoghue, Barreto, Mnih,
  Flennerhag, and Singh]{zahavy2021discovering}
Zahavy, T., O'Donoghue, B., Barreto, A., Mnih, V., Flennerhag, S., and Singh,
  S.
\newblock Discovering diverse nearly optimal policies withsuccessor features.
\newblock \emph{arXiv preprint arXiv:2106.00669}, 2021{\natexlab{a}}.

\bibitem[Zahavy et~al.(2021{\natexlab{b}})Zahavy, O'Donoghue, Desjardins, and
  Singh]{zahavy2021reward}
Zahavy, T., O'Donoghue, B., Desjardins, G., and Singh, S.
\newblock Reward is enough for convex {MDP}s.
\newblock In Beygelzimer, A., Dauphin, Y., Liang, P., and Vaughan, J.~W.
  (eds.), \emph{Advances in Neural Information Processing Systems},
  2021{\natexlab{b}}.
\newblock URL \url{https://openreview.net/forum?id=ELndVeVA-TR}.

\bibitem[Zahavy et~al.(2022)Zahavy, Schroecker, Behbahani, Baumli, Flennerhag,
  Hou, and Singh]{zahavy2022domino}
Zahavy, T., Schroecker, Y., Behbahani, F., Baumli, K., Flennerhag, S., Hou, S.,
  and Singh, S.
\newblock Discovering policies with domino: Diversity optimization maintaining
  near optimality, 2022.
\newblock URL \url{https://arxiv.org/abs/2205.13521}.

\end{thebibliography}
\bibliographystyle{icml2022}

\newpage
\appendix
\onecolumn

\section{Additional Related Work} \label{sec:additional_related}

Beyond that which is covered in the main text, there is a rich history of work on CMDPs. \citet{borkar2005actor} was the first to analyze an actor-critic approach to CMDPs, while \citet{bhatnagar2012online} was the first to expand this approach to function approximation. \citet{tessler2018reward,achiam2017constrained,efroni2020cmdps,bohez2019success,chow2018lyapunov,paternain2019duality} all focus on integrating constraints into sequential decision problems \citep{altman99cmdps}, commonly done, as noted, via Lagrangian relaxation \citep{tessler2018reward}.
\citet{calian2020metal} argue for a \textit{soft-constraint} approach to CRL, wherein the solution is not absolutely required to satisfy a particular constraint, but rather is penalized in proportion to its violation \citep{thomas2017wellbehaved,dulac2020empirical}. This philosophy lies in contrast to \textit{hard-constraint} approaches, for which a solution is marked as invalid if there is any constraint violation \citep{dalal2018safe,bura2022dope}. In zero-violation approaches, methods are initialized in the feasible region, and only updated in ways that are guaranteed not to leave the feasible region \citep{liu2021learning,simao2021alwayssafe}. In applications, soft constraints may be preferable when violations do not result in catastrophic system failure, but rather simply undesirable behavior (e.g., inefficiency). \citet{calian2020metal} attempt to stabilize the learning process in CMDPs by using meta-gradients \citep{xu2018metagrads, zahavy2020self} to adapt the learning rate of the Lagrange multiplier online. More recently, Bootstrapped meta gradients \citep{flennerhag2021bootstrapped} have been analyzed and shown to provide a form of optimism \citep{flennerhag2023optimistic}; thus, it would be interesting to see if they can help to achieve LIC in CMDPs.  \citet{stooke2020pid} apply a principled approach to damping the dynamics of the Lagrange multiplier via PID control with the goal of reducing constraint overshoots over the course of training. However, neither of these two approaches guarantees LIC---this may have a connection to \cref{thm:singly_optimistic}, which shows that only one optimistic player (e.g., the Lagrange player) is in general not sufficient to guarantee LIC. \citet{efroni2020cmdps} address the role of exploration in CMDPs, proving bounds for both the linear programming formulation and the primal-dual formulation of the problem, though they only consider standard gradients and do not focus on LIC vs. AIC. 

Rather than formulate CRL as a CMDP, \citet{huang2022lp3} and \citet{abdolmaleki2020dmpo} consider a multi-objective approach. That is, rather than integrate the constraints and the task reward into a single, non-stationary reward, they optimize the task reward and each constraint reward independently, and then search over a pareto front which balances among them.

There are a number of recent applications of CRL which have achieved impressive practical successes. One example is the use of MuZero \citep{schrittwieser2020muzero} for video compression by \citet{mandhane2022learn2encode}. Specifically, the agent was trained on a constrained MDP to maximize video quality with a constraint on the allowed bit rate. Others involve using CRL for quality-diversity optimization where the agent is trying to find a set of diverse skill while all of the skills are required to satisfy a near-optimality constraint on the reward \citep{zahavy2021discovering,zahavy2022domino,kumar2020one}. 

Beyond extra-gradient methods (see \cref{sec:eg_methods} below), there are other approaches which aim to achieve LIC in min-max games. The symplectic gradient adjustment \citep{balduzzi2018mechanics} uses a signed additive term to the gradient to push the dynamics away from unstable equilibria and towards stable ones. Hamiltonian gradient descent updates the iterates in the direction of the (negative) gradient of the squared norm of the signed partial derivatives, and has been shown to achieve LIC in a variety of min-max games \citep{abernethy2021last}. However, when applied to CMDPs, minimizing the squared gradient with respect to the Lagrange multiplier(s) is equivalent to an apprenticeship learning problem \citep{abbeel2004apprenticeship, zahavy2020apprenticeship, shani2022oal}, which is itself a convex MDP representing a challenging optimization problem \citep{zahavy2021reward}. \citet{perolat2021poincare} instead augment the objective with an adaptive regularizer, solving the resulting convex/concave (but biased) problem exactly before iteratively refitting with progressively lesser regularization. 

One aspect we have not touched on is the question of exploration. When the agent has uncertainty about the rewards or the constraints, how can it `explore' sufficiently well in order to solve the CMDP? This would require information seeking behaviour to discover the rewards and constraints in the environment. A common heuristic for exploration is `optimism in the face of uncertainty' \cite{dayan1996exploration, strehl2008analysis, o2021klearning}, however there is little work applying optimism to CMDPs. Some preliminary work in constrained bandits \cite{o2021vbos} or more generally in two-player zero sum matrix games \cite{o2020stochastic} is encouraging, but the presented algorithms are not online and therefore questions like LIC are not applicable. Similarly, classic adversarial algorithms like EXP3 \cite{exp3} typically only guarantee AIC. Clearly more work is required in this area.

As noted in the main text, many policy optimization methods in deep RL currently use a form of trust region to stabilize optimization. KL-based trust regions are discussed in detail by \citet{agarwal2021theory,geist2019mdpi,moskovitz2022towards,shani2020mdpotheory}. While not explored further here, it would also be interesting to study policy optimization methods which use trust regions generated by other divergence measures or distances, such as the Wasserstein distance \citep{moskovitz2021efficient,pacchiano2020bgrl}.

\section{Extreme Constraints and Oscillations} \label{sec:extreme_constraints}
As noted in the main text, so-called ``extreme'' constraints don't often induce oscillations in practice. Here, we present a simple, relatively informal argument to provide an intuition for why this is the case. Consider a simple CMDP with one constraint and where values lie in the range $[0, B]$, where $B < \infty$. (Note this bounding can be easily obtained for any reward function which is upper- and lower-bounded by adding the lower bound to rewards to make them non-negative.) The Lagrangian in this case is
\begin{align}
    \mathcal L(d_\pi, \mu) = -\langle r_0, d_\pi \rangle + \mu(\langle r_1, d_\pi\rangle - \theta). 
\end{align}
An \textit{extreme} constraint threshold $\theta$, then, is one which is close to 0 or B, as those are the bounds on the value. Say that $\theta = 0$. Then the Lagrangian reduces to
\begin{align}
    \mathcal L(d_\pi, \mu) = -\langle r_0, d_\pi \rangle + \mu\langle r_1, d_\pi\rangle,
\end{align}
with the gradients being
\begin{align*}
    \grad_{d_\pi}\mathcal L &= \mu r_1 - r_0 \\
    \grad_\mu \mathcal L &= \langle r_1, d_\pi\rangle = v_1. 
\end{align*}
Because $v_1 \geq 0$, the Lagrange multiplier $\mu$ will always be increasing, which, because of $\grad_{d_\pi}\mathcal L$, means that the occupancy measure will increasingly be updated to align with the constraint reward $r_1$ and ignore the task reward $r_0$. Therefore, the constraint reward will dominate the optimization of $d_\pi$, and there is no ``back and forth'' between optimizing the constraint reward and the task reward. Similarly, if $\theta = B$, $v_1 - B \leq 0$, so $\mu$ will be non-increasing, eventually dropping to 0, so that $d_\pi$ will only optimize the task reward $r_1$. 

This also provides motivation for why problems with ``intermedaite'' constraints which induce oscillations are the most valuable/interesting CMDPs---because extreme constraints cause one reward function to dominate, the problem essentially reduces to an MDP, and can reasonably be solved with standard RL methods.

\section{Extra-Gradient Methods} \label{sec:eg_methods}
The minimax problems studied in the paper all fall within the broader class of \textit{variational inequality} (VI) problems, which can be written as
\begin{align*}
    \mathrm{find} \ x^\star \in \mathcal X \subseteq \reals^d \quad \mathrm{s.t.} \quad \langle F(x^\star), x - x^\star \rangle \geq 0 \ \forall x\in\mathcal X
\end{align*}
for some single-valued operator $F: \reals^d \to \reals^d$. For example, if $F = \grad \mathcal L$ for some differentiable loss $\mathcal L$, the solution to the VI problem is a critical point of $\mathcal L$. The \textit{extra-gradient} algorithm \citep[EG;][]{korpelevich1976eg} is as follows:
\begin{align*}
    x^{k+1/2} &= \Pi_\mathcal X(x^k - \eta^k F(x^k)) \\
    x^{k+1} &= \Pi_\mathcal X(x^k - \eta^k F(x^{k+1/2})),
\end{align*}
where $\Pi_\mathcal X(z) \triangleq \min_{x\in\mathcal X} \|x - z\|$ is the projection operator onto $\mathcal X$. The Bregman variant of EG is called the Mirror-Prox method \citep[MP;][]{nemirovski2004prox}. 
While EG achieves the optimal $\mathcal O(1/k)$ convergence rate when $V$ is monotone and Lipschitz-continuous, it is difficult to scale, as each update requires two gradient computations and two projections into $\mathcal X$. In addition to the optimistic gradient method, there are a number of so-called ``single-call'' EG methods (one of which is the optimistic gradient):
\begin{itemize}
    \item \textit{Past EG} \citep[PEG;][]{chiang2012online,gidel2019variational,popov1980modification}:
        \begin{align}
        \begin{split} \label{eq:peg}
            x^{k+1/2} &= \Pi_\mathcal X(x^k - \eta^k F(x^{k-1/2})) \\
            x^{k+1} &= \Pi_\mathcal X(x^k - \eta^k F(x^{k+1/2}))
        \end{split}
        \end{align}
    \item \textit{Reflected Gradient} \citep[RG;][]{chambolle2011first,cui2016analysis}:
        \begin{align*}
            x^{k+1/2} &= x^k - (x^{k-1} - x^k) \\
            x^{k+1} &= \Pi_\mathcal X(x^k - \eta^k F(x^{k+1/2})).
        \end{align*}
\end{itemize}
Both PEG and RG are very closely related to OGD, and are identical in the unconstrained case. 
As an additional note, the term ``optimistic gradient'' is frequently applied broadly within the literature, and is often used to refer to the EG method or any of its single-call variants. 
For a more comprehensive discussion and analysis of these methods, we refer the reader to \citet{hsieh2019single_call}. 

\subsection{PEG Policy Iteration}
As described in \cref{sec:reload}, policy evaluation in the tabular case is equivalent to gradient computation, and a trust region step equates to a projection step in OMD. While OMD only requires one of these each per update, PEG (\cref{eq:peg}) requires two projections and one gradient call per update, which makes it less scalable than OMD. Nonetheless, we can derive a Bregman PEG-based optimization method for optimizing CMDPs which carries similar guarantees to ReLOAD. 

Let $\ell(x) \triangleq \mathcal L(x, y^k)$. Then we can write the PEG update as
\begin{align*}
    x^{k+1/2} &= \argmin_{x\in\mathcal X} \langle \grad\ell(x^{k-1/2}), x\rangle + \frac{1}{\eta^k} D_\Omega(x; x^k) \\
    x^{k+1} &= \argmin_{x\in\mathcal X} \langle \grad\ell(x^{k+1/2}), x \rangle + \frac{1}{\eta^k} D_\Omega(x; x^k). 
\end{align*}
Setting $\Omega_\pi$ as the negative entropy and $\Omega_\mu$ as the squared Euclidean norm, we can derive the policy iteration-style algorithm in \cref{alg:peg_mdpi}. We verified its LIC on the paradoxical CMDP from \cref{fig:toy}, with results plotted in \cref{fig:peg_mdpi}.

\begin{algorithm}[!t]
	\caption{PEG-MDPI}\label{alg:peg_mdpi}
		\begin{algorithmic}[1] 
		    \STATE Require: CMDP $\mathcal M_C$, step sizes $\{\eta_\pi, \eta_\mu\} > 0$
		    \STATE Initialize $ \pi^1$, $\mu^1$, $\pi^0$, $\mu^0$
            \FOR{$k=1,\dots,K$}
                \STATE Half-step: 
                \begin{align*}
                    \pi^{k+1/2} &=  \frac{ \pi^k\exp\left(  q_{\pi^{k-1/2}}^{\mu^{k-1/2}}/\eta_\pi \right)}{\left\langle \pi^k\exp\left(  q_{\pi^{k-1/2}}^{\mu^{k-1/2}}/\eta_\pi \right),  1 \right\rangle} \\
                     \mu^{k+1/2} &=  \max\{ \mu^k - \eta_\mu ( v_{1:N}^{k-1/2} -  \theta),  0\}
                \end{align*}
            \STATE $\{ q_n^{k+1/2}\}_{n=0}^N \gets \texttt{PolicyEval}(\mathcal M_C,  \pi^{k+1/2})$
            \STATE $ q_{\mu^{k+1/2}}^{k+1/2} \gets - q_0^{k+1/2} + \sum_{n=1}^N \mu^{k+1/2}_n  q_n^{k+1/2}$ \quad (mixed $q$-values)
            \STATE $ v_{1:N}^{k+1/2} \gets [\langle  q_1^{k+1/2},  \pi^{k+1/2}\rangle, \dots,  \langle  q_N^{k+1/2},  \pi^{k+1/2}\rangle]^\top$
            \STATE Full-step:
                \begin{align*}
                    \pi^{k+1} &=  \frac{ \pi^k\exp\left(  q_{\pi^{k+1/2}}^{\mu^{k+1/2}}/\eta_\pi \right)}{\left\langle \pi^k\exp\left(  q_{\pi^{k+1/2}}^{\mu^{k+1/2}}/\eta_\pi \right),  1 \right\rangle} \\
                     \mu^{k+1} &=  \max\{ \mu^k - \eta_\mu ( v_{1:N}^{k+1/2} -  \theta),  0\}
                \end{align*}
            \ENDFOR
            \STATE \textbf{return} $\pi^K, \mu^K$
	\end{algorithmic}
\end{algorithm}


\section{Theoretical Results} \label{sec:theory}

\subsection{Impossibility Results}

\mdbad*

\begin{proof}
    This proof is originally due to \citet{daskalakis2018training_gans}, but we repeat it here for completeness. 
    
    Consider the problem
    \begin{align*}
        \min_{x\in\reals} \max_{y\in\reals} xy.
    \end{align*}
     The gradient descent-ascent updates at step $k$ are:
    \begin{align*}
        x_{k+1} &= x_k - \eta y_k \\
        y_{k+1} &= y_k + \eta x_k.
    \end{align*}
    For simplicity of notation we consider a fixed learning rate $\eta$, but the same result is obtained for variable learning rates. This problem has a unique SP at $(x^\star, y^\star) = (0, 0)$. Consider the squared distance from the origin at step $k$, $\Delta_k \triangleq x_k^2 + y_k^2$. We then have
    \begin{align*}
        \Delta_{k+1} &= x_{k+1}^2 + y_{k+1}^2 \\
            &= (x_k - \eta y_k)^2 + (y_k + \eta x_k)^2 \\
            &= x_k^2 - 2\eta x_k y_k + \eta^2 y_k^2 + y_k^2 + 2\eta x_k y_k + \eta^2 x_k^2 \\
            &= \Delta_k + \eta^2 \Delta_k \\
            &= (1 + \eta^2) \Delta_k.
    \end{align*}
    Therefore, for any $\eta > 0$, the distance from the SP grows with every step of gradient ascent-descent.
\end{proof}

\singlyoptimistic* 

\begin{proof}
Consider $\min_{x_\reals} \max_{y\in\reals} \ xy$. The OGD-GA (singly-optimistic) dynamics are given by
\begin{align*}
    x_{k+1} &= x_k - 2\eta y_k + \eta y_{k-1} \\
    y_{k+1} &= y_k + \eta x_k.
\end{align*}

\noindent \textit{Stability analysis}: To facilitate analysis, we can rewrite this (discrete-time) dynamical system by introducing an extra variable $z$ which tracks $y$:
\begin{align*}
    x_{k+1} &= x_k - 2\eta y_k + \eta z_k \\
    y_{k+1} &= y_k + \eta x_k \\
    z_{k+1} &= y_k.
\end{align*}
The Jacobian is given by
\begin{align*}
    J = \begin{pmatrix}
        1 & -2\eta & \eta \\
        \eta & 1 & 0 \\
        0 & 1 & 0
    \end{pmatrix}.
\end{align*}
Its eigenvalues $\lambda$ are obtained by solving 
\begin{align}
    \mathrm{det}(J - \lambda I) = -\lambda^3 + 2\lambda^2 - (2\eta^2 +1)\lambda + \eta^2 = 0.
\end{align}
Two of the three eigenvalues have modulus greater than one for all $\eta > 0$, so $\rho(J) > 1$ and the system is unstable. 
\end{proof}

\subsection{Analysis of Mixed-Bregman OMD via Monotone Operators}

\paragraph{CMDPs as Monotone Inclusion Problems}
The theory of monotone operators provides a general, powerful framework for analyzing the behavior of convex optimization problems \citep{ryu20222monotone}. While monotone operator theory's generality often facilitates relatively simple proofs of convergence, it can also make it more challenging to show \textit{rates} of convergence. Nonetheless, as the main goal of our theoretical analysis is to simply demonstrate the fundamental soundness of our proposed approach, monotone operator theory is a useful choice. Specifically, we cast convex-concave SP problems as \textit{monotone inclusion problems}. 
Monotone inclusion problems take the form 
\begin{align*}
   \mathrm{find }\  x \in \mathcal H \quad \mathrm{s.t.} \quad 0 \in F(x)
\end{align*}
where $F: \mathcal H \to \mathcal H$ is a \textit{monotone} operator and $\mathcal H$ is a Hilbert space. A Hilbert space is a vector space upon which an inner product $\langle \cdot, \cdot\rangle$ is defined and for which the distance induced by this inner product is a complete metric space (one which contains every Cauchy sequence within it). One example of a vector space is the $d$-dimensional reals $\reals^d$, for which the inner product is the standard dot product and the induced metric is the Euclidean distance. A monotone operator is an operator for which the following inequality holds:
\begin{align*}
    \langle F(x) - F(y), x - y \rangle \geq 0
\end{align*}
for all $x, y \in \mathcal H$. Often inclusion problems can be made simpler to solve if the operator $F$ can be \textit{split} into two (or more) operators $F = A + B$, as $A$ and $B$ may be more tractable to use and evaluate separately. We therefore consider inclusion problems of the form
\begin{align} 
    \mathrm{find }\  x \in \mathcal H \quad \mathrm{s.t.} \quad 0 \in (A + B)(x),
\end{align}
where $A: \mathcal H \rightrightarrows \mathcal H$ and $B: \mathcal H \to \mathcal H$ are monotone operators. The notation $\rightrightarrows$ indicates $A$ is (in the general case) a set-valued operator, one which maps a point in $\mathcal H$ to a (possibly empty) subset of $\mathcal H$. A wide variety of problems can be expressed in this form, and in particular we consider the Lagrangian of the CMDP problem expressed as follows:
\begin{align}
    \min_{d_\pi \in \reals^{|\mathcal S||\mathcal A|}} \max_{\mu\in \reals^{N}} \ \mathbb I_{\mathcal K}(d_\pi) + \mathcal L(d_\pi, \mu) + \mathbb I_{\mathbb R_{\geq 0}^N}(\mu), 
\end{align}
where $\mathbb I_\mathcal X(\cdot)$ is the indicator function which equals 0 inside the set $\mathcal X \subseteq \mathcal H$ and $+\infty$ outside of $\mathcal X$. We can therefore express the CMDP problem in the form of \cref{eq:monotone_inc} by noting that a SP must satisfy the first-order optimality condition:
\begin{align}
    \mathrm{find }\  \begin{bmatrix} d_\pi \\ \mu \end{bmatrix}  \quad \mathrm{s.t.} \quad 
    \begin{bmatrix} 0 \\ 0 \end{bmatrix} \in \begin{bmatrix}
        \partial  \mathbb I_{\mathcal K}(d_\pi) \\
        \partial \mathbb I_{\mathbb R_{\geq 0}^N}(\mu)
    \end{bmatrix}
    + \begin{bmatrix}
        \grad_{d_\pi} \mathcal L(d_\pi, \mu) \\
        -\grad_\mu \mathcal L(d_\pi, \mu)
    \end{bmatrix}
\end{align}
where we can note that $\partial \mathbb I_\mathcal X = N_\mathcal X$, where $N_\mathcal X$ is the normal cone operator for $\mathcal X$ \citep{ryu20222monotone}.

\paragraph{(Optimistic) Mirror Descent as Fixed Point Iteration}
The mirror descent update for the Bregman divergence $D_\Omega(\cdot; \cdot)$ generated by the strictly convex, continuously differentiable function $\Omega(\cdot)$ and applied to a differentiable loss function $\ell(\cdot)$ is given by
\begin{align}
    x^{k+1} &= \argmin_{x\in\mathcal X} \ \langle \grad \ell(x^k), x \rangle + \frac{1}{\eta^k} D_\Omega(x; x^k) \\
    &= \argmin_{x\in\reals^d}\  \langle \grad\ell(x^k), x \rangle + \frac{1}{\eta^k} D_\Omega(x; x^k) + \mathbb I_\mathcal X(x).
\end{align}
This is equivalent to solving the following inclusion problem:
\begin{align}
    0 &\in \grad\ell(x^k) + \frac{1}{\eta^k}(\grad\Omega(x) - \grad\Omega(x^k)) + N_\mathcal X(x) \\
    \iff \grad\Omega(x^k) - \eta^k \grad\ell(x^k) &\in (\grad\Omega + \eta^k N_\mathcal X)(x) \\
    \iff x &\in (\grad\Omega + \eta^k N_\mathcal X)^{-1}(\grad\Omega(x^k) - \eta^k \grad\ell(x^k)) = \mathrm{Prt}_{\eta^k N_\mathcal X}^\Omega(\grad\Omega(x^k) - \eta^k \grad\ell(x^k)),
\end{align}
where $\mathrm{Prt}_{\eta^k N_\mathcal X}^\Omega = (\grad\Omega + \eta^k N_\mathcal X)^{-1}$ is the \textit{proto-resolvent} of $\eta^k N_\mathcal X$ relative to $\Omega$ \citep{reich2011bregman}. Thus, mirror descent can be seen as performing fixed point iteration (FPI) as follows:
\begin{align}
    x^{k+1} = \mathrm{Prt}_{\eta^k N_\mathcal X}^\Omega(\grad\Omega - \eta^k \grad\ell)(x^k).
\end{align}
For optimistic mirror descent, we write the update as
\begin{align}
    x^{k+1} = \mathrm{Prt}_{\eta^k N_\mathcal X}^\Omega(\grad\Omega(x^k) - \eta^k\grad\ell(x^k) - \eta^{k-1}(\grad\ell(x^k) - \grad\ell(x^{k-1})) ).
\end{align}
More generally, for SP problems like \cref{eq:saddlepoint_monotone}, we can think of $x^k = [d_\pi^k, \mu^k]\tr \in \mathcal K \times \reals_{\geq 0} = \mathcal X$, where 
\begin{align}
\begin{split}
    A = \begin{bmatrix}
        N_{\mathcal K} \\ N_{\reals_{\geq 0}^N}
    \end{bmatrix} \qquad
    B = \begin{bmatrix}
        \grad_{d_\pi} \mathcal L(d_\pi, \mu) \\
        -\grad_\mu \mathcal L(d_\pi, \mu)
    \end{bmatrix} \qquad
    \grad\Omega = \begin{bmatrix}
        \grad \Omega_{\pi} \\ \grad \Omega_\mu
    \end{bmatrix}.
\end{split}
\end{align}
As a note, in the following analysis we'll frequently consider the Bregman divergence generated by $\Omega$:
\begin{align}
    D_\Omega(x_1; x_2) \triangleq \Omega(x_1) - \Omega(x_2) - \langle \grad\Omega(x_2), x_1 - x_2\rangle.
\end{align}
In the ``stacked''/SP case, we can consider this to be the \textit{stacked} divergence where $\Omega = [\Omega_\pi, \Omega_\mu]^\top$ and $\grad\Omega$ is as above.
We can then write the OMD update as
\begin{align} \label{eq_omd_monotone}
    x^{k+1} = \mathrm{Prt}_{\eta^k A}^\Omega(\grad\Omega(x^k) - \eta^kB(x^k) - \eta^{k-1}(B(x^k) - B(x^{k-1}))).
\end{align}
When $\Omega_\pi(\cdot) = \Omega_\mu(\cdot) = \frac{1}{2}\|\cdot\|^2$, this is equivalent to \textit{forward-reflected-backward splitting} \citep{malitsky2020forb}. 

\subsection{Convergence Analysis}
Before showing convergence, we require several prepatory lemmas. In the following, we assume that $A: \mathcal H \rightrightarrows H$ is maximal monotone and $B: \mathcal H \to \mathcal H$ is monotone and $L$-Lipschitz. 

\begin{restatable}{lemma}{bregman_bound} \label{thm:bregman_bound}
    Let $F: \mathcal H \rightrightarrows \mathcal H$ be maximal monotone, and let $d_1, u_1, u_0, v_2, v_1 \in \mathcal H$ be arbitrary. Define $d_2$ as
    \begin{align*}
        d_2 = \mathrm{Prt}_F^\Omega(\grad\Omega(d_1) - u_1 - (v_1 - u_0)).
    \end{align*}
    Then, $\forall x \in \mathcal H$ and $u \in -F(x)$, we have
    \begin{align*}
        D_\Omega(x;d_2) + \langle v_2 - u_1, x - d_2 \rangle \leq D_{\Omega}(x; d_1) &+ \langle v_1 - u_0, x - d_1 \rangle \\ &+ \langle v_1 - u_0, d_1 - d_2 \rangle - D_\Omega(d_2; d_1) - \langle v_2 - u, d_2 - x\rangle.
    \end{align*}
\end{restatable}
\begin{proof}
    From the definition of the proto-resolvent, we have
    \begin{align*}
        \grad\Omega(d_2) + F(d_2) \ni \grad\Omega(d_1) - u_1 - (v_1 - u_0).
    \end{align*}
    Then by the monotonicity of $F$,
    \begin{align*}
        0 &\leq \langle -u - \grad\Omega(d_1) + u_1 + (v_1 - u_0) + \grad\Omega(d_2), x - d_2 \rangle \\
        &= -\langle \grad\Omega(d_1) - \grad\Omega(d_2), x - d_2 \rangle + \langle u + v_1 - u_0, x - d_2\rangle + \langle u, d_2 - x\rangle \\
        &\stackrel{(1)}{=} -D_\Omega(x; d_2) - D_\Omega(d_2; d_1) + D_\Omega(x; d_1) + \langle u + v_1 - u_0, x - d_2\rangle + \langle u, d_2 - x\rangle,
    \end{align*}
    where $(1)$ is due to the Bregman 3-point identity. Then we can simply add and subtract both $\langle v_2, x - d_2\rangle$ and $\langle d_1, v_1 - u_0\rangle$ and rearrange to get the desired result.
\end{proof}

\begin{restatable}{lemma}{bregman_bound2} \label{thm:bregman_bound2}
    Let $x \in (A + B)^{-1}(0)$ and let $(x^k)$ be given by \cref{eq_omd_monotone}. Suppose that $(\eta^k) \subseteq [\varepsilon, \frac{1 - 2\varepsilon}{2L}]$ for some $\varepsilon > 0$. Then for all $k\in\mathbb N$ we have
    \begin{align*}
    \begin{split}
        D_\Omega(x; x^{k+1}) + \eta^k\langle B(x^{k+1}) - B(x^k),\ &x - x^{k+1} \rangle + D_\Omega(x^{k+1}; x^k) + \left(\frac{\varepsilon}{2} - \frac{1}{4}\right)\|x^{k+1} - x^k\|^2 \\
        & \leq D_\Omega(x; x^k) + \eta^{k-1}\langle B(x^k) - B(x^{k-1}), x - x^k\rangle + \frac{1}{4}\|x^k - x^{k-1}\|^2.
    \end{split}
    \end{align*}
\end{restatable}
\begin{proof}
    Apply \cref{thm:bregman_bound} with 
    \begin{align*}
    \begin{split}
        &F \triangleq \eta^k A \quad d_1 \triangleq x^k \quad u_0 \triangleq \eta^{k-1}B(x^{k-1}) \quad v^1 \triangleq \eta^{k-1}B(x^k) \\
        &u \triangleq \eta^k B(x) \quad d_2 \triangleq x^{k+1} \quad u_1 \triangleq \eta^{k} B(x^k) \quad v_2 \triangleq \eta^k B(x^{k+1})
    \end{split}
    \end{align*}
    to get 
    \begin{align*}
    \begin{split}
        D_\Omega(x; x^{k+1}) + &\eta^k\langle B(x^{k+1}) - B(x^k), x - x^{k+1}\rangle + D_\Omega(x^{k+1}; x^k) \\
        &\leq D_\Omega(x; x^k) + \eta^{k-1}\langle B(x^k) - B(x^{k-1}), x - x^k\rangle \\ & \qquad + \eta^{k-1}\underbrace{\langle B(x^k) - B(x^{k-1}), x^k - x^{k+1}\rangle}_{(1)} - \eta^k \underbrace{\langle B(x^{k+1}) - B(x), x^{k+1} - x\rangle}_{(2)}.
    \end{split}
    \end{align*}
    Since $B$ is monotone, $(2)$ is non-negative and can be dropped. Since $B$ is also $L$-Lipschitz, $(1)$ can be written as
    \begin{align*}
        \langle B(x^k) - B(x^{k-1}), x^k - x^{k+1}\rangle &\leq L\|x^k - x^{k-1}\|\|x^k - x^{k+1}\| \\
        &\leq \frac{L}{2}(\|x^k - x^{k-1}\|^2 + \|x^k - x^{k+1}\|^2).
    \end{align*}
    Applying these changes and rearranging gives
    \begin{align*}
        D_\Omega(x; x^{k+1}) + \eta^k\langle B(x^{k+1}) - B(x^k),\  &x - x^{k+1}\rangle + D_\Omega(x^{k+1}; x^k) - \frac{1}{2} \eta^{k-1}L\|x^{k+1} - x^k\|^2  \\
        &\leq D_\Omega(x; x^k) + \eta^{k-1}\langle B(x^k) - B(x^{k-1}), x - x^k\rangle + \frac{1}{2}\eta^{k-1}L\|x^k - x^{k-1}\|^2.
    \end{align*}
    Then since $\eta^{k-1}L \leq 1/2$ and
\begin{align*}
    -\frac{1}{2}\eta^{k-1}L \geq -\frac{1 - 2\varepsilon}{4} = -\frac{1}{4} + \frac{\varepsilon}{2},
\end{align*}
we get the desired bound.
\end{proof}

Now we are able to prove convergence.

\convergence*

\begin{proof}
    Let $x\in (A+B)^{-1}(0)$. We can telescope \cref{thm:bregman_bound2} to get
    \begin{align} \label{eq:telescope}
    \begin{split}
        D_\Omega(x; x^{k+1}) + \eta^k \langle B(x^{k+1}) - B(x^k),\  &x - x^{k+1}\rangle + D_\Omega(x^{k+1}; x^k) - \frac{1}{4}\|x^{k+1} - x^k\|^2  + \frac{\varepsilon}{2} \sum_{i=0}^k \|x^{i+1} - x^i\|^2 \\ 
        &\leq D_\Omega(x; x^0) + \eta^{-1}\langle B(x^0) - B(x^{-1}), x - x^0\rangle + \frac{1}{4} \|x^0 - x^{-1}\|^2. 
    \end{split}
    \end{align}
    Because $B$ is $L$-Lipschitz, we have
    \begin{align*}
        \eta^k \langle B(x^{k+1}) - B(x^k), x - x^{k+1}\rangle &\geq -\eta^k L\|x^{k+1} - x^k\|\|x - x^{k+1}\| \\
        &\geq -\frac{1}{2} \eta^k L(\|x^{k+1} - x^k\|^2 + \|x - x^{k+1}\|).
    \end{align*}
    Then, because $\frac{1}{2}\eta^kL \leq \frac{1 - 2\varepsilon}{4} < \frac{1}{4}$ and since our choice of $\Omega$ implies $D_\Omega(w; z) \geq \frac{1}{2}\|w - z\|^2$, 
    we can write
    \begin{align*}
        D_\Omega(x^{k+1}; x^k) - \frac{1}{4}\|x^{k+1} - x^k\|^2 
        \geq \frac{1}{2}\|x^{k+1} - x^k\|^2 - \frac{1}{4}\|x^{k+1} - x^k\|^2 
        = \frac{1}{4}\|x^{k+1} - x^k\|^2.
    \end{align*}
    Combining this with \cref{eq:telescope} gives
    \begin{align*}
        D_\Omega(x; x^{k+1}) + \frac{\varepsilon}{2}\sum_{i=0}^k \|x^{i+1} - x^i\|^2  \leq D_\Omega(x; x^0) + \eta^{-1}\langle B(x^0) - B(x^{-1}), x - x^0\rangle + \frac{1}{4} \|x^0 - x^{-1}\|^2.
    \end{align*}
    Therefore, we can see that $(x^k)$ is bounded above and moreover that $\|x^{k+1} - x^k\| \to 0$. Let $\bar{x}$ be a sequential weak cluster point of $(x^k)$. We know that 
    \begin{align*}
        x^{k+1} = (\grad\Omega + \eta^k A)^{-1}(\grad\Omega(x^k) - \eta^k B(x^k)) - \eta^{k-1}(B(x^k) - B(x^{k-1}))
    \end{align*}
    so that
    \begin{align*}
        \grad\Omega(x^{k+1}) - \grad\Omega(x^k) + \eta^k B(x^k) + \eta^{k-1}(B(x^k) - B(x^{k-1})) \in - \eta^k A(x^{k+1}).
    \end{align*}
    By subtracting $\eta^k B(x^{k+1})$ from both sides and rearranging, we get
    \begin{align} \label{eq:seq}
        \frac{1}{\eta^k} \left(\grad\Omega(x^{k}) - \grad\Omega(x^{k+1}) + \eta^k (B(x^{k+1}) - B(x^k)) + \eta^{k-1}(B(x^{k-1}) - B(x^{k})) \right) \in (A + B)(x^{k+1}).
    \end{align}
    Since $A+B$ is maximal monotone, its graph is demiclosed. Take the limit along a subsequence of $(x^k)$ which converges to $\bar x$ in \cref{eq:seq} and, noting that $\eta^k \geq \varepsilon\ \forall k\in\mathbb N$, we can see that the differences on the left side of \cref{eq:seq} vanish in the limit so that $0 \in (A+ B)(\bar x)$. 
    
    To show that $(x^k)$ is weakly convergent to $\bar x$, note that by combining \cref{thm:bregman_bound2} and the Lipschitzness of $B$, we can see that the limit
    \begin{align*}
        \lim_{k\to\infty} \ D_\Omega(\bar x; x^k) + \eta^{k-1}\langle B(x^k)- B(x^{k-1}), \bar x - x^k\rangle + \frac{1}{4} \|x^k - x^{k-1}\|^2 
    \end{align*}
    exists. Since $(x^k)$ and $(\eta^k)$ are bounded, $\|x^k - x^{k-1}\| \to 0$, and since $B$ is continuous, the limit is equal to $\lim_{k\to\infty} D_\Omega(\bar x; x^k)$. Since the cluster point $\bar x$ was chosen arbitrarily, $(x^k)$ is weakly convergent by \citet{malitsky2020forb} Lemma 2.2, and the proof is complete. 
\end{proof}

This result then immediately implies the following corollary when the learning rate $\eta$ is fixed.
\begin{restatable}{corollary}{convergence_fixed} \label{thm:convergence_fixed}
    Let $A: \mathcal H \rightrightarrows \mathcal H$ be maximal monotone and let $B: \mathcal H \to \mathcal H$ be monotone and $L$-Lipschitz and suppose that $(A+B)^{-1}(0) \neq \varnothing$. Suppose that $\eta \in (0, \frac{1}{2L})$. Given $x^0, x^{-1} \in \mathcal H$, define the sequence $(x^k)$ according to 
    \begin{align}
         x^{k+1} = \mathrm{Prt}_{\eta A}^\Omega(\grad\Omega(x^k) - 2\eta B(x^k) + \eta B(x^{k-1}))
    \end{align}
    Then $(x^k)$ converges weakly to a point contained in $(A + B)^{-1}(0)$. 
\end{restatable}

\rate*

\begin{proof}
Without loss of generality, let $A$ be $m$-strongly monotone. Note that we do not lose generality here because if $B$ were $m$-strongly monotone instead, we could write $A + B = (A + mI) + (B - mI)$, thereby making $A' = A + mI$ strongly monotone while preserving the monotonicity of $B' = B - mI$. Let $x \in (A + B)^{-1}(0)$. Then with a constant learning rate $\eta \in (0, 1/2L)$ as in \cref{thm:convergence_fixed} we can use the strong monotonicity of $A$ in place of standard monotonicity in \cref{thm:bregman_bound} and propagate the resulting inequality through \cref{thm:bregman_bound2} to get
\begin{align*}
    \eta m \|x - x^{k+1} \|^2 + D_\Omega(x; x^{k+1}) &+ \eta \langle B(x^{k+1} - B(x^k), x - x^{k+1} \rangle + \bd(\xkk; x^k) - \frac{1}{2} L\eta \|\xkk - x^k\|^2  \\
    &\leq \bd(x; x^k) + \eta\langle B(x^k) - B(x^{k-1}), x - x^k\rangle + \frac{1}{4}\|x^k - x^{k-1}\|^2. 
\end{align*}
Then because $\Omega$ is $\sigma$-strongly convex and has $\Lo$-Lipschitz smooth gradient, by \cite{bauschke2011monotone} Theorem 18.5, we have
\begin{align*}
    \eta m \|x - x^{k+1} \|^2 + \frac{\sigma}{2}\|x - \xkk\|^2 &+ \eta \langle B(x^{k+1} - B(x^k), x - x^{k+1} \rangle + \frac{\sigma}{2}\|\xkk - x^k\|^2 - \frac{1}{2} L\eta \|\xkk - x^k\|^2  \\
    &\leq \frac{\Lo}{2}\|x - x^k\|^2 + \eta\langle B(x^k) - B(x^{k-1}), x - x^k\rangle + \frac{1}{4}\|x^k - x^{k-1}\|^2
\end{align*}
Simplifying gives 
\begin{align*}
    (\sigma + 2\eta m) \|x - x^{k+1} \|^2 &+ 2\eta \langle B(x^{k+1} - B(x^k), x - x^{k+1} \rangle + (\sigma - L\eta)\|\xkk - x^k\|^2 \\
    &\leq \Lo\|x - x^k\|^2 + 2\eta\langle B(x^k) - B(x^{k-1}), x - x^k\rangle + \frac{1}{2}\|x^k - x^{k-1}\|^2.
\end{align*}
Define sequences $(a^k)$ and $(b^k)$ such that 
\begin{align*}
    a^k &\triangleq \frac{\Lo}{2}\|x - x^k\|^2 \\
    b^k &\triangleq \frac{\Lo}{2}\|x - x^k\|^2 + 2\eta\langle B(x^k) - B(x^{k-1}), x - x^k\rangle + \frac{1}{2}\|x^k - x^{k-1}\|^2.
\end{align*}
We'd like to rewrite the LHS of the inequality in terms of $a^{k+1}$ and $b^{k+1}$. Observe that 
\begin{align*}
    (\sigma + 2\eta m)\|x - x^{k+1}\|^2 &= \frac{2}{\Lo}(\sigma + 2\eta m)\frac{\Lo}{2} \|x - x^{k+1}\|^2 = \frac{2}{\Lo}(\sigma + 2\eta m) a^{k+1} \\
    &= \left(\frac{2\sigma}{\Lo} + \frac{4\eta m}{\Lo} \right)a^{k+1} = \frac{\Lo}{2}\|x - \xkk\|^2  + \left(\frac{2\sigma}{\Lo} + \frac{4\eta m}{\Lo} - 1\right)a^{k+1},
\end{align*}
and that 
\begin{align*}
    (\sigma - L\eta) \|\xkk - x^k\|^2 = \frac{1}{2}\|\xkk - x^k\|^2 + \left(\sigma - L\eta - \frac{1}{2}\right)\|\xkk - x^k\|^2.
\end{align*}
Then we have
\begin{align*}
    \left(\frac{2\sigma}{\Lo} + \frac{4\eta m}{\Lo} - 1\right)a^{k+1} + b^{k+1} + \left(\sigma - L\eta - \frac{1}{2}\right)\|\xkk - x^k\|^2 \leq a^k + b^k.
\end{align*}
Set $\varepsilon \triangleq \min\left\{\frac{4\sigma - 2\Lo + 8\eta m}{\Lo + 2}, \sigma - L\eta - \frac{1}{2} \right\}$. Then we can write
\begin{align*}
    \left(\frac{2\sigma}{\Lo} + \frac{4\eta m}{\Lo} - 1\right)a^{k+1} + b^{k+1} + \varepsilon\|\xkk - x^k\|^2 \leq a^k + b^k.
\end{align*}
Consider the LHS. Add and subtract $\varepsilon b^{k+1}/2$ to get 
\begin{align*}
     &\left(\frac{2\sigma}{\Lo} -1 + \frac{4\eta m}{\Lo} - \frac{\varepsilon}{2}\right)a^{k+1} + \left(1 + \frac{\varepsilon}{2}\right) b^{k+1} + \frac{3\varepsilon}{4}\|\xkk - x^k\|^2 - \varepsilon \eta \langle B(x^{k+1}) - B(x^k), x - \xkk \rangle \\
     &\stackrel{(i)}{\geq} \left(\frac{2\sigma}{\Lo} -1 + \frac{4\eta m}{\Lo} - \frac{\varepsilon}{2}\right)a^{k+1} + \left(1 + \frac{\varepsilon}{2}\right) b^{k+1} + \frac{3\varepsilon}{4}\|\xkk - x^k\|^2 - \varepsilon\eta L \|\xkk - x^k \|\|x - \xkk\| \\
     &\stackrel{(ii)}{\geq} \left(\frac{2\sigma}{\Lo} -1 + \frac{4\eta m}{\Lo} - \frac{\varepsilon}{2}\right)a^{k+1} + \left(1 + \frac{\varepsilon}{2}\right) b^{k+1} + \frac{3\varepsilon}{4}\|\xkk - x^k\|^2 - \frac{\varepsilon\eta L}{2}\left(\|\xkk - x^k \|^2 + \|x - \xkk\|^2 \right) \\
     &\geq \left(\frac{2}{\Lo} - 1 +  \frac{4\eta m}{\Lo} - \frac{\varepsilon}{2} - \frac{\varepsilon}{2\Lo} \right)a^{k+1} + \left(1 + \frac{\varepsilon}{2}\right) b^{k+1} + \frac{\varepsilon}{2} \|\xkk - x^k \|^2,
\end{align*}
where $(i)$ is because $B$ is $L$-Lipschitz and $(ii)$ is due to Young's inequality. Now, set $\alpha \triangleq \min\left\{\frac{2}{\Lo} - 1 +  \frac{4\eta m}{\Lo} - \frac{\varepsilon}{2} - \frac{\varepsilon}{2\Lo}, 1 + \frac{\varepsilon}{2} \right\} > 1$ due to $\varepsilon \leq \frac{4\sigma - 2\Lo + 8\eta m}{\Lo + 2}$. We then get
\begin{align*}
    &\alpha\left(a^{k+1} + b^{k+1}\right) + \frac{\varepsilon}{2}\|\xkk - x^k\|^2 \leq a^k + b^k \\
    &\Rightarrow \alpha(a^{k+1} + b^{k+1}) \leq a^k + b^k,
\end{align*}
which, iterating, yields
\begin{align*}
    a^{k+1} + b^{k+1} \leq \frac{1}{\alpha}( a^k + b^k) \leq \frac{1}{\alpha}\left(\frac{1}{\alpha}( a^{k-1} + b^{k-1}) \right) \leq \cdots \leq \frac{1}{\alpha^k}(a^0 + b^0).
\end{align*}
Therefore, $x^k \to x$ with rate $\mathcal O(1/\alpha^k)$, and since $x$ was arbitrarily chosen, it is unique. 
\end{proof}

\convergenceconvexreload*

\begin{proof}
To connect this problem to \cref{thm:convergence}, we can can take advantage of the fact that monotonicity is preserved by concatenation \citep{ryu20222monotone} and set $x^k = [d_\pi^k, \mu^k]\tr \in \mathcal K \times \reals_{\geq 0} = \mathcal X$ such that
\begin{align}
\begin{split}
    A = \begin{bmatrix}
        N_{\mathcal K} \\ N_{\reals_{\geq 0}^N}
    \end{bmatrix} \quad
    B = \begin{bmatrix}
        \grad_{d_\pi} \mathcal L \\
        -\grad_\mu \mathcal L
    \end{bmatrix} \quad
    \grad\Omega = \begin{bmatrix}
        \grad \Omega_{\pi} \\ \grad \Omega_\mu
    \end{bmatrix}.
\end{split}
\end{align}
We can further confirm that the normal cone operator is maximal monotone \citep{ryu20222monotone} and that $\grad\mathcal L$ is monotone and 1-Lipschitz, as the problem is bilinear. We then get the desired result by applying \cref{thm:convergence_fixed}.
\end{proof}

\section{Overview of Extragradient Approaches} \label{sect:extragradients}

\section{Further Algorithm Details} \label{sect:algorithm_details}
Below, we provide a more detailed overview of the algorithms used in the paper, both ReLOAD variants and baseline approaches.

\paragraph{Mirror Descent Policy Iteration (MDPI)} MDPI \citep{geist2019mdpi} is a modified policy iteration algorithm which, rather than apply a hard/greedy improvement step to the policy, uses a soft-improvement step:
\begin{align*}
    \pi^{k+1} = \frac{\pi^k \exp( q^k/\eta_\pi)}{\langle\pi^k \exp( q^k/\eta_\pi),  1 \rangle }
\end{align*}
To adapt MDPI for Lagrangian optimization, we update the policy using the mixed $q$-value $ q^k \to  q_\mu^k$ as described in the main text and update the Lagrange multiplier(s) with projected gradient ascent. This is $\mu$-MDPI, and is similar to the procedure used in \citet{shani2022oal}. To add ReLOAD, we simply replace the mixed $q$-values with optimistic mixed $q$-values. This procedure is summarized in \cref{alg:tabular_reload}.

\paragraph{Mirror Descent Policy Optimization (MDPO)} MDPO \citep{tomar2021mdpo} is a DRL approach designed to apply a scalable approximation of MD to the policy update, and is closely related to both TRPO \citep{schulman2015trpo} and PPO \citep{schulman2017ppo}. There are both on- and off-policy versions of MDPO. Here, we focus on the on-policy version, for which at step the algorithm approximately performs the following update to the policy parameters $\theta$: 
\begin{align}
\begin{split} \label{eq:mdpo_update}
    \theta^{k+1} &\gets \argmax_{\theta\in\Theta} \ \Psi(\theta, \theta^k) \quad \mathrm{where} \\
    \Psi(\theta, \theta^k) = \E_{s\sim d_{\pi_{\theta^k}}}&\left[\E_{a\sim \pi_{\theta^k}} A^{\pi^{\theta^k}}(s, a) - \frac{1}{\eta^k} \kl[\pi_{\theta}(\cdot\mid s) || \pi_{\theta^k}(\cdot\mid s)] \right],
\end{split}
\end{align}
where $A^\pi(s,a)\triangleq q^\pi(s,a) - v^\pi(s)$ is the advantage function for $\pi$. Rather than solve this optimization exactly, for each update to the policy, MDPO performs an inner loop consisting of $m$ SGD steps on $\Psi$ using $(s,a)$ pairs obtained from rollouts using the current policy $\pi_{\theta^k}$. To adapt MDPO for Lagrangian optimization ($\mu$-MDPO), we computed the advantages for the task and constraint rewards $A^\pi_n(s,a) = q^\pi_n(s,a) - v^\pi_n(s)$ and then mixed them using the Lagrange multiplier(s) $\mu^n$ to create mixed advantages:
\begin{align*}
    A_{\mu^k}^{\theta^k}(s,a) = A_0^{\theta^k}(s,a) - \sum_{n=1}^N \mu_n A_n^{\theta^k}(s,a).
\end{align*}
Note that this requires a value function estimate for each constraint reward in addition to the task reward. We then substituted the mixed advantages for the standard advantages in \cref{eq:mdpo_update} to update the policy. To compute the Lagrange multiplier update, we computed the average constraint violation (gradient) on a minibatch of $L$ states obtained from rollouts of the current policy 
\begin{align*}
    \Delta^k_n = \frac{1}{L}\sum_{\ell=1}^L (v^\pi(s_\ell) - \theta_n)
\end{align*}
and updated the Lagrange multiplier(s) as 
\begin{align*}
    \mu^{k+1}_n = \max\{\mu^k_n + \eta_\mu \Delta^k_n, 0\} \quad \mathrm{for} \ n=1, 2, \dots, N.
\end{align*}
To add ReLOAD, we repeated the above steps but used the optimistic mixed advantages
\begin{align*}
    \tilde A_{\mu^k}^{\pi^k}(s,a) = 2A_{\mu^k}^{\pi^k}(s,a) - A_{\mu^{k-1}}^{\pi^{k-1}}(s,a)
\end{align*}
and optimistic values to compute the constraint violations:
\begin{align*}
    \tilde v^{\theta^k}_n(s_\ell) = 2v^{\theta^k}_n(s_\ell)  - v^{\theta^{k-1}}_n(s_\ell)  \quad \mathrm{for} \ n=1, 2, \dots, N.
\end{align*}
To compute the optimistic advantages and values, we maintain a copy of the previous parameters $\theta^{k-1}$ to evaluate on the new trajectories collected by $\pi_{\theta^k}$. In practice, we also use entropy regularization on the policy, making our implementation of $\mu$-MDPO equivalent to \textit{magnetic mirror-descent} \citep[MMD;][]{sokota2022unified}.

\paragraph{IMPALA} The Importance-Weighted Actor-Learner Architecture \citep[IMPALA][]{espeholt2018impala} is a distributed actor-critic agent which uses a set of actors to generate trajectories of experience which are then used by one or more learners to update the policy and value function parameters off-policy. To correct for the disjunction between the actor parameters and the learner parameters, updates are computed using an importance-weighting correction technique called V-trace. Unlike the other base agents with which we paired ReLOAD, IMPALA does not by default use a trust region approach for policy optimization. Therefore, our first modification was to add one via an inner loop in the style of MDPO. We then also added modifications for Lagrangian optimization and optimism as in MDPO. However, since we applied the IMPALA-based agents to large-scale tasks, optimization was inherently less stable, and so we used a sigmoid function to bound the Lagrange multiplier as done in \citet{zahavy2021reward,stooke2020pid}, leading to mixed advantages of the form:
\begin{align*}
     A_{\mu^k}^{\theta^k}(s,a) = \left(N - \sum_{n=1}^N \sigma(\mu_n^k) \right) A_0^{\theta^k}(s,a) + \sum_{n=1}^N \sigma(\mu_n^k) A_n^{\theta^k}(s,a).
\end{align*}
The overall procedure for the learner is summarized in \cref{alg:reload_impala_learner_step}, \cref{alg:reload_impala_fixedoutputs}, and \cref{alg:reload_impala_loss}. Code for V-trace can be found at this link: \url{https://github.com/deepmind/rlax/blob/master/rlax/_src/vtrace.py}.

\begin{algorithm}[!t]
	\caption{ReLOAD-IMPALA \texttt{LearnerStep}}\label{alg:reload_impala_learner_step}
		\begin{algorithmic}[1] 
		    \STATE Require: network parameters $\theta^k$, previous parameters $\theta^{k-1}$, Lagrange multiplier $\mu^k$, previous Lagrange multiplier $\mu^{k-1}$, total learner updates $K$, trajectory rollouts $\{\tau_n\}_{n=1}^N$, number of inner loop updates $m$, constraint value $\varphi_e$
		    \STATE Compute trust region step-size: $\eta^k \gets 1 - k/K$
		    \STATE $\pi_{\theta^{k}}(\cdot|s_t), \textcolor{purple}{\pi_{\theta^{k-1}}(\cdot|s_t)}, \textcolor{purple}{v^0_{\theta^{k-1}}(s_t)}, \textcolor{purple}{v^1_{\theta^{k-1}}(s_t)}, \log \nu(a_t|s_t), \textcolor{purple}{\rho_{\theta^{k-1}}(s_t, a_t)} \gets$\texttt{GetFixedOutputs}$(\theta^k,\textcolor{purple}{\theta^{k-1}}, \{\tau_n\}_{n=1}^N)$
		    \STATE Set inner loop parameters $\theta^k_{(0)} \gets \theta^k$
            \FOR{$i=0,\dots,m-1$}
                \STATE $\theta_{(i+1)}^k \gets \theta_{(i)}^k - \epsilon \grad_\theta \mathcal L(\theta_{(i)}^k, \eta^k, \mu^k, \pi_{\theta^{k}}(\cdot|s_t), \textcolor{purple}{\pi_{\theta^{k-1}}(\cdot|s_t)}, \textcolor{purple}{v^0_{\theta^{k-1}}(s_t)}, \textcolor{purple}{v^1_{\theta^{k-1}}(s_t)}, \log \nu(a_t|s_t), \textcolor{purple}{\rho_{\theta^{k-1}}(s_t, a_t)}, \mu^{k-1}, \{\tau_n\}_{n=1}^N)$
            \ENDFOR
            \STATE Reset network parameters $\theta^{k+1} \gets \theta^k_{(m)}$
            \STATE Update Lagrange multiplier:
                \begin{align*}
                    \tilde \Delta_1^k &\gets 2(v_{\theta^{k+1}}^1(s_t) - \varphi_1) - (v_{\theta^{k-1}}^1(s_t) - \varphi_1) \quad \text{(optimistic constraint violation)} \\
                    \mu^{k+1} &\gets \max\{\mu^k + \eps_\mu \tilde \Delta_1^k ,\ 0\}  \quad \text{(projected gradient ascent)}
                \end{align*}
	\end{algorithmic}
\end{algorithm}

\begin{algorithm}[!t]
	\caption{ReLOAD-IMPALA \texttt{GetFixedOutputs}}\label{alg:reload_impala_fixedoutputs}
		\begin{algorithmic}[1] 
		    \STATE // A function to compute everything that does not change within the inner loop.
		    \STATE Require: network parameters $\theta^k$, previous parameters $\textcolor{purple}{\theta^{k-1}}$, trajectory rollouts $\{\tau_n\}_{n=1}^N$
		    \STATE Unroll current parameters: $\pi_{\theta^{k}}(\cdot|s_t), \ \_, \ \_ \gets \texttt{Unroll}(\{\tau_n\}, \theta^k)$
		    \STATE Unroll previous parameters: $\textcolor{purple}{\pi_{\theta^{k-1}}(\cdot|s_t)}, \textcolor{purple}{v^0_{\theta^{k-1}}(s_t)}, \textcolor{purple}{v^1_{\theta^{k-1}}(s_t)} \gets \texttt{Unroll}(\{\tau_n\}, \textcolor{purple}{\theta^{k-1}})$
            \STATE Compute behavior policy log-probs $\log \nu(a_t|s_t)$
            \STATE Compute previous policy importance weights: $\textcolor{purple}{\rho_{\theta^{k-1}}(s_t, a_t)} = \frac{\textcolor{purple}{\pi_{\theta^{k-1}}(a_t|s_t)}}{\nu^k(a_t|s_t)}$
            \STATE Return $\pi_{\theta^{k}}(\cdot|s_t), \textcolor{purple}{\pi_{\theta^{k-1}}(\cdot|s_t)}, \textcolor{purple}{v^0_{\theta^{k-1}}(s_t)}, \textcolor{purple}{v^1_{\theta^{k-1}}(s_t)}, \log \nu(a_t|s_t), \textcolor{purple}{\rho_{\theta^{k-1}}(s_t, a_t)}$
	\end{algorithmic}
\end{algorithm}

\begin{algorithm}[!t]
	\caption{ReLOAD-IMPALA \texttt{Loss}: \\ $\mathcal L(\theta_{(i)}^k, \eta^k, \mu^k, \pi_{\theta^{k}}(\cdot|s_t), \textcolor{purple}{\pi_{\theta^{k-1}}(\cdot|s_t)}, \textcolor{purple}{v^0_{\theta^{k-1}}(s_t)}, \textcolor{purple}{v^1_{\theta^{k-1}}(s_t)}, \log \nu(a_t|s_t), \textcolor{purple}{\rho_{\theta^{k-1}}(s_t, a_t)}, \mu^{k-1}, \{\tau_n\}_{n=1}^N)$ }\label{alg:reload_impala_loss}
		\begin{algorithmic}[1] 
		    \STATE Unroll trajectories:
    		    \begin{align*}
    		    \{ \textcolor{teal}{\pi_{\theta^k_{(i)}}(\cdot|s_t)}, \textcolor{teal}{v^0_{\theta^k_{(i)}}(s_t)}, \textcolor{teal}{v^1_{\theta^k_{(i)}}(s_t)} \} \gets \texttt{Unroll}(\{\tau_n\}, \textcolor{teal}{\theta^k_{(i)}})
    		    \end{align*}
    		\STATE Compute importance ratios wrt behavior policy $\nu^k$:
    		    \begin{align*}
    		        \textcolor{teal}{\rho_{\theta^k_{(i)}}(s_t, a_t)} = \frac{\textcolor{teal}{\pi_{\theta^k_{(i)}}(a_t|s_t)}}{\nu^k(a_t|s_t)}
    		    \end{align*}
    		\STATE Compute rewards: $\textcolor{teal}{r^0_t(\theta_{(i)}^k)}, \textcolor{teal}{r^1_t( \theta_{(i)}^k)} \gets \texttt{ComputeRewards}(r_t, \textcolor{teal}{\theta_{k}^{(i)}})$
    		\STATE Compute TD errors and advantages:
    		    \begin{align*}
    		        \textcolor{teal}{\delta^0_{(i)}, A^0_{\theta^k_{(i)}}} & \gets \texttt{V-Trace}(\textcolor{teal}{v_{\theta^k_{(i)}}^0}, \textcolor{teal}{r^0_t(\theta_{(i)}^k)}, \textcolor{teal}{\rho_{\theta^k_{(i)}}}); \quad \textcolor{teal}{\delta^1_{(i)}, A^1_{\theta^k_{(i}}} \gets \texttt{V-Trace}(\textcolor{teal}{v_{\theta^k_{(i)}}^1}, \textcolor{teal}{r^1_t(\theta_{(i)}^k)}, \textcolor{teal}{\rho_{\theta^k_{(i)}}}) \\
    		        \_, \textcolor{purple}{A^0_{\theta^{k-1}}} &\gets \texttt{V-Trace}(\textcolor{purple}{v_{\theta^{k-1}}^0}, \textcolor{teal}{r^0_t(\theta_{(i)}^k)}, \textcolor{purple}{\rho_{\theta^{k-1}}}); \quad \_, \textcolor{purple}{A^1_{\theta^{k-1}}} \gets \texttt{V-Trace}(\textcolor{purple}{v_{\theta^{k-1}}^1}, \textcolor{teal}{r^1_t(\theta_{(i)}^k)}, \textcolor{purple}{\rho_{\theta^{k-1}}})
    		    \end{align*}
    		\STATE Weight cumulants:
    		    \begin{align*}
    		        \textcolor{teal}{A_{\theta^k_{(i)}}^{\mu^k}} \gets (1 - \sigma(\textcolor{teal}{\mu^k})) \textcolor{teal}{A^0_{\theta^k_{(i)}}} - \sigma(\textcolor{teal}{\mu^k}) \textcolor{teal}{A_{\theta^k_{(i)}}^1}; \quad \textcolor{purple}{A_{\theta^{k-1}}^{\mu^{k-1}}} \gets (1 - \sigma(\textcolor{purple}{\mu^{k-1}})) \textcolor{purple}{A^0_{\theta^{k-1}}} - \sigma(\textcolor{purple}{\mu^{k-1}}) \textcolor{purple}{A_{\theta^{k-1}}^1}
    		    \end{align*}
    		\STATE Compute optimistic advantages: $\textcolor{orange}{\tilde A^{\mu^k}_{\theta^k_{(i)}}} \gets 2\textcolor{teal}{A_{\theta^k_{(i)}}^{\mu^k}} - \textcolor{purple}{A_{\theta^{k-1}}^{\mu^{k-1}}}$
    		\STATE Policy loss:
    		    \begin{align*}
    		        \mathcal L_\pi = \sum_t -\textcolor{teal}{\rho_{\theta^k_{(i)}}(s_t, a_t)} \textcolor{orange}{\tilde A^{\mu^k}_{\theta^k_{(i)}}(s_t, a_t)} \textcolor{teal}{\log \pi_{\theta^k_{(i)}}(a_t|s_t)} + \frac{1}{\eta^k}\kl[\textcolor{teal}{\pi_{\theta^k_{(i)}}(\cdot|s_t)}, \pi_{\theta^k}(\cdot|s_t)]
    		    \end{align*}
    		\STATE Value loss: 
    		    \begin{align*}
    		        \mathcal L_V = \sum_t \textcolor{teal}{\delta^0_{(i)}(s_t, a_t)}^2 + \textcolor{teal}{\delta^1_{(i)}(s_t, a_t)}^2
    		    \end{align*}
    		\STATE Regularization loss: 
    		    \begin{align*}
    		        \mathcal L_\Omega = \sum_t \kl[\textcolor{teal}{\pi_{\theta^k_{(i)}}(\cdot|s_t)}|| \mathcal N(0, I)]
    		    \end{align*}
    		\STATE Return $\alpha_\pi \mathcal L_\pi + \alpha_V \mathcal L_V + \alpha_\Omega \mathcal L_\Omega$
	\end{algorithmic}
\end{algorithm}


\paragraph{Distributional MPO (DMPO)} Maximum a posteriori policy optimization \citep[MPO;][]{abdolmaleki2018mpo} is a deep policy iteration-style algorithm whose improvement step can be broken into two stages which are reminiscent of the ``E'' and ``M'' steps of the EM algorithm. In the E-step, a non-parametric variational ``policy'' $\nu^k(a| s)$ is constructed using off-policy samples as follows:
\begin{align} \label{eq:mpo_variational}
    \nu^k(a|s) \propto \pi_{\theta^k}(a| s)\exp\left(\frac{q_{\theta^k}(s,a)}{\omega^\star}\right)
\end{align}
where $\omega^\star$ is the minimizer of the convex dual function
\begin{align*}
    g(\omega) = \omega \epsilon + \omega \E_{s\sim d_\pi} \E_{a \sim \pi_{\theta^k}} \exp\left(q_{\theta^k}(s,a)/\omega \right)
\end{align*}
where $\epsilon$ is a constant defining the radius of a KL trust region around the current policy. This (soft-)improved policy is then distilled into the parametric policy in an approximate ``M'' step using another KL trust region: 
\begin{align*}
    \theta^{k+1} = \argmax_{\theta} \E_{d_\nu}\E_{\nu^k} \log \pi_{\theta}(a|s) \quad \mathrm{s.t.} \quad \E_{d_\nu} \kl[\pi_{\theta^k}(\cdot|s)||\pi_\theta(\cdot|s)] < \epsilon. 
\end{align*}
In \textit{distributional} MPO \citep[DMPO;][]{abdolmaleki2020dmpo}, in order to provide more accurate value estimates, the standard critic is replaced with a distributional critic using a categorical representation of the return distribution \citep{bellemare2017distributional}. To add Lagrangian optimization to DMPO ($\mu$-DMPO), we used the same strategy as in IMPALA to update both the Lagrange multipliers and to compute the mixed $q$-values $q_{\theta^k}^\mu$ (rather than advantages in this case) to plug into \cref{eq:mpo_variational}. Similarly, for ReLOAD we computed the optimistic mixed $q$-values 
\begin{align*}
    \tilde q_{\theta^k}^{\mu^k}(s,a) = q_{\theta^k}^{\mu^k}(s,a) - q_{\theta^{k-1}}^{\mu^{k-1}}(s,a)
\end{align*}
and plugged them into \cref{eq:mpo_variational}. The MPO policy update code we used can be found here: \url{https://github.com/deepmind/rlax/blob/master/rlax/_src/mpo_ops.py}.

\paragraph{Multi-Objective MPO (MO-MPO)} Rather than scalarize the task and constraint rewards into a single, non-stationary reward via Lagrangian optimization, multi-objective MPO \citep[MO-MPO;][]{abdolmaleki2020dmpo} instead estimates a separate variational policy for each task and constraint reward $\nu_n^k(a|s)$ by solving
\begin{align*}
    \max_{\nu^k_n} \  \E_{d_{\pi_{\theta^k}},\nu_n^k} q_n^k(s,a) \quad \mathrm{s.t.} \quad \E_{d_{\pi_{\theta^k}}} \kl[\nu^k_n(a|s)||\pi_{\theta^k}(a|s)] < \epsilon_n
\end{align*}
for $n = 0, 1, \dots, N$, where $\epsilon_n$ encodes the preferences/thresholds/desired trade-offs among the task and constraint rewards. 
MO-MPO then distills each of these improved policies into the parametric policy:
\begin{align*}
    \theta^{k+1} = \argmax_{\theta} \  \sum_{n=0}^N\E_{d_{\nu^k_n}}\E_{\nu^k_n} \log \pi_{\theta}(a|s) \quad \mathrm{s.t.} \quad \E_{d_\nu} \kl[\pi_{\theta^k}(\cdot|s)||\pi_\theta(\cdot|s)] < \beta. 
\end{align*}
where $\beta\in\reals_+$ is the radius of the trust region.

\paragraph{Reward-Constrained D4PG (RC-D4PG)} Deep deterministic policy gradients \citep[DDPG;][]{lillicrap2015ddpg} is a DRL algorithm which uses a deterministic policy $\pi_{\theta_\pi}(s)$ to maximize the objective $J(\theta) = \E[q_{\theta_q}(s,a)\mid s = s_t, a = \pi_{\theta_\pi}(s_t)]$ (where $\theta = \{\theta_\pi, \theta_q\}$) using the deterministic policy gradient estimate $\grad_{\theta_\pi}J(\theta) = \E[\grad_a q_{\theta_q}(s,a) \grad_{\theta_\pi}\pi_{\theta_\pi}(s_t)]$. Distributed distributional DDPG \citep[D4PG;][]{barthmaron2018d4pg} is a distributed algorithm which augments DDPG with distributional critics using a categorical representation for the return distribution \citep{bellemare2017distributional}. Reward-constrained D4PG \citep[RC-D4PG;][]{calian2020metal} augments D4PG for Lagrangian optimization in the same way as $\mu$-IMPALA and $\mu$-MDPO augment their respective base algorithms, with the exception that rather than update the Lagrange multiplier by using constraint value estimates obtained from samples from the most recent policy, RC-D4PG maintains a buffer of constraint returns from across training and uses them to compute constraint violations. 

\paragraph{MetaL} Meta-gradients for the Lagrange learning rate \citep[MetaL;][]{calian2020metal} extends RC-D4PG by using meta-gradients \citep{xu2018metagrads} to meta-learn the learning rate for the Lagrange multiplier $\eta_\mu$ with the goal of stabilizing optimization. Specifically, they define a set of inner losses for the actor, critic, and Lagrange multiplier corresponding to the standard components of Lagrangian optimization for CMDPs while also defining an outer loss $L_{outer} = L_{critic}(\theta_q^{k+1}(\eta_\mu^k), \mu^{k+1}(\eta_\mu^k))$. That is, the outer loss is the standard critic loss evaluated using the updated critic parameters and Lagrange multipliers from the inner loop (which are themselves functions of the current Lagrange multiplier learning rate). The Lagrange multiplier learning rate is then updated using meta-gradients on this loss.

\begin{algorithm}[!t]
	\caption{\textcolor{MidnightBlue}{ReLOAD}-MDPI}\label{alg:tabular_reload}
		\begin{algorithmic}[1] 
		    \STATE Require: CMDP $\mathcal M_C$, step sizes $\{\eta_\pi, \eta_\mu\} > 0$
		    \STATE Initialize $ \pi^1$, $\mu^1$, $\pi^0$, $\mu^0$
            \FOR{$k=1,\dots,K$}
                \STATE $\{ q_n\}_{n=0}^N \gets \texttt{PolicyEval}(\mathcal M_C,  \pi^k)$
                \STATE $ q_{\mu}^k \gets - q_0^k + \sum_{n=1}^N \mu^k_n  q_n^k$ \quad (mixed $q$-values)
                \STATE $ v_{1:N}^k \gets [\langle  q_1^k,  \pi^k\rangle, \dots,  \langle  q_N^k,  \pi^k\rangle]^\top$
                \STATE Update policy with OMWU:
                    \begin{align*}
                         \pi^{k+1} &= \frac{ \pi^k \exp\left( ( \textcolor{MidnightBlue}{2} q_\mu^k \textcolor{MidnightBlue}{-\  q_\mu^{k-1}})/\eta_\pi \right)}{\left\langle  \pi^k \exp\left( (\textcolor{MidnightBlue}{2} q_\mu^k \textcolor{MidnightBlue}{-\  q_\mu^{k-1}})/\eta_\pi \right), \mathbf 1 \right\rangle} 
                    \end{align*} 
                \STATE Update Lagrange multiplier with projected OGA:
                    \begin{align*}
                         \mu^{k+1} &= \max\{ \mu^k + \eta_\mu (\textcolor{MidnightBlue}{2}  v^k_{1:N} \textcolor{MidnightBlue}{-\  v^{k-1}_{1:N}} -  \theta), \  0 \}
                    \end{align*}
            \ENDFOR
            \STATE \textbf{return} $\pi^K, \mu^K$
	\end{algorithmic}
\end{algorithm}

\section{Further Experimental Details} \label{sect:experiment_details}

\paragraph{Performance Measures} In simple cases like the paradoxical CMDP, we know the optimal solution exactly and can compute, e.g., the $L_2$ distance between the iterates produced by the algorithm $(d_\pi^k, \mu^k)$ and the SP $(d_\pi^\star, \mu^\star)$. In more complex settings such as the RWRL suite and control suite, previous work, e.g., \citet{calian2020metal} and \citet{huang2022lp3} has used the \textit{penalized reward} $R_{\mathrm{penalized}}$:
\begin{align}
    R_{\mathrm{penalized}} = R_0 - \sum_{n=1}^N \max\{R_n - \theta_n, 0 \},
\end{align}
which is the average task return $R_0$ minus the \textit{constraint overshoot}---how much the average constraint returns violate their associated constraint thresholds. (No bonus is given for being further beneath the threshold than necessary.) However, this metric ignores the effect of the Lagrange multiplier, which in Lagrangian optimization acts on a coefficient on the constraint term, e.g., $\mathcal L = v_0 + \mu(v_1 - \theta_1)$ for a single constraint. We therefore prefer the \textit{weighted reward} $R_{\mathrm{weighted}}$ given by
\begin{align}
    R_{\mathrm{weighted}} = R_0 - \sum_{n=1}^N \mu^\star_n \max\{R_n - \theta_n, 0\},
\end{align}
where $\mu_\star^n$ is the optimal Lagrange multiplier. In our case, for large-scale environments we use a sigmoid to bound the Lagrange multiplier as described in \cref{sect:algorithm_details}. Therefore, because the agent is maximizing the weighted value/advantage $(1 - \sigma(\mu)) A^0 + \sigma(\mu) A^1$, we divide by $1 - \sigma(\mu)$ to give
\begin{align}
    R_{\mathrm{weighted}} = R_0 - \sum_{n=1}^N \hat \mu_n^\star \max\{R_n - \theta_n, 0\} \quad \mathrm{where} \quad \hat \mu_n^\star \coloneqq \frac{\sigma(\mu^\star_n)}{1 - \sigma(\mu^\star_n)}. 
\end{align}
This approach is also used by \citet{stooke2020pid} and \citet{zahavy2022domino}. To estimate the optimal Lagrange multiplier(s), we ran $\mu$- agents for 8 random seeds on each task and averaged the Lagrange multipliers, as they enjoy guarantees for average-iterate convergence. 

\paragraph{The Paradoxical CMDP}
Tabular agents were trained for 500 episodes, with each episode lasting 10 time steps and with the agent starting in a uniformly-randomly chosen state. The discount factor $\gamma$ was 0.9. $\mu$-MDPI-Avg was obtained by maintaining a running average of the iterates of $\mu$-MDPI. DRL agents ($\mu$-MDPO and ReLOAD-MDPO) were trained for 30,000 episodes. The policy was parameterized as an MLP with a single hidden layer with 16 units. States were provided as a one-hot encoding. The agents were trained using RMSProp with an initial learning rate of 6e-4, which decayed linearly to 1e-4 over the course of training. There were 5 inner loop steps to approximate the MD update with a fixed MD step size of 0.25.  

\paragraph{Catch} 
The Catch task was as described by \citet{osband2019bsuite} with the addition of the constraint reward and threshold as described in \cref{sec:experiments_main}. $\mu$-MDPO and ReLOAD-MDPO agents were trained for 30,000 episodes with the same hyperparameter settings as for the Paradoxical CMDP, with the exception that the policy network had two hidden layers, each with 32 hidden units. 

\paragraph{The RWRL Suite} Hyperparameters and training regimes for 0.1-D4PG, RC-D4PG, and MetaL are as reported by \citet{calian2020metal}. For $\mu$-DMPO and ReLOAD-DMPO, hyperparameters are given in \cref{tab:dmpo_hparams}. All task settings are as reported in \citet{calian2020metal}. 

\paragraph{Oscillating Control Suite}

Control suite domains and accompanying thresholds are described in \cref{tab:oscillating_settings}. OGD-{IMPALA, DMPO} is implemented by augmenting $\mu$-{IMPALA, DMPO} with the optimistic gradient descent optimizer from the Optax library, whose code is available at: \url{https://github.com/deepmind/optax/blob/master/optax/_src/alias.py}. PID control is implemented as described in \citet{stooke2020pid} with the same values of $K_P = 0.95$ and $K_D = 0.9$ as used in that work. For IMPALA, which does not normally have a trust region component to the policy update, RNTR-IMPALA implemented ReLOAD without a trust region (that is, only using optimistic advantages/value estimates).

\begin{table}[t]
    \small
    \centering
    \setlength\tabcolsep{3pt}
    \begin{tabular}{lccc}
        \toprule \hline
        \textbf{Domain, Task} & \textbf{Constrained Quantity} & \textbf{Observation Dims} & $\mathbf{\theta}$ \Tstrut \\
        \midrule
        Walker, Walk & Height & 24 & 0.5 \\
        Walker, Walk & Velocity & 14-23 & 155 \\
        Reacher, Easy & Velocity & 4-5 & 1.02 \\
        Quadruped, Walk & Torque & 54-77 & 610 \\
        Humanoid, Walk & Height & 21 & 400 \\
        \bottomrule
    \end{tabular}
    \caption{Experiment settings for oscillating control suite experiments. For multidimensional quantities, the constraint was placed on the $L_2$ norm.}
    \label{tab:oscillating_settings}
    \vspace{-1em}
\end{table}

\section{Further Experimental Results} \label{sect:experiment_results}

\begin{figure}[ht] \label{fig:peg_mdpi}
    \begin{center}
    \centerline{\includegraphics[width=0.5\linewidth]{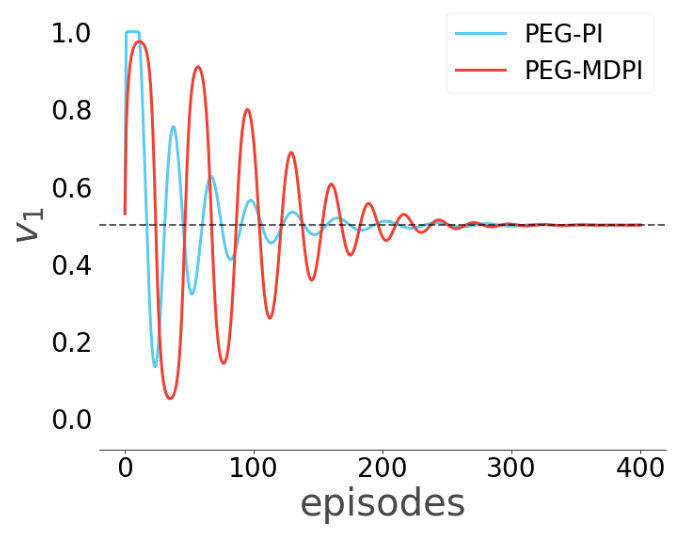}} 
    \caption{\textbf{Tabular PEG-based policy iteration in the paradoxical CMDP.} PEG-PI uses projected optimistic gradient descent instead of MWU for the policy, while PEG-MDPI uses optimistic MWU for the policy update, which here produces the same updates as ReLOAD. Both variants converge in the last-iterate.}
    \end{center}
\end{figure}

\begin{figure}[ht] \label{fig:toy_extra}
    \begin{center}
    \centerline{\includegraphics[width=0.8\linewidth]{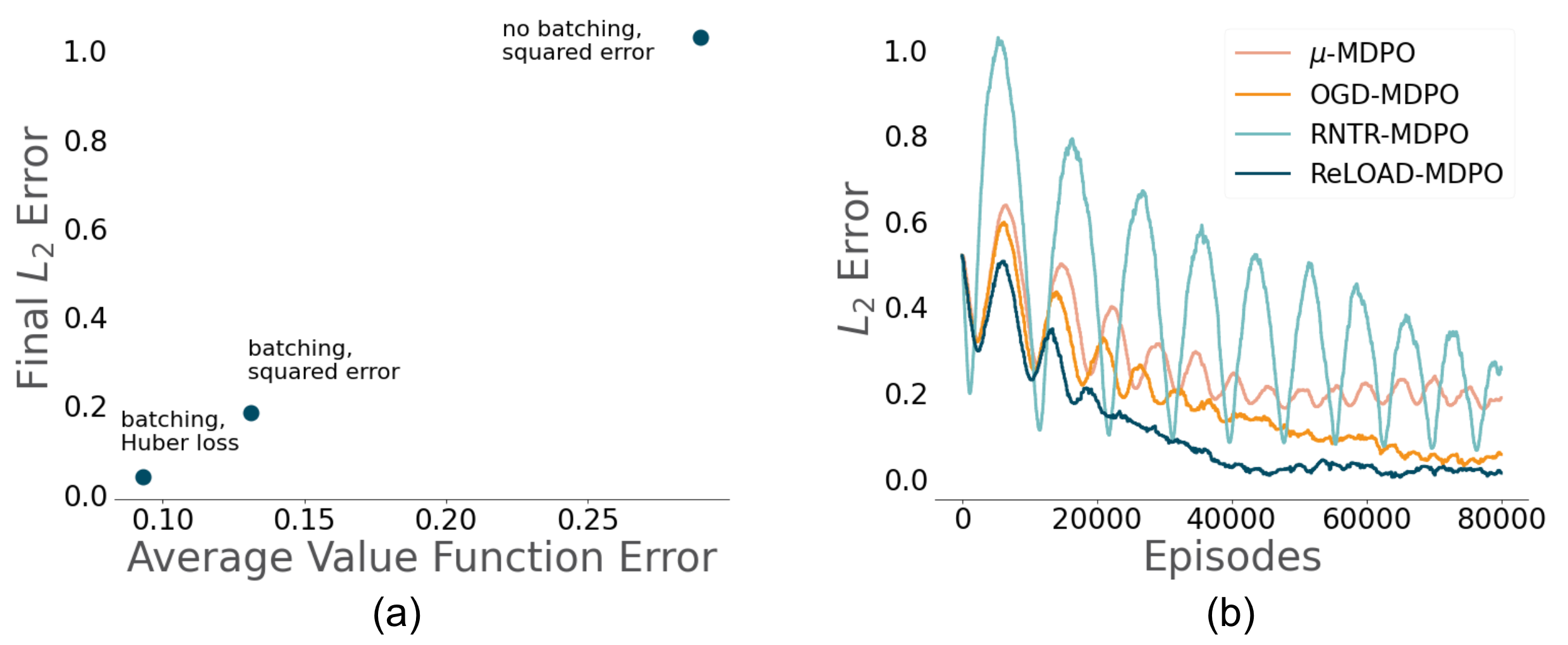}} 
    \caption{\textbf{Policy evaluation and ablations in the paradoxical CMDP.} (a) To demonstrate the importance of accurate value estimation on performance, we plotted the $L_2$ distance of the final $(d_\pi^K, \mu^K)$ pair obtained by ReLOAD-MDPO from the SP of the paradoxical CMDP in \cref{fig:toy}a as a function of the average error in the value estimates over the course of training. We can see that performance gets dramatically worse as the value function error increases, and that adding standard techniques like batching and Huber loss rather than squared error loss improves value estimates sufficiently to allow for strong performance of the overall algorithm. (b) Here, we compare ReLOAD-MDPO against various ablations by measuring the $L_2$ distance from the SP of the paradoxical CMDP over the course of training. We can see that both $\mu$-MDPO and ReLOAD without a trust region (RNTR-) fail to converge. However, performing OGD directly on the policy parameters, rather than via optimistic value estimates, performs nearly as well as ReLOAD. This is likely because the CMDP only has two states, so using the gradient computed from value estimates obtained from an old batch of data will likely be close to that obtained from the current batch of data, as there's more likely to be significant overlap. If we look to higher-dimensional tasks like control suite (\cref{fig:osc}), there is a much larger gap between the direct OGD approach and ReLOAD as the state-action pairs used for value estimation are more likely to differ from one minibatch to another.}
    \label{fig:toy_extras}
    \end{center}
\end{figure}

\begin{figure}[ht]
    \begin{center}
    \centerline{\includegraphics[width=0.66\linewidth]{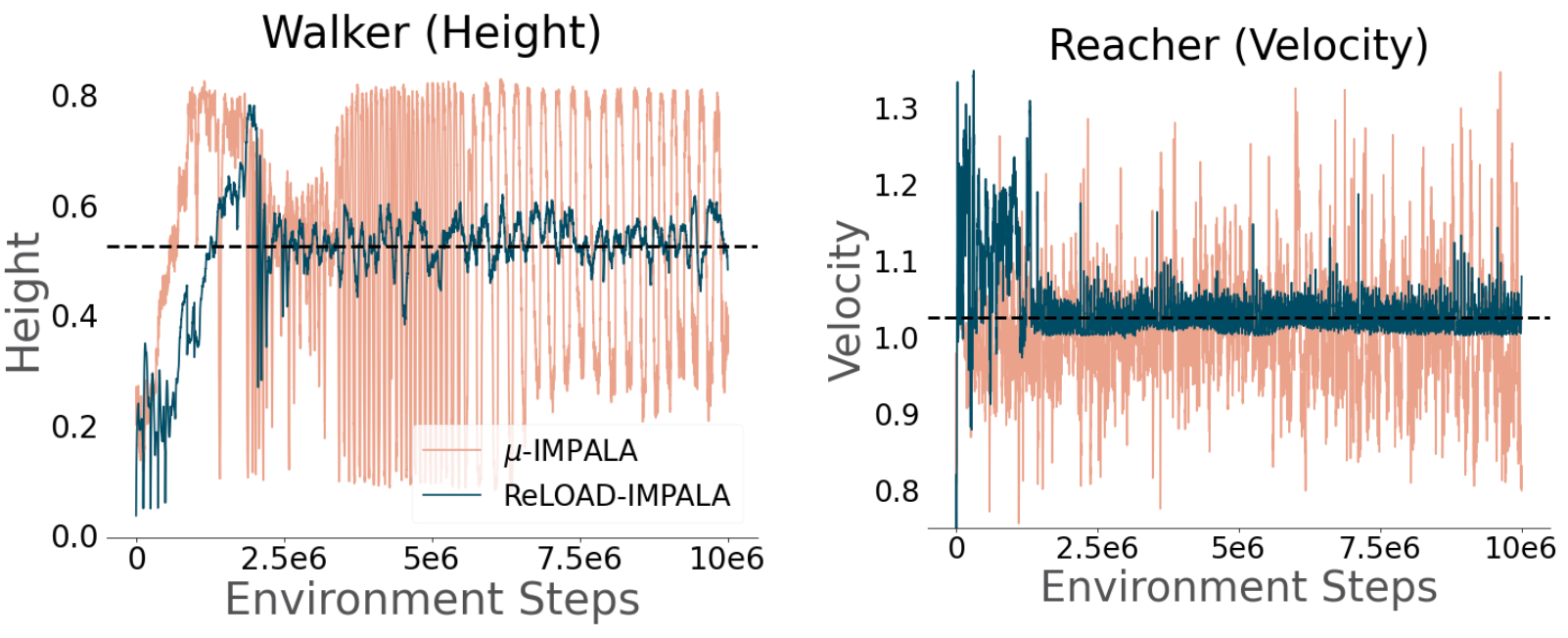}}
    \caption{Learning curves for $\mu$-IMPALA and ReLOAD-IMPALA on \texttt{Walker, Walk} with a height constraint (left) and \texttt{Reacher, Easy} with a velocity constraint (right). The horizontal dotted line indicates the value of the constraint. In each case, ReLOAD-IMPALA significantly dampens oscillations compared to $\mu$-IMPALA.}
    \label{fig:impala_curves}
    \end{center}
\end{figure}

\begin{figure}[ht]
    \begin{center}
    \centerline{\includegraphics[width=0.99\linewidth]{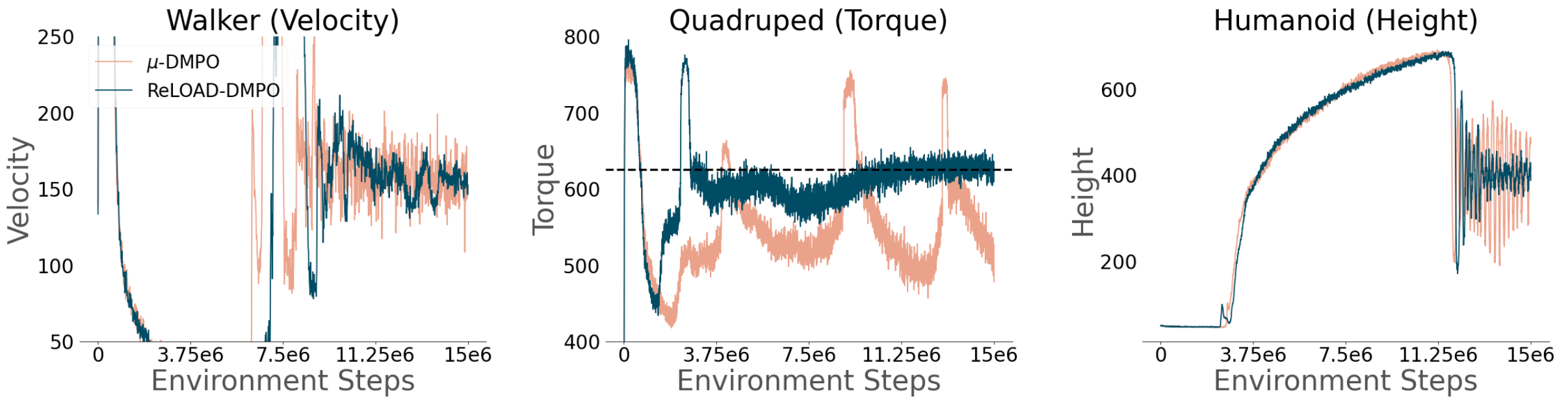}}
    \caption{Learning curves for $\mu$-DMPO and ReLOAD-DMPO on \texttt{Walker, Walk} with a velocity constraint (left), \texttt{Quadruped, Walk} with a torque constraint (center), and \texttt{Humanoid, Walk} with a height constraint (right). The horizontal dotted line indicates the value of the constraint. In each case, ReLOAD-DMPO significantly dampens oscillations compared to $\mu$-DMPO.}
    \label{fig:dmpo_curves}
    \end{center}
\end{figure}

\begin{table}[t]
    \small
    \centering
    \setlength\tabcolsep{3pt}
    \begin{tabular}{llccccc}
        \toprule \hline
        \textbf{Domain} & \textbf{Algorithm} & \textbf{Safety-Coeff/Threshold} & \textbf{Weighted Reward} & $\mathbf{R_{penalized}}$ & \textbf{Task Reward} & \textbf{Constraint Violation} \Tstrut \\
        \midrule
        Walker & 0.1-D4PG & 0.05/0.057 & $-153.8 \pm 12.1$ & $552.8 \pm 12.0$ & $979.6 \pm 0.3$ & $426.8 \pm 12.1$ \\
         &  & 0.05/0.077 & $-175.4 \pm 20.0$ & $544.3 \pm 19.9$ & $979.1 \pm 0.5$ & $434.8 \pm 20.0$ \\
         &  & 0.05/0.097 & $-75.1 \pm 20.5$ & $581.8 \pm 20.3$ & $978.6 \pm 0.4$ & $396.8 \pm 20.5$ \\
         &  & 0.1/0.057 & $533.9 \pm 9.0$ & $811.3 \pm 8.7$ & $978.9 \pm 0.4$ & $167.6 \pm 9.0$ \\
         &  & 0.1/0.077 & $630.2 \pm 7.9$ & $847.7 \pm 7.7$ & $979.1 \pm 0.3$ & $131.4 \pm 7.9$ \\
         &  & 0.1/0.097 & $652.4 \pm 5.8$ & $855.8 \pm 5.7$ & $978.7 \pm 0.3$ & $122.9 \pm 5.8$ \\
         &  & 0.2/0.057 & $982.0 \pm 0.2$ & $982.0 \pm 0.2$ & $982.0 \pm 0.2$ & $0.0 \pm 0.0$ \\
         &  & 0.2/0.077 & $981.8 \pm 0.2$ & $981.8 \pm 0.2$ & $981.8 \pm 0.2$ & $0.0 \pm 0.0$ \\
         &  & 0.2/0.097 & $981.5 \pm 0.3$ & $981.5 \pm 0.3$ & $981.5 \pm 0.3$ & $0.0 \pm 0.0$ \\
         &  & 0.3/0.057 & $982.7 \pm 0.3$ & $982.7 \pm 0.3$ & $982.7 \pm 0.3$ & $0.0 \pm 0.0$ \\
         &  & 0.3/0.077 & $981.8 \pm 0.3$ & $981.8 \pm 0.3$ & $981.8 \pm 0.3$ & $0.0 \pm 0.0$ \\
         &  & 0.3/0.097 & $982.0 \pm 0.1$ & $982.0 \pm 0.1$ & $982.0 \pm 0.1$ & $0.0 \pm 0.0$ \\
         & RC-D4PG & 0.05/0.057 & $268.7 \pm 34.6$ & $296.1 \pm 28.3$ & $312.7 \pm 34.0$ & $16.6 \pm 6.4$ \\
         &  & 0.05/0.077 & $325.9 \pm 44.4$ & $346.6 \pm 36.9$ & $359.0 \pm 43.8$ & $12.5 \pm 7.2$ \\
         &  & 0.05/0.097 & $329.1 \pm 33.4$ & $334.8 \pm 32.8$ & $338.2 \pm 33.4$ & $3.4 \pm 1.3$ \\
         &  & 0.1/0.057 & $914.3 \pm 10.9$ & $915.0 \pm 10.8$ & $915.5 \pm 10.9$ & $0.5 \pm 0.4$ \\
         &  & 0.1/0.077 & $942.7 \pm 1.6$ & $945.4 \pm 1.2$ & $947.1 \pm 1.4$ & $1.7 \pm 0.8$ \\
         &  & 0.1/0.097 & $956.1 \pm 3.3$ & $959.3 \pm 3.3$ & $961.3 \pm 3.2$ & $2.0 \pm 0.8$ \\
         &  & 0.2/0.057 & $980.5 \pm 0.3$ & $981.0 \pm 0.4$ & $981.4 \pm 0.2$ & $0.3 \pm 0.2$ \\
         &  & 0.2/0.077 & $979.1 \pm 1.0$ & $980.7 \pm 1.1$ & $981.7 \pm 0.3$ & $1.0 \pm 1.0$ \\
         &  & 0.2/0.097 & $977.7 \pm 0.8$ & $980.3 \pm 0.9$ & $981.9 \pm 0.3$ & $1.6 \pm 0.8$ \\
         &  & 0.3/0.057 & $978.4 \pm 0.5$ & $980.5 \pm 0.5$ & $981.8 \pm 0.3$ & $1.3 \pm 0.4$ \\
         &  & 0.3/0.077 & $978.6 \pm 0.9$ & $980.5 \pm 0.7$ & $981.6 \pm 0.3$ & $1.1 \pm 0.8$ \\
         &  & 0.3/0.097 & $979.6 \pm 0.6$ & $980.9 \pm 0.6$ & $981.7 \pm 0.4$ & $0.8 \pm 0.5$ \\
         & MetaL & 0.05/0.057 & $-168.5 \pm 10.5$ & $540.0 \pm 9.6$ & $968.0 \pm 1.5$ & $428.0 \pm 10.4$ \\
         &  & 0.05/0.077 & $-138.8 \pm 13.7$ & $551.7 \pm 13.5$ & $968.9 \pm 1.4$ & $417.1 \pm 13.6$ \\
         &  & 0.05/0.097 & $-131.1 \pm 13.0$ & $553.9 \pm 12.2$ & $967.8 \pm 1.0$ & $413.8 \pm 13.0$ \\
         &  & 0.1/0.057 & $663.4 \pm 19.8$ & $853.7 \pm 19.8$ & $968.6 \pm 1.2$ & $114.9 \pm 19.8$ \\
         &  & 0.1/0.077 & $627.4 \pm 17.8$ & $841.7 \pm 17.1$ & $971.2 \pm 1.1$ & $129.5 \pm 17.8$ \\
         &  & 0.1/0.097 & $766.9 \pm 19.8$ & $893.3 \pm 19.7$ & $969.7 \pm 0.7$ & $76.4 \pm 19.7$ \\
         &  & 0.2/0.057 & $981.3 \pm 0.5$ & $982.0 \pm 0.6$ & $982.5 \pm 0.2$ & $0.4 \pm 0.4$ \\
         &  & 0.2/0.077 & $980.2 \pm 0.3$ & $981.5 \pm 0.3$ & $982.3 \pm 0.1$ & $0.8 \pm 0.3$ \\
         &  & 0.2/0.097 & $978.6 \pm 0.8$ & $980.7 \pm 0.7$ & $981.9 \pm 0.3$ & $1.3 \pm 0.7$ \\
         &  & 0.3/0.057 & $977.7 \pm 0.9$ & $980.6 \pm 1.0$ & $982.3 \pm 0.4$ & $1.7 \pm 0.8$ \\
         &  & 0.3/0.077 & $978.4 \pm 0.5$ & $980.6 \pm 0.6$ & $981.9 \pm 0.2$ & $1.3 \pm 0.5$ \\
         &  & 0.3/0.097 & $979.6 \pm 0.5$ & $980.9 \pm 0.5$ & $981.6 \pm 0.2$ & $0.8 \pm 0.5$ \\
         & ReLOAD & 0.05/0.057 & $367.0 \pm 114.8$ & $626.9 \pm 113.4$ & $783.8 \pm 96.9$ & $157.0 \pm 9.6$ \\
         &  & 0.05/0.077 & $423.3 \pm 80.1$ & $716.7 \pm 43.5$ & $893.9 \pm 26.1$ & $177.2 \pm 24.0$ \\
         &  & 0.05/0.097 & $315.8 \pm 43.1$ & $677.7 \pm 28.9$ & $896.4 \pm 22.1$ & $218.6 \pm 10.3$ \\
         &  & 0.1/0.057 & $754.1 \pm 39.0$ & $875.4 \pm 23.7$ & $948.7 \pm 15.8$ & $73.3 \pm 10.1$ \\
         &  & 0.1/0.077 & $717.9 \pm 26.8$ & $877.1 \pm 11.2$ & $973.3 \pm 4.8$ & $96.2 \pm 9.9$ \\
         &  & 0.1/0.097 & $746.9 \pm 28.4$ & $889.3 \pm 12.0$ & $975.3 \pm 4.7$ & $86.0 \pm 10.4$ \\
         &  & 0.2/0.057 & $907.7 \pm 6.9$ & $950.6 \pm 5.0$ & $976.6 \pm 4.1$ & $26.0 \pm 1.9$ \\
         &  & 0.2/0.077 & $881.5 \pm 9.7$ & $940.7 \pm 5.9$ & $976.4 \pm 4.3$ & $35.7 \pm 3.1$ \\
         &  & 0.2/0.097 & $882.7 \pm 7.8$ & $942.1 \pm 5.2$ & $977.9 \pm 4.0$ & $35.9 \pm 2.4$ \\
         &  & 0.3/0.057 & $938.3 \pm 5.8$ & $963.0 \pm 5.0$ & $977.9 \pm 4.3$ & $14.9 \pm 1.2$ \\
         &  & 0.3/0.077 & $930.4 \pm 5.6$ & $960.5 \pm 4.6$ & $978.8 \pm 4.0$ & $18.2 \pm 1.3$ \\
         &  & 0.3/0.097 & $925.8 \pm 5.8$ & $958.3 \pm 4.7$ & $977.9 \pm 4.0$ & $19.6 \pm 1.5$ \\
        \midrule
        \bottomrule
    \end{tabular}
    \caption{Mean performance for \texttt{RWRL-Walker} averaged over 8 random seeds for each of the 12 constraint settings in the RWRL suite, where lower values of the safety coefficient and threshold indicate more challenging tasks. $\pm$ values denote one standard error. }
    \label{tab:rwrl_walker}
    \vspace{-1em}
\end{table}

\begin{table}[t]
    \small
    \centering
    \setlength\tabcolsep{3pt}
    \begin{tabular}{llccccc}
        \toprule \hline
        \textbf{Domain} & \textbf{Algorithm} & \textbf{Safety-Coeff/Threshold} & \textbf{Weighted Reward} & $\mathbf{R_{penalized}}$ & \textbf{Task Reward} & \textbf{Constraint Violation} \Tstrut \\
        \midrule
        Cartpole & 0.1-D4PG & 0.05/0.07 & $833.8 \pm 1.3$ & $352.6 \pm 1.2$ & $851.7 \pm 0.2$ & $499.1 \pm 1.3$ \\
         &  & 0.05/0.09 & $834.3 \pm 1.5$ & $372.6 \pm 1.5$ & $851.5 \pm 0.1$ & $478.9 \pm 1.5$ \\
         &  & 0.05/0.115 & $835.0 \pm 0.9$ & $399.7 \pm 0.8$ & $851.2 \pm 0.1$ & $451.5 \pm 0.9$ \\
         &  & 0.1/0.07 & $844.9 \pm 0.7$ & $568.2 \pm 0.7$ & $855.2 \pm 0.1$ & $287.0 \pm 0.7$ \\
         &  & 0.1/0.09 & $845.5 \pm 1.0$ & $588.3 \pm 1.0$ & $855.1 \pm 0.2$ & $266.8 \pm 1.0$ \\
         &  & 0.1/0.115 & $846.4 \pm 0.9$ & $612.9 \pm 0.9$ & $855.1 \pm 0.2$ & $242.2 \pm 0.9$ \\
         &  & 0.2/0.07 & $850.7 \pm 4.8$ & $763.1 \pm 6.3$ & $853.9 \pm 2.4$ & $90.8 \pm 4.2$ \\
         &  & 0.2/0.09 & $854.7 \pm 1.3$ & $791.3 \pm 1.5$ & $857.0 \pm 0.5$ & $65.7 \pm 1.2$ \\
         &  & 0.2/0.115 & $852.0 \pm 4.4$ & $807.6 \pm 5.8$ & $853.6 \pm 2.4$ & $46.0 \pm 3.7$ \\
         &  & 0.3/0.07 & $857.5 \pm 0.3$ & $825.5 \pm 0.3$ & $858.7 \pm 0.1$ & $33.3 \pm 0.3$ \\
         &  & 0.3/0.09 & $858.4 \pm 0.6$ & $845.4 \pm 0.5$ & $858.9 \pm 0.1$ & $13.5 \pm 0.5$ \\
         &  & 0.3/0.115 & $856.7 \pm 1.9$ & $856.6 \pm 2.0$ & $856.7 \pm 1.9$ & $0.1 \pm 0.1$ \\
         & RC-D4PG & 0.05/0.07 & $278.6 \pm 114.7$ & $158.7 \pm 21.7$ & $283.1 \pm 88.2$ & $124.4 \pm 73.3$ \\
         &  & 0.05/0.09 & $260.4 \pm 84.4$ & $144.8 \pm 43.9$ & $264.7 \pm 66.4$ & $119.9 \pm 52.2$ \\
         &  & 0.05/0.115 & $422.8 \pm 105.5$ & $269.1 \pm 23.0$ & $428.5 \pm 83.5$ & $159.4 \pm 64.4$ \\
         &  & 0.1/0.07 & $262.6 \pm 48.7$ & $180.4 \pm 49.6$ & $265.7 \pm 43.1$ & $85.3 \pm 22.6$ \\
         &  & 0.1/0.09 & $534.4 \pm 104.0$ & $296.2 \pm 93.7$ & $543.3 \pm 81.9$ & $247.1 \pm 64.1$ \\
         &  & 0.1/0.115 & $710.0 \pm 45.1$ & $537.5 \pm 35.5$ & $716.4 \pm 38.8$ & $178.9 \pm 23.0$ \\
         &  & 0.2/0.07 & $310.9 \pm 90.6$ & $298.3 \pm 86.4$ & $311.4 \pm 90.3$ & $13.1 \pm 7.5$ \\
         &  & 0.2/0.09 & $770.0 \pm 18.0$ & $760.6 \pm 16.4$ & $770.3 \pm 17.7$ & $9.7 \pm 3.0$ \\
         &  & 0.2/0.115 & $847.7 \pm 2.4$ & $847.4 \pm 2.4$ & $847.7 \pm 2.4$ & $0.3 \pm 0.1$ \\
         &  & 0.3/0.07 & $845.2 \pm 3.4$ & $844.7 \pm 3.6$ & $845.3 \pm 3.4$ & $0.6 \pm 0.3$ \\
         &  & 0.3/0.09 & $857.3 \pm 0.2$ & $857.1 \pm 0.1$ & $857.3 \pm 0.1$ & $0.2 \pm 0.1$ \\
         &  & 0.3/0.115 & $859.2 \pm 0.3$ & $858.7 \pm 0.2$ & $859.2 \pm 0.1$ & $0.4 \pm 0.3$ \\
         & MetaL & 0.05/0.07 & $832.9 \pm 1.1$ & $348.2 \pm 1.1$ & $850.9 \pm 0.2$ & $502.7 \pm 1.1$ \\
         &  & 0.05/0.09 & $829.5 \pm 2.3$ & $365.1 \pm 2.5$ & $846.8 \pm 2.1$ & $481.6 \pm 1.0$ \\
         &  & 0.05/0.115 & $834.4 \pm 1.7$ & $391.9 \pm 1.6$ & $850.9 \pm 0.2$ & $458.9 \pm 1.7$ \\
         &  & 0.1/0.07 & $843.4 \pm 1.0$ & $580.3 \pm 0.9$ & $853.2 \pm 0.2$ & $272.9 \pm 1.0$ \\
         &  & 0.1/0.09 & $844.3 \pm 2.3$ & $599.3 \pm 2.3$ & $853.4 \pm 0.2$ & $254.1 \pm 2.3$ \\
         &  & 0.1/0.115 & $845.5 \pm 1.2$ & $623.7 \pm 1.3$ & $853.7 \pm 0.1$ & $230.0 \pm 1.2$ \\
         &  & 0.2/0.07 & $851.1 \pm 2.6$ & $787.3 \pm 3.3$ & $853.5 \pm 2.0$ & $66.2 \pm 1.6$ \\
         &  & 0.2/0.09 & $850.8 \pm 3.8$ & $809.3 \pm 4.2$ & $852.3 \pm 2.1$ & $43.1 \pm 3.1$ \\
         &  & 0.2/0.115 & $855.1 \pm 2.1$ & $831.8 \pm 1.9$ & $855.9 \pm 0.2$ & $24.1 \pm 2.1$ \\
         &  & 0.3/0.07 & $856.1 \pm 1.9$ & $839.9 \pm 1.5$ & $856.7 \pm 0.4$ & $16.8 \pm 1.9$ \\
         &  & 0.3/0.09 & $852.8 \pm 3.2$ & $848.5 \pm 4.1$ & $852.9 \pm 2.7$ & $4.4 \pm 1.6$ \\
         &  & 0.3/0.115 & $857.3 \pm 2.0$ & $856.9 \pm 2.1$ & $857.3 \pm 2.0$ & $0.4 \pm 0.2$ \\
         & ReLOAD & 0.05/0.07 & $875.8 \pm 2.6$ & $798.3 \pm 9.4$ & $878.7 \pm 2.4$ & $80.4 \pm 7.4$ \\
         &  & 0.05/0.09 & $871.5 \pm 3.0$ & $781.0 \pm 11.2$ & $874.8 \pm 2.8$ & $93.8 \pm 8.7$ \\
         &  & 0.05/0.115 & $869.1 \pm 3.0$ & $769.0 \pm 11.3$ & $872.8 \pm 2.7$ & $103.9 \pm 8.9$ \\
         &  & 0.1/0.07 & $875.8 \pm 2.6$ & $801.9 \pm 9.0$ & $878.5 \pm 2.4$ & $76.6 \pm 6.9$ \\
         &  & 0.1/0.09 & $869.5 \pm 2.9$ & $780.7 \pm 9.9$ & $872.8 \pm 2.7$ & $92.1 \pm 7.6$ \\
         &  & 0.1/0.115 & $871.7 \pm 3.0$ & $786.6 \pm 10.1$ & $874.8 \pm 2.8$ & $88.2 \pm 7.7$ \\
         &  & 0.2/0.07 & $869.9 \pm 2.9$ & $795.1 \pm 8.1$ & $872.7 \pm 2.7$ & $77.5 \pm 6.0$ \\
         &  & 0.2/0.09 & $871.9 \pm 3.0$ & $798.7 \pm 8.9$ & $874.7 \pm 2.8$ & $76.0 \pm 6.5$ \\
         &  & 0.2/0.115 & $869.8 \pm 2.9$ & $789.2 \pm 8.6$ & $872.7 \pm 2.7$ & $83.5 \pm 6.4$ \\
         &  & 0.3/0.07 & $870.1 \pm 2.8$ & $802.9 \pm 7.5$ & $872.6 \pm 2.7$ & $69.7 \pm 5.5$ \\
         &  & 0.3/0.09 & $867.8 \pm 2.6$ & $797.1 \pm 7.3$ & $870.5 \pm 2.5$ & $73.4 \pm 5.5$ \\
         &  & 0.3/0.115 & $867.6 \pm 2.6$ & $793.3 \pm 7.3$ & $870.4 \pm 2.4$ & $77.1 \pm 5.5$ \\
        \midrule
        \bottomrule
    \end{tabular}
    \caption{Mean performance for \texttt{RWRL-Cartpole} averaged over 8 random seeds for each of the 12 constraint settings in the RWRL suite, where lower values of the safety coefficient and threshold indicate more challenging tasks. $\pm$ values denote one standard error. }
    \label{tab:rwrl_cartpole}
    \vspace{-1em}
\end{table}

\begin{table}[t]
    \small
    \centering
    \setlength\tabcolsep{3pt}
    \begin{tabular}{llccccc}
        \toprule \hline
        \textbf{Domain} & \textbf{Algorithm} & \textbf{Safety-Coeff/Threshold} & \textbf{Weighted Reward} & $\mathbf{R_{penalized}}$ & \textbf{Task Reward} & \textbf{Constraint Violation} \Tstrut \\
        \midrule
        Quadruped & 0.1-D4PG & 0.05/0.545 & $109.5 \pm 0.0$ & $546.6 \pm 0.0$ & $999.6 \pm 0.0$ & $453.0 \pm 0.0$ \\
         &  & 0.05/0.645 & $305.9 \pm 0.0$ & $646.5 \pm 0.0$ & $999.5 \pm 0.0$ & $353.0 \pm 0.0$ \\
         &  & 0.05/0.745 & $502.4 \pm 0.0$ & $746.5 \pm 0.0$ & $999.5 \pm 0.0$ & $253.0 \pm 0.0$ \\
         &  & 0.1/0.545 & $523.7 \pm 0.0$ & $547.5 \pm 0.0$ & $999.5 \pm 0.0$ & $452.0 \pm 0.0$ \\
         &  & 0.1/0.645 & $629.0 \pm 0.0$ & $647.5 \pm 0.0$ & $999.5 \pm 0.0$ & $352.0 \pm 0.0$ \\
         &  & 0.1/0.745 & $734.2 \pm 0.0$ & $747.5 \pm 0.0$ & $999.5 \pm 0.0$ & $252.0 \pm 0.0$ \\
         &  & 0.2/0.545 & $988.7 \pm 56.4$ & $942.9 \pm 56.3$ & $999.3 \pm 0.0$ & $56.4 \pm 56.4$ \\
         &  & 0.2/0.645 & $999.2 \pm 0.0$ & $999.2 \pm 0.0$ & $999.2 \pm 0.0$ & $0.0 \pm 0.0$ \\
         &  & 0.2/0.745 & $999.3 \pm 0.0$ & $999.3 \pm 0.0$ & $999.3 \pm 0.0$ & $0.0 \pm 0.0$ \\
         &  & 0.3/0.545 & $999.3 \pm 0.0$ & $999.3 \pm 0.0$ & $999.3 \pm 0.0$ & $0.0 \pm 0.0$ \\
         &  & 0.3/0.645 & $999.3 \pm 0.0$ & $999.3 \pm 0.0$ & $999.3 \pm 0.0$ & $0.0 \pm 0.0$ \\
         &  & 0.3/0.745 & $999.2 \pm 0.0$ & $999.2 \pm 0.0$ & $999.2 \pm 0.0$ & $0.0 \pm 0.0$ \\
         & RC-D4PG & 0.05/0.545 & $-157.5 \pm 71.1$ & $279.5 \pm 71.1$ & $732.5 \pm 71.1$ & $453.0 \pm 0.0$ \\
         &  & 0.05/0.645 & $28.1 \pm 60.0$ & $368.4 \pm 60.0$ & $721.2 \pm 60.0$ & $352.8 \pm 0.2$ \\
         &  & 0.05/0.745 & $298.6 \pm 49.0$ & $542.7 \pm 49.0$ & $795.7 \pm 49.0$ & $253.0 \pm 0.0$ \\
         &  & 0.1/0.545 & $454.9 \pm 116.0$ & $469.8 \pm 158.4$ & $751.8 \pm 81.5$ & $282.1 \pm 82.6$ \\
         &  & 0.1/0.645 & $948.0 \pm 43.8$ & $950.4 \pm 44.7$ & $995.7 \pm 1.2$ & $45.3 \pm 43.8$ \\
         &  & 0.1/0.745 & $676.6 \pm 56.6$ & $686.6 \pm 73.0$ & $875.6 \pm 38.7$ & $189.0 \pm 41.2$ \\
         &  & 0.2/0.545 & $999.1 \pm 0.0$ & $999.1 \pm 0.0$ & $999.1 \pm 0.0$ & $0.0 \pm 0.0$ \\
         &  & 0.2/0.645 & $998.7 \pm 1.4$ & $996.5 \pm 1.4$ & $999.1 \pm 0.1$ & $2.6 \pm 1.4$ \\
         &  & 0.2/0.745 & $997.9 \pm 4.1$ & $992.8 \pm 4.1$ & $999.1 \pm 0.1$ & $6.3 \pm 4.1$ \\
         &  & 0.3/0.545 & $998.8 \pm 2.6$ & $996.6 \pm 2.6$ & $999.2 \pm 0.0$ & $2.6 \pm 2.6$ \\
         &  & 0.3/0.645 & $999.2 \pm 0.2$ & $999.0 \pm 0.2$ & $999.2 \pm 0.0$ & $0.2 \pm 0.2$ \\
         &  & 0.3/0.745 & $999.1 \pm 0.0$ & $999.1 \pm 0.0$ & $999.1 \pm 0.0$ & $0.0 \pm 0.0$ \\
         & MetaL & 0.05/0.545 & $108.6 \pm 0.1$ & $545.6 \pm 0.1$ & $998.6 \pm 0.1$ & $452.9 \pm 0.1$ \\
         &  & 0.05/0.645 & $305.3 \pm 0.1$ & $645.8 \pm 0.1$ & $998.7 \pm 0.0$ & $352.9 \pm 0.1$ \\
         &  & 0.05/0.745 & $501.6 \pm 0.0$ & $745.7 \pm 0.0$ & $998.7 \pm 0.0$ & $253.0 \pm 0.0$ \\
         &  & 0.1/0.545 & $523.0 \pm 0.3$ & $546.8 \pm 0.3$ & $998.4 \pm 0.1$ & $451.6 \pm 0.3$ \\
         &  & 0.1/0.645 & $674.4 \pm 44.0$ & $690.7 \pm 44.1$ & $998.7 \pm 0.1$ & $308.0 \pm 44.0$ \\
         &  & 0.1/0.745 & $766.7 \pm 31.4$ & $778.3 \pm 31.5$ & $998.4 \pm 0.1$ & $220.1 \pm 31.4$ \\
         &  & 0.2/0.545 & $998.3 \pm 4.6$ & $994.2 \pm 4.6$ & $999.2 \pm 0.0$ & $4.9 \pm 4.6$ \\
         &  & 0.2/0.645 & $999.0 \pm 0.7$ & $998.4 \pm 0.7$ & $999.1 \pm 0.0$ & $0.7 \pm 0.7$ \\
         &  & 0.2/0.745 & $999.2 \pm 0.0$ & $999.2 \pm 0.0$ & $999.2 \pm 0.0$ & $0.0 \pm 0.0$ \\
         &  & 0.3/0.545 & $999.2 \pm 0.3$ & $998.9 \pm 0.3$ & $999.2 \pm 0.0$ & $0.3 \pm 0.3$ \\
         &  & 0.3/0.645 & $999.2 \pm 0.2$ & $999.0 \pm 0.2$ & $999.2 \pm 0.0$ & $0.2 \pm 0.2$ \\
         &  & 0.3/0.745 & $998.7 \pm 1.8$ & $995.7 \pm 1.8$ & $999.2 \pm 0.0$ & $3.5 \pm 1.8$ \\
         & ReLOAD & 0.05/0.545 & $651.8 \pm 27.6$ & $657.7 \pm 16.6$ & $663.8 \pm 4.7$ & $6.1 \pm 25.4$ \\
         &  & 0.05/0.645 & $657.7 \pm 3.4$ & $659.3 \pm 3.4$ & $661.0 \pm 0.9$ & $1.7 \pm 3.2$ \\
         &  & 0.05/0.745 & $650.0 \pm 19.2$ & $654.7 \pm 18.7$ & $659.5 \pm 1.9$ & $4.9 \pm 19.1$ \\
         &  & 0.1/0.545 & $996.4 \pm 1.0$ & $996.4 \pm 1.0$ & $997.1 \pm 0.3$ & $0.7 \pm 0.9$ \\
         &  & 0.1/0.645 & $993.4 \pm 3.7$ & $993.7 \pm 3.5$ & $998.4 \pm 0.5$ & $4.7 \pm 3.6$ \\
         &  & 0.1/0.745 & $982.3 \pm 7.3$ & $983.2 \pm 7.2$ & $999.1 \pm 0.1$ & $16.0 \pm 7.3$ \\
         &  & 0.2/0.545 & $998.0 \pm 2.6$ & $992.6 \pm 2.6$ & $999.3 \pm 0.0$ & $6.7 \pm 2.6$ \\
         &  & 0.2/0.645 & $997.4 \pm 3.3$ & $988.9 \pm 3.3$ & $999.3 \pm 0.0$ & $10.4 \pm 3.3$ \\
         &  & 0.2/0.745 & $996.2 \pm 4.6$ & $982.8 \pm 4.6$ & $999.3 \pm 0.0$ & $16.5 \pm 4.6$ \\
         &  & 0.3/0.545 & $998.4 \pm 2.6$ & $992.5 \pm 2.6$ & $999.3 \pm 0.0$ & $6.8 \pm 2.6$ \\
         &  & 0.3/0.645 & $998.0 \pm 3.1$ & $989.7 \pm 3.2$ & $999.3 \pm 0.1$ & $9.6 \pm 3.1$ \\
         &  & 0.3/0.745 & $997.9 \pm 3.7$ & $988.8 \pm 3.7$ & $999.3 \pm 0.0$ & $10.6 \pm 3.7$ \\
        \midrule
        \bottomrule
    \end{tabular}
    \caption{Mean performance for \texttt{RWRL-Quadruped} averaged over 8 random seeds for each of the 12 constraint settings in the RWRL suite, where lower values of the safety coefficient and threshold indicate more challenging tasks. $\pm$ values denote one standard error. }
    \label{tab:rwrl_quadruped}
    \vspace{-1em}
\end{table}

\begin{table}[t]
    \small
    \centering
    \setlength\tabcolsep{3pt}
    \begin{tabular}{llccc}
        \toprule \hline
        \textbf{Domain (Constraint)} & \textbf{Algorithm} & \textbf{Weighted Reward} & \textbf{Task Reward} & \textbf{Constraint Violation} \Tstrut \\
        \midrule
        Walker (Height) & $\mu$-IMPALA & $354.5 \pm 27.7$ & $360.6 \pm 32.4$ & $5.5 \pm 4.6$ \\
         & OGD-IMPALA & $487.5 \pm 2.8$ & $510.1 \pm 15.1$ & $20.4 \pm 8.3$ \\
         & $\mu^\star$-IMPALA & $356.7 \pm 32.8$ & $527.5 \pm 65.0$ & $153.6 \pm 29.0$ \\
         & PID-IMPALA & $408.1 \pm 12.7$ & $462.6 \pm 27.4$ & $49.0 \pm 19.8$ \\
         & RNTR-IMPALA & $386.6 \pm 22.7$ & $399.0 \pm 28.2$ & $11.1 \pm 6.5$ \\
         & ReLOAD-IMPALA & $549.0 \pm 2.0$ & $592.2 \pm 16.4$ & $38.8 \pm 10.8$ \\
        \midrule
        Reacher (Velocity) & $\mu$-IMPALA & $436.9 \pm 31.7$ & $450.1 \pm 68.4$ & $117.3 \pm 26.0$ \\
         & OGD-IMPALA & $22.4 \pm 10.9$ & $44.6 \pm 13.1$ & $197.7 \pm 22.5$ \\
         & $\mu^\star$-IMPALA & $893.6 \pm 8.5$ & $913.6 \pm 18.0$ & $178.1 \pm 11.6$ \\
         & PID-IMPALA & $428.5 \pm 25.6$ & $439.5 \pm 72.0$ & $97.6 \pm 22.2$ \\
         & RNTR-IMPALA & $841.7 \pm 102.8$ & $850.0 \pm 103.0$ & $73.7 \pm 11.6$ \\
         & ReLOAD-IMPALA & $938.1 \pm 2.9$ & $961.8 \pm 11.7$ & $209.7 \pm 14.1$ \\
        \midrule
        \bottomrule
    \end{tabular}
    \caption{Oscillating control suite results for IMPALA-based agents, averaged over 8 random seeds. $\pm$ values indicate one standard error.}
    \label{tab:impala}
    \vspace{-1em}
\end{table}

\begin{table}[t]
    \small
    \centering
    \setlength\tabcolsep{3pt}
    \begin{tabular}{llccc}
        \toprule \hline
        \textbf{Domain (Constraint)} & \textbf{Algorithm} & \textbf{Weighted Reward} & \textbf{Task Reward} & \textbf{Constraint Violation} \Tstrut \\
        \midrule
        Walker (Velocity) & $\mu$-DMPO & $195.9 \pm 57.6$ & $441.0 \pm 57.5$ & $150.9 \pm 2.1$ \\
         & OGD-DMPO & $108.2 \pm 10.0$ & $357.5 \pm 10.0$ & $153.5 \pm 1.1$ \\
         & $\mu^\star$-DMPO & $118.1 \pm 15.9$ & $932.7 \pm 3.8$ & $501.5 \pm 15.4$ \\
         & PID-DMPO & $300.1 \pm 18.6$ & $544.2 \pm 18.5$ & $150.3 \pm 1.6$ \\
         & ReLOAD-DMPO & $321.0 \pm 49.3$ & $569.4 \pm 49.2$ & $152.9 \pm 2.0$ \\
        \midrule
        Quadruped (Torque) & $\mu$-DMPO & $-254.8 \pm 23.5$ & $567.1 \pm 22.1$ & $558.3 \pm 8.0$ \\
         & OGD-DMPO & $-220.1 \pm 20.8$ & $664.6 \pm 20.7$ & $601.0 \pm 1.3$ \\
         & $\mu^\star$-DMPO & $-328.4 \pm 36.1$ & $200.4 \pm 33.7$ & $359.2 \pm 13.0$ \\
         & PID-DMPO & $-232.2 \pm 21.1$ & $655.4 \pm 20.7$ & $602.9 \pm 4.2$ \\
         & ReLOAD-DMPO & $-137.0 \pm 55.7$ & $768.4 \pm 51.7$ & $615.0 \pm 20.6$ \\
        \midrule
        Humanoid (Height) & $\mu$-DMPO & $92.5 \pm 76.8$ & $685.8 \pm 59.4$ & $576.7 \pm 48.7$ \\
         & OGD-DMPO & $1.9 \pm 64.3$ & $606.1 \pm 49.6$ & $587.2 \pm 40.9$ \\
         & $\mu^\star$-DMPO & $-45.9 \pm 0.1$ & $0.0 \pm 0.0$ & $44.7 \pm 0.1$ \\
         & PID-DMPO & $39.9 \pm 67.0$ & $538.3 \pm 50.0$ & $484.4 \pm 44.6$ \\
         & ReLOAD-DMPO & $114.4 \pm 44.3$ & $650.1 \pm 36.3$ & $520.6 \pm 25.2$ \\
        \midrule
        \bottomrule
    \end{tabular}
    \caption{Oscillating control suite results for DMPO-based agents, averaged over 8 random seeds. $\pm$ values indicate one standard error.}
    \label{tab:dmpo_results}
    \vspace{-1em}
\end{table}

\begin{table}[t]
    \small
    \centering
    \setlength\tabcolsep{3pt}
    \begin{tabular}{lc}
        \toprule \hline
        \textbf{Hyperparameter} & \textbf{Value} \Tstrut \\
        \midrule
        regularization weight $\alpha_\Omega$ & 1e-2 \\
        value loss weight $\alpha_V$ & 0.25 \\
        policy loss weight $\alpha_\pi$ & 1.0 \\
        discount factor $\gamma$ & 0.99 \\
        RMSProp decay & 0.99  \\
        RMSProp $\epsilon$ & 1e-4 \\
        initial step size & 6e-4 \\
        final step size & 6e-6 \\
        max gradient norm & 0.2 \\
        inner loop steps & 5 \\
        initial Bregman step size & 2.0 \\
        final Bregman step size & 0.5 \\
        initial $\mu$ learning rate & 1e-1 \\
        final $\mu$ learning rate & 1e-3 \\
        network hidden units per layer & 256 \\
        network depth & 3 \\
        training duration (environment steps) & 10e6 \\
        \bottomrule
    \end{tabular}
    \caption{Hyperparameter settings for IMPALA experiments.}
    \label{tab:impala_hparams}
    \vspace{-1em}
\end{table}

\begin{table}[t]
    \small
    \centering
    \setlength\tabcolsep{3pt}
    \begin{tabular}{lc}
        \toprule \hline
        \textbf{Hyperparameter} & \textbf{Value} \Tstrut \\
        \midrule
        batch size & 256 \\
        replay size & 2e6 \\
        optimizer & Adam \\
        regularization weight & 1e-2 \\
        value loss weight & 0.25 \\
        discount factor $\gamma$ & 0.99 \\
        initial step size & 3e-4 \\
        final step size & 1e-5 \\
        initial $\mu$ learning rate & 1e-1 \\
        final $\mu$ learning rate & 1e-3 \\
        network hidden units per layer & 256 \\
        network depth & 3 \\
        training duration & 15e6  \\
        \bottomrule
    \end{tabular}
    \caption{Hyperparameter settings for DMPO experiments.}
    \label{tab:dmpo_hparams}
    \vspace{-1em}
\end{table}

\begin{figure}[ht]
    \begin{center}
    \centerline{\includegraphics[width=0.6\linewidth]{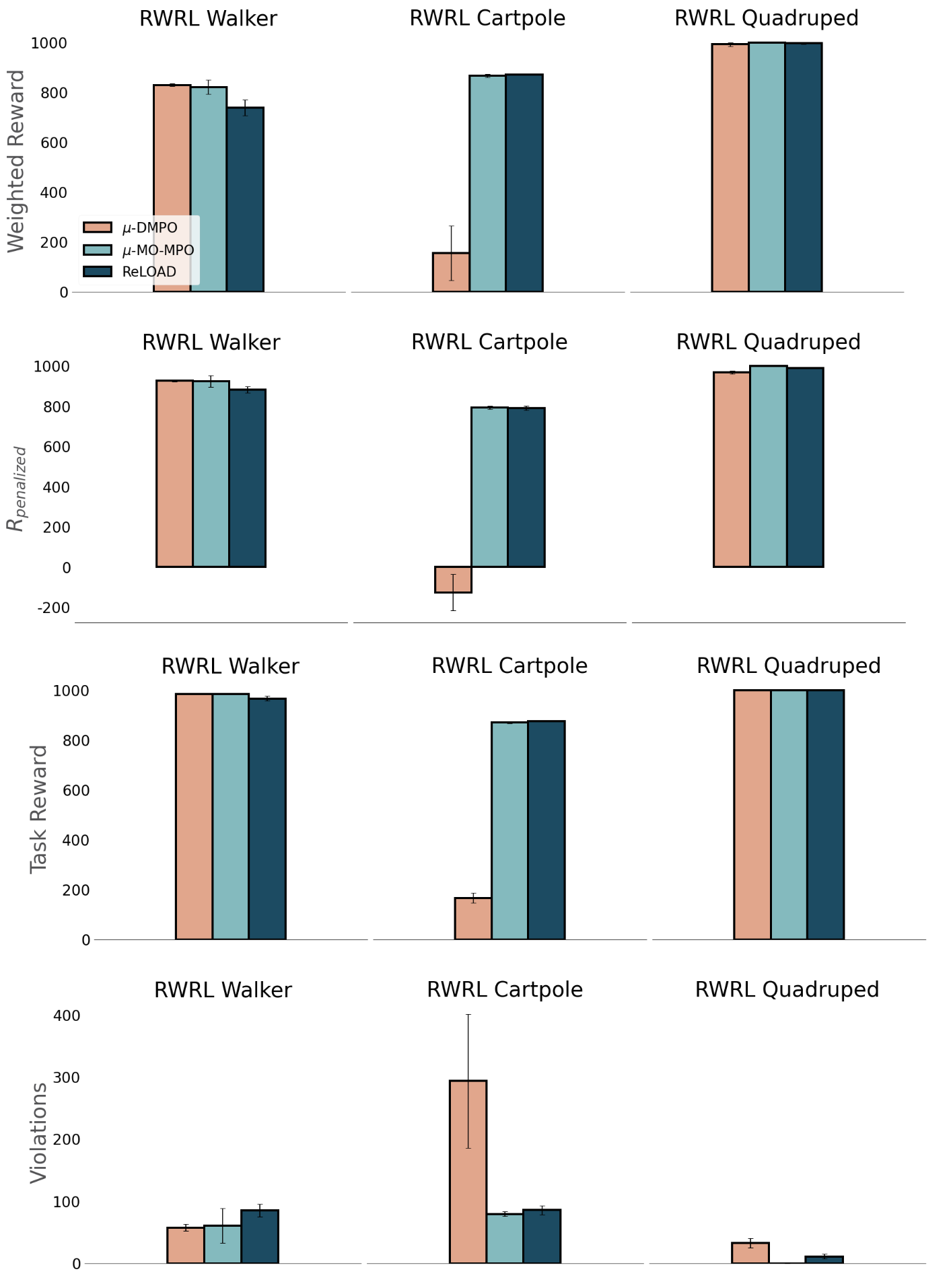}}
    \caption{\textbf{ReLOAD-DMPO outperforms $\mu$-DMPO and matches the performance of MO-DMPO on the RWRL Suite.} Here, we trained agents on the comparatively easy RWRL settings used by \cite{huang2022lp3} over 8 random seeds. Error bars denote one standard error.}
    \label{fig:rwrl_lp3}
    \end{center}
\end{figure}


\end{document}